\documentclass[journal]{IEEEtran} 

\usepackage[colorlinks=true, allcolors=blue]{hyperref}

\usepackage{graphicx, algorithm, algorithmic}
\usepackage{subcaption,comment}
\usepackage{booktabs, cite} 

\usepackage{etex}

\usepackage{chngcntr}
\usepackage{bbm}

\usepackage{url}            
\usepackage{amsfonts}       
\usepackage{amsmath,amssymb,amsthm, bm}
\usepackage{array}
\usepackage{mdwmath}

\usepackage{multirow}

\usepackage{graphicx}
\usepackage{caption}

\usepackage{tikz}
\usepackage{pgfplots}

\usepackage{tikz}
\usepackage{pgfplots, pgfplotstable}
\usepgfplotslibrary{groupplots}
\pgfplotsset{select coords between index/.style 2 args={
		x filter/.code={
			\ifnum\coordindex<#1\fi
			\ifnum\coordindex>#2\fi
		}
}}

\pgfplotsset{
	my stylecompare/.style={
		width=.5\linewidth,
		height=4.5cm,
		label style={font=\Large},
		title style={font=\Large},
		x tick label style={font =\small, /pgf/number format/1000 sep=},
		major x tick style = transparent,
		major y tick style = transparent,
		every y tick label/.style={
			xshift=-.7cm, yshift=-2pt,anchor=south west,inner sep=0pt,font=\small
		},
	},
	my legend style compare/.style={
		legend entries={
			RWF,
			proj-RWF,
			LRPR2,
			AltMinLowRaP,
			AltMinLowRaP (random init),
			AltMinLowRaP (rank known),
			RWF ($m=4n$),
			RWF ($m=3n$),
		},
		legend style={
			at={(-.05,1.65)},
			anchor=north west,
		},
		legend columns=7,
		legend style={font=\small},
	},
	cycle multi list={
		{red, line width=1.6pt, mark=square,mark size=3.5pt, mark repeat=10},
		{black, line width=1.6pt, mark=Mercedes star,mark size=6pt, mark repeat=10},
		{blue, line width=1.6pt, mark=o,mark size=4pt, select coords between index={0}{9}},
		{olive, line width=1.6pt, mark=diamond,mark size=5pt, select coords between index={0}{9}},
		{teal, line width=1.6pt, mark=triangle,mark size=5pt},
		{brown, line width=1pt, mark=oplus,mark size=3.5pt, mark repeat=1, select coords between index={0}{80}},
	},
}

\pgfplotstableread[col sep = comma]{time_comp_news.dat}\timecomp
\pgfplotstableread[col sep = comma]{rank_est_large_whp_new.dat}\rankestlarge
\pgfplotstableread[col sep = comma]{unknown_rank_small_whp.dat}\rankestsmall

\pgfplotstableread[col sep = comma]{all_work.dat}\allwork

\pgfplotstableread[col sep = comma]{pst_001_whp.dat}\pstsmall
\pgfplotstableread[col sep = comma]{pst_08_whp.dat}\pstlarge

\newtheorem{theorem}{Theorem}[section]
\newtheorem{lemma}[theorem]{Lemma}
\newtheorem{corollary}[theorem]{Corollary}

\newtheorem{remark}[theorem]{Remark}

\newtheorem{claim}[theorem]{Claim}

\newcommand{\bproof}{ \begin{IEEEproof} }
	\newcommand{\eproof}{ \end{IEEEproof} }
\newcommand{\beqno}{\begin{equation*} }
\newcommand{\eeqno}{\end{equation*} }
\newcommand{\beqa}{\begin{eqnarray*} }
	\newcommand{\eeqa}{\end{eqnarray*} }
\newcommand{\beq}{\begin{equation} }
\newcommand{\eeq}{\end{equation} }

\newcommand{\p}{\bm{p}}

\newcommand{\tot}{\mathrm{tot}}
\renewcommand{\a}{\bm{a}}

\newcommand{\x}{\bm{x}}
\newcommand{\e}{\bm{e}}
\newcommand{\h}{\bm{h}}
\renewcommand{\b}{\bm{b}}
\newcommand{\y}{\bm{y}}
\newcommand{\w}{\bm{w}}
\newcommand{\X}{\bm{X}}
\newcommand{\U}{{\bm{U}}}
\newcommand{\R}{{\bm{R}}}
\newcommand{\B}{\bm{B}}
\newcommand{\G}{\bm{G}}

\newcommand{\I}{\bm{I}}
\newcommand{\Y}{\bm{Y}}
\newcommand{\g}{\bm{g}}

\newcommand{\dd}{\bm{d}}

\newcommand{\xhat}{\bm{\hat{x}}}
\newcommand{\bhat}{\bm{\hat{b}}}
\newcommand{\Uhat}{\hat\U}

\newcommand{\D}{{\bm{D}}}

\newcommand{\F}{{\bm{F}}}
\newcommand{\bP}{{\bm{P}}}

\newcommand{\E}{\mathbb{E}}

\newcommand{\z}{\bm{z}}
\newcommand{\indic}{\mathbbm{1}}

\newcommand{\SE}{\mathrm{SE}}
\newcommand{\dist}{\mathrm{dist}}

\newcommand{\V}{\bm{V}}
\newcommand{\A}{\bm{A}}
\newcommand{\C}{\bm{C}}
\newcommand{\Chat}{\bm{\hat{C}}}
\newcommand{\cb}{\bm{c}}

\newcommand{\M}{\bm{M}}

\newcommand{\Span}{\mathrm{Span}}
\newcommand{\trace}{\mathrm{trace}}

\newcommand{\svdeq}{\overset{\mathrm{SVD}}=} 
\newcommand{\qreq}{\overset{\mathrm{QR}}=} 

\newcommand{\bi}{\begin{itemize}} 
\newcommand{\ei}{\end{itemize}}
\newcommand{\ben}{\begin{enumerate}}
\newcommand{\een}{\end{enumerate}}

\renewcommand{\S}{\mathcal{S}}
\newcommand{\W}{\bm{W}}

\begin{document}
\title{Phaseless Principal Components Analysis (PCA) \thanks{Part of this work is submitted to a blind conference.}}
\title{Phaseless PCA: Low-Rank Matrix Recovery from Column-wise Phaseless Measurements}
\title{Provable Low Rank Phase Retrieval
\thanks{Part of this work appears in the proceedings of ICML 2019 \cite{lrpr_icml}.}}
\author{Seyedehsara Nayer, Praneeth Narayanamurthy, Namrata Vaswani
\\Iowa State University, Ames, IA}
\date{}

\newcommand{\matdist}{\text{mat-dist}}
\newcommand{\Bstar}{{\B^*}}
\newcommand{\bstar}{\b^*}
\newcommand{\tB}{\tilde\B^*}
\newcommand{\tb}{\tilde\b^*}
\newcommand{\td}{\tilde{\bm{d}}^*}
\newcommand{\init}{{\mathrm{init}}}
\newcommand{\bea}{\begin{eqnarray}} 
\newcommand{\eea}{\end{eqnarray}}

\newcommand{\Ustar}{\U^*}
\newcommand{\Xstar}{{\X^*}}
\newcommand{\xstar}{\x^*}
\newcommand{\deltinit}{\delta_\init}
\newcommand{\deltapt}{\delta_{t}}
\newcommand{\deltaptplus}{\delta_{t+1}}

\newcommand{\bSigma}{{\bm\Sigma^*}}
\newcommand{\tSigma}{\bm{E}_{det}}
\newcommand{\sigmin}{{\sigma_{\min}^*}}
\newcommand{\sigmax}{{\sigma_{\max}^*}}

\newcommand{\HH}{{\bm{D}}}
\newcommand{\GG}{{\bm{M}}}
\newcommand{\SSS}{{\bm{S}}}
\newcommand{\ik}{{ik}}
\newcommand{\full}{{\mathrm{full}}}
\newcommand{\qfull}{q_\full}
\newcommand{\sub}{{\mathrm{sub}}}

\newcommand{\initm}{m_\init}
\renewcommand{\SE}{\sin \Theta}

\newcommand{\checkU}{\U} \newcommand{\checktB}{\B} \newcommand{\checktb}{\b}

\maketitle





\begin{abstract}
	We study the Low Rank Phase Retrieval (LRPR) problem defined as follows: recover an $n \times q$ matrix $\Xstar$ of rank $r$ from a different and independent set of $m$ phaseless (magnitude-only) linear projections of each of its columns. To be precise, we need to recover $\Xstar$ from $\y_k := |\A_k{}' \xstar_k|, k=1,2,\dots, q$ when the measurement matrices $\A_k$ are mutually independent. Here $\y_k$ is an $m$ length vector, $\A_k$ is an $n \times m$ matrix, and $'$ denotes matrix transpose. The question is when can we solve LRPR with $m \ll n$? A reliable solution can enable fast and low-cost phaseless dynamic imaging, e.g., Fourier ptychographic imaging of live biological specimens. 
	In this work, we develop the first provably correct approach for solving this LRPR problem. Our proposed algorithm, Alternating Minimization for Low-Rank Phase Retrieval (AltMinLowRaP), is an AltMin based solution  and hence is also provably fast (converges geometrically). 
	Our guarantee shows that AltMinLowRaP solves LRPR to $\epsilon$ accuracy, with high probability, as long as $m q \ge C n r^4 \log(1/\epsilon)$, the matrices $\A_k$ contain i.i.d. standard Gaussian entries,
	and the right singular vectors of $\Xstar$ satisfy the incoherence assumption from matrix completion literature. Here $C$ is a numerical constant that only depends on the condition number of $\Xstar$ and on its incoherence parameter. Its time complexity is only $ C mq nr \log^2(1/\epsilon)$.
	
	Since even the linear (with phase) version of the above problem is not fully solved, the above result is also the first complete solution and guarantee for the linear case. Finally, we also develop a simple extension of our results for the dynamic LRPR setting.
	%
	%
\end{abstract}
%
%
%
%
%

\section{Introduction}
In recent years, there has been a resurgence of interest in the classical phase retrieval (PR) problem \cite{fineup,ger_saxton}. The original PR problem  involved recovering an $n$-length signal $\xstar$ from the {\em magnitudes} of its Discrete Fourier Transform (DFT) coefficients. Its generalized version, studied in recent literature, replaces DFT by inner products with any arbitrary design vectors, $\a_i$. Thus, the goal is to recover $\xstar$ from $\y_i:=|\langle \a_i ,\xstar \rangle|$, $i=1,2, \dots, m$. These are commonly referred to as phaseless linear projections of $\xstar$.
While practical PR methods have existed for a long time, e.g., see \cite{fineup,ger_saxton}, the focus of the recent work has been on obtaining correctness guarantees for these and newer algorithms.  This line of work includes convex relaxation methods \cite{candes2013phase,candes_phaselift} as well as non-convex methods \cite{pr_altmin,altmin_irene_w,wf,twf,rwf,taf,pr_mc_reuse_meas,pr_random_init}. It is easy to see that, without extra assumptions, PR requires $m \ge n$. The best known guarantees -- see \cite{twf} and follow-up works -- prove exact recovery with high probability (whp) with order-optimal number of measurements/samples: $m = C n$; and with time complexity $C mn \log(1/\epsilon)$ that is nearly linear in the problem size.
Here and below, $C$ is reused often to refer to a constant more than one. 
Most guarantees for PR assume that $\a_i$'s are independent and identically distributed (iid) standard Gaussian vectors. When this is assumed, we refer to the PR problem as ``standard PR''.%

A natural approach to reduce the sample complexity to below $n$ is to impose structure on the unknown signal(s). In existing literature, with the exception of sparse PR which has been extensively studied, e.g., \cite{voroninski13,jaganathan2013sparse2,pr_altmin,sparta,cai,fastphase}, there is little other work on structured PR.
Low rank is the other common structure. 
This can be used in one of two ways. One is to assume that the unknown signal/image, whose phaseless linear projections are available, can be rearranged to form a low-rank matrix. This would be valid only for very specific types of images for which different image rows or columns look similar, so that the entire image matrix can be modeled as low rank. {\em In general it is not a very practical model for images, and this is probably why this setting has not been explored in the literature. We do not consider this model here either.}

A more practical, and commonly used, low-rank model in biological applications \cite{st_imaging}, is for the dynamic imaging setting. It assumes that a set, e.g., a time sequence, of signals/images is generated from a lower dimensional subspace of the ambient space. For our problem, we assume that we have a set of $m$ phaseless linear projections of each signal, with a different set of measurement vectors used for each signal. The question is can we jointly recover the signals using an $m \ll n$ and when? This setting was first studied in our recent work \cite{lrpr_tsp} where we called it ``Low-Rank PR'' (LRPR). It is a valid model whenever the set/sequence of signals is sufficiently similar (correlated).  A solution to LRPR can enable fast and low-cost phaseless dynamic imaging of live biological specimens, in vitro. See Sec. \ref{motiv} and \cite{TCIgauri} for a detailed motivation for studying LRPR.

\begin{figure*}[t!]
	\begin{center}
		\resizebox{.7\linewidth}{!}{
			\begin{tabular}{ccc}
				\\    \newline
				\includegraphics[scale=1, trim={.1cm, .1cm, .1cm, .1cm}, clip=true]{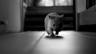}
				&
				\includegraphics[scale=1, trim={.1cm, .1cm, .1cm, .1cm}, clip=true]{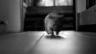}
				&
				\includegraphics[scale=1, trim={.1cm, .1cm, .1cm, .1cm}, clip=true]{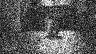}
				\\    \newline
				\subcaptionbox{Original\label{1}}{\includegraphics[scale=1, trim={.1cm, .1cm, .1cm, .1cm}, clip=true]{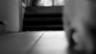}}
				&
				\subcaptionbox{AltMinLowRaP\label{1}}{\includegraphics[scale=1, trim={.1cm, .1cm, .1cm, .1cm}, clip=true]{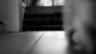}}
				&
				\subcaptionbox{RWF\label{1}}{\includegraphics[scale=1, trim={.1cm, .1cm, .1cm, .1cm}, clip=true]{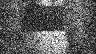}}
				
			\end{tabular}
		}
	\end{center}
	\vspace{-0.1cm}
	\caption{Recovering a video of a moving mouse (only approximately low-rank) from simulated coded diffraction pattern (CDP) measurements. The images are shown at $k = 60, 78$. We describe the experiment in Sec. \ref{small_r}.
	}
	\label{fig:vid_frames}
	\vspace{-0.3cm}
\end{figure*}

\subsection{Low Rank PR (LRPR) Problem Setting and Notation}\label{sec:prob}
We study the LRPR problem described above. This was first introduced and briefly studied in \cite{lrpr_tsp} where we developed two algorithms, evaluated them experimentally, and provided a guarantee for the initialization step of one of them.
The goal is to  
to recover an $n \times q$ matrix $\Xstar:= [\xstar_1, \xstar_2, \dots, \xstar_k, \dots, \xstar_q]$, of rank $r$, from measurements
\begin{align}\label{eq:problem}
	\y_{ik}:=| \langle \a_{ik} \bm{e}_k{}', \Xstar \rangle| =  | \langle \a_{ik}, \xstar_k \rangle|, \ i \in [m], k \in [q],
\end{align}
when all the $\a_\ik$'s are mutually independent.  For proving guarantees we assume also that they are iid standard Gaussian and real-valued. Here and below the notation $ [m]:=\{ 1,2,\dots,m \}$, and $\e_k$ refers to the $k$-th column of $\I_q$ (identity matrix of size $q\times q$). We are interested in the low rank setting when $r \ll \min(n,q)$.

By defining the $m$-length vector $\y_k:=[\y_{1,k}, \y_{2,k}, \dots, \y_{m,k}]'$ and the $n \times m$ matrix $\A_k := [\a_{1,k},\a_{2,k}, \dots, \a_{m,k}]$, and letting $|\z|$ denote element-wise magnitude of a vector, the above measurement model can also be rewritten as
\begin{align}\label{prob2}
	\y_k :=  | \A_k{}' \xstar_k|, \ k=1,2,\dots, q
\end{align}
where $'$ denotes vector or matrix transpose.


The requirement that the measurement vectors used for imaging different $\xstar_k$'s be different and independent is what allows us to show that $m \ll n$ suffices. 
To understand this point in a simple fashion, consider the $r=1$ setting and suppose that $\xstar_k = \xstar_1$ (all columns are equal). We would then have $mq$ iid Gaussian measurements of $\xstar_1$ and hence $mq \ge C n$ would suffice \cite{twf}. If $q=n$, this means just a constant number of measurements $m$ per column (signal) suffices.   For $r>1$ but small, we will show that we can extend this idea to show that, when $q=n$ (or is larger), the required  value of $m$ depends only on the value of $r$ and not on $n$.
On the other hand, if $\a_\ik = \a_i$ for all $k$, then, in the above $r=1$ example, only the first $m$ measurements are useful (the others are just repeats of these). Thus we will still need $m \ge Cn$ in this case.  This $\a_\ik = \a_i$ case, and its linear version, is what has been studied extensively in the literature \cite{wainwright_linear_columnwise,cov_sketch}. In this case, $m$ needs to be at least $(n+q) r$. 




Let $\X^* \svdeq \U^* \bm\Sigma^* \B^*$ denote its singular value decomposition (SVD) so that $\U^* \in \mathbb{R}^{n\times r}$, $\B^* \in \mathbb{R}^{r \times q}$, and $\bm\Sigma^* \in \mathbb{R}^{r \times r} $ is a diagonal matrix. {\em Observe that this notation is a little non-standard, if the SVD was $\U^* \bm\Sigma^* {\V^*}'$, we are letting $\B^*:={\V^*}'$. Thus, columns of $\U^*$ and rows of $\B^*$ are orthonormal.}
We use $\sigmax,\sigmin$ to denote the maximum, minimum singular values of $\Xstar$ and $\kappa = \sigmax/\sigmin$ to denote its condition number.
Finally, we define
\[
\tB := \bm\Sigma^*  \B^*.
\]
We use the above non-standard notation for SVD because our solution approach will recover columns of $\tB$, $\tb_k$, individually by solving an $r$-dimensional standard PR problem (it is more intuitive to talk about recovery of column vectors than of rows). With the above notation, the QR decomposition of an estimate of $\tB$, denoted $\hat\B$, will be written as $\hat\B \qreq \R_B \B$ with $\B$ being an $r \times q$ matrix with orthonormal rows (or equivalently $\hat\B' \qreq \B' (\R_B)'$).

{\em Right Incoherence. }
Observe that we have global measurements of each column, but not of the entire matrix. Thus, in order to correctly recover $\Xstar$ with small $m$, we need an assumption that allows for correct ``interpolation'' across the rows. One way to ensure this is to borrow the right incoherence (incoherence or denseness of right singular vectors) assumption from matrix completion literature \cite{matcomp_candes,lowrank_altmin}. 
In our notation, this means that we need to assume that
\begin{align}
	\max_k \| \b^*_{k} \|^2 \leq \mu^2 \frac{r}{ q },
	\label{right_incoh_1}
\end{align}
with $\mu \ge 1$ being a constant. Clearly, this implies that
\begin{align}
	\|\xstar_k\|^2 = \|\tb_k\|^2 \le  \sigmax^2 \mu^2 \frac{r}{ q } = \kappa^2 \sigmin^2  \mu^2  \frac{r }{ q }   \leq   \kappa^2 \mu^2 \frac{\|\X^*\|_F^2 }{ q}
	\label{right_incoh_2}
\end{align}
for each $k$.
If we assume $\kappa$ is a constant, up to constant factors, \eqref{right_incoh_2} also implies \eqref{right_incoh_1}. Thus, up to constant factors, requiring right incoherence is the same as requiring that the maximum energy of any signal $\xstar_k$ is within constant factors of the average.

\begin{table*}[t!]
	\caption{\small{Comparing our result (first result for LRPR) with the first and the best results for (bounded time) non-convex algorithms for the three related problems to ours -  LRMC, PR, Sparse PR. Here ``best'' refers to best sample complexity. We treat $\kappa,\mu$ as constants. $\Xstar$ is an $n\times q$ rank-$r$ matrix; $\xstar$ is an $n$-length vector signal; and $m$ is the number of samples needed per signal (per column of $\Xstar$). LRPR, Sparse PR, and Standard PR results assume iid Gaussian measurements, while LRMC results assume iid Bernoulli model on observed entries.
	}}
	\begin{center}
		\renewcommand*{\arraystretch}{1.01}
		\resizebox{\linewidth}{!}{
			\begin{tabular}{|l|l|l|l|l|} \toprule
				Problem & Global & Assumptions & Sample Complexity  & Time Complexity per signal  \\
				& Measurements?  &&& (with $m$ = its lower bound)  \\ \midrule
				{\bf LRPR (first)} & {\bf No} & {\bf  right incoherence,} & {\bf  $m \ge C \frac{n}{q} r^4 \log (1/\epsilon)$,} &  {\bf $C n (n/q) r^5  \log^3(1/\epsilon)$ } \\
				{\bf (our work) }   &  &  {\bf $\Xstar$ has rank $r$ } & {\bf  $m \ge C \max(r, \log n, \log q)$}    & \\ \hline
				LRMC (first) \cite{lowrank_altmin} &  No  &  left \& right incoherence, & $m \ge C \frac{\max(n,q)}{q} r^{4.5} \log (1/\epsilon)$ &  $C (n/q) r^{6.5} \log n  \log^2(1/\epsilon)$ \\
				&   & $\Xstar$ has rank $r$  && \\ \hline
				LRMC (best) \cite{rmc_gd} & No  & left \& right incoherence & $m \ge C \frac{ \max(n,q) }{q} r^2 \log^2 n \log^2(1/\epsilon) $ &  $C (n/q) r^3 \log^n \log^2( 1/\epsilon)$ \\
				& &  $\Xstar$ has rank $r$  && \\ \hline
				Sparse PR (first) \cite{pr_altmin} & Yes  & $\xstar$ is $s$-sparse in canonical basis, & $m \ge C s^2 \log n \log (1/\epsilon) $ &  $Cn s^3 \log n \log^2(1/\epsilon)$ \\
				&     &                min nonzero entry lower bounded & & \\ \hline
				Sparse PR (best) \cite{cai,fastphase} &  Yes  & $\xstar$ is $s$-sparse in canonical basis & $m \ge C s^2 \log n $ &  $Cn s^2 \log n \log(1/\epsilon)$ \\
				&&&& \\ \hline
				Standard PR (first) \cite{pr_altmin} & Yes  & None & $m \ge C  n \log^3 n \log(1/\epsilon)$ &  $Cn^2 \log^3 n \log(1/\epsilon)$ \\
				&     &                  & & \\ \bottomrule
				Standard PR (best) \cite{twf,rwf} & Yes  & None & $m \ge C  n$ &  $Cn^2 \log(1/\epsilon)$ \\
				&     &                  & & \\ \bottomrule
				
			\end{tabular}
		}
		\label{compare_assu}
	\end{center}
	\vspace{-0.05in}
\end{table*}

\subsubsection{Notation}
We use $\|.\|$ to denote the $l_2$-norm of a vector or the induced 2-norm matrix and $\|.\|_F$ to denote the Frobenius norm. We use $\indic_{\mathrm{statement}}$ to denote the indicator function; it takes the value one if $\mathrm{statement}$ is true and is zero otherwise.
A tall matrix with orthonormal columns is referred to as a {\em ``basis matrix''}. For two basis matrices $\U_1, \U_2$, we define the subspace error (distance) as
\[
\SE(\U_1,\U_2) :=  \|(\I - \U_1 \U_1') \U_2\|. 
\]
This measures the sine of the largest principal angle between the two subspaces. We often use terms like ``estimate $\U$'' when the goal is to really estimate its column span, $\Span(\U)$.
The phase-invariant distance is defined as
\[
\dist(\xstar,\xhat) : = \min_{\theta \in [-\pi, \pi]} \|\xstar - e^{-j \theta} \xhat\|
\]
For our guarantees, we  {\em work with real valued vectors and matrices}, and in this case this simplifies to $\dist(\xstar,\xhat) = \min( \|\xstar - \xhat\|, \|\xstar + \xhat\|)$.
Define the corresponding distance between two matrices as
\[
\matdist(\hat\X,\Xstar)^2:= \sum_{k=1}^q \dist^2(\xstar_k,\xhat_k).
\]
We {\em reuse} the letters $c$, $C$ to denote different numerical constants in each use, with the convention $C \geq 1$ and $c < 1$.%

\subsection{Our Contributions and their Significance and Novelty} \label{contrib}
This work provides the {\em first} provably correct solution, AltMinLowRaP (Alternating Minimization for Low Rank PR), for Low Rank PR. AltMinLowRaP is a fast alternating minimization (AltMin) based solution approach with a carefully designed spectral initialization. We can prove that AltMinLowRaP converges geometrically to an $\epsilon$-accurate solution as long as (i) right incoherence stated in \eqref{right_incoh_1} holds, and (ii) the total number of available measurements, $mq$, is at least a constant (that depends on $\kappa,\mu$) times $nr^4 \log(1/\epsilon)$. Its time complexity is order $mqnr \log^2(1/\epsilon)$, but if we replace $mq$ by its lowest allowed value, then the time complexity becomes $O(n^2 r^5 \log^3(1/\epsilon))$. If $q = n$ (or is larger), ignoring log factors, this implies that only about $r^4$ (or lesser) samples per signal suffice when using AltMinLowRaP. Moreover, when using these many samples, the time complexity per signal is only about $C n r^5$. On the other hand, standard PR approaches (to recover each signal $\xstar_k$ individually) necessarily need $m \ge C n$ samples, and order $ C n^2$ time, per signal \cite{twf,rwf}. In the regime of small $r$, e.g., $r=\log n$, our result provides a significant sample, and time, complexity improvement over standard PR. Moreover, in this regime, our sample complexity is also only a little worse than its order optimal value of $(n + q)r$.
We demonstate the practical power of AltMinLowRaP in Fig. \ref{fig:vid_frames}, Fig. \ref{fig:time_comp}, and the other experiments described later.

The key insight that helps obtain the above reduction in sample complexity is the following observation: for both the initialization and the update steps for $\Ustar$, conditioned on $\Xstar$, we have access to $mq$ mutually independent measurements. These are not identically distributed (because the different $\bstar_k$'s could have different distributions), however, we can carefully use the right incoherence assumption to show that the distributions are ``similar enough'' so that concentration inequalities can be applied jointly for all the $mq$ samples. 
%
%
We also briefly study the dynamic setting, Phaseless Subspace Tracking, which allows the underlying signal subspace to change with time in a piecewise constant fashion.

To our best knowledge, even the linear version of our problem -- recover $\Xstar$ from $\y_k:=\A_k{}'\xstar_k,k=1,2,\dots,q$ -- does not have any provably correct solutions (as we explain in Sec. \ref{rel_work}).  Thus, our work also provides the first provable solution for this linear version. Our result provides an immediate corollary for this case as well. 
What has been studied extensively is the $\A_k = \A$  version of both LRPR and the above linear version \cite{cov_sketch,wainwright_linear_columnwise}. These are completely different problems as explained earlier in Sec. \ref{sec:prob}.

\subsubsection{Significance}
The other somewhat related problem to ours, Sparse PR, is actually quite different. This is because it involves recovery from global measurements of the sparse vector (each measurement depends on the entire unknown sparse vector), where as, in our case, the measurements are not global for the entire matrix $\Xstar$. It is well known that, when studying iterative (non-convex) solutions to problems, the global measurements' setting is easier to study, and one can obtain better sample complexity guarantees for it, as compared to its non-global counterpart \cite{nonconvex_review}.
For example, we can compare guarantees for iterative low rank matrix sensing (LRMS) from iid Gaussian linear projections with those for low rank matrix completion (LRMC) when assuming the iid Bernoulli model on observed entries \cite{matcomp_candes,lowrank_altmin,Bresler_matrix}. LRMS can be solved using an iterative algorithm with nearly order optimal number of measurements, e.g., the approach of \cite{Bresler_matrix} needs $mq \ge C (n+q) r \log \max(n,q)$, while even the best iterative LRMC guarantee (under the iid Bernoulli measurement model) \cite{rmc_gd} needs $mq \ge C  (n+q) r^2 \log^2(nq)  \log^2 (1/\epsilon)$.

In this sense, the problem closest to ours that is extensively studied is LRMC. Of course LRMC involves recovery from completely local but linear measurements of $\Xstar$, while LRPR involves recovery from nonlinear but column-wise global measurements. For this reason, for LRPR, in the regime of $q$ significantly larger than $n$, the required sample complexity $m$ is very small. As an example, suppose that $q \ge nr^4$, then we only need $m \ge C \max(r, \log q, \log(1/\epsilon))$. But this does not happen for LRMC.

We provide a comparison with the first and the best guarantees for non-convex (iterative) solutions for LRMC, sparse PR, and standard PR in Table \ref{compare_assu}. These and other works are discussed in detail in Sec. \ref{rel_work}. As can be seen from the table, the first guarantee for iterative solutions to many problems is often sub-optimal (either needs more samples or more assumptions) compared to the best one that appeared later. Moreover, in the practical regime of $r$ being order $\log n$ or smaller, our LRPR sample complexity is as good or better than that of the best LRMC guarantee.

\subsubsection{Novelty} \label{novelty}
%
In the absence of relevant existing work for even solving the linear version of our problem, except  for a convex solution for the $\a_\ik=\a_i$ case (which is a significantly different problem), developing and analyzing our approach was not a straightforward extension of existing ideas. For example, the AltMinLowRaP algorithm itself is not just alternating standard PR over $\U$ and $\B$. The PR problem for recovering $\Ustar$ given an estimate of $\tB$ is significantly different from standard PR; see Sec. \ref{altminlowrap_algo} below.

For the above reasons, it is also not possible to directly modify proof techniques from existing work.
We borrow some ideas from LRMC \cite{lowrank_altmin} and  standard PR results \cite{rwf,twf}. But the major difference is that concentration bounds need to applied differently than for either of these problems. (i) The LRMC guarantees use results for Bernoulli random variables (which is a much more well-developed literature that has also been studied in the context of random graphs). In our setting, the random variables are not Bernoulli and not even bounded. Hence  we rely on the sub-exponential Bernstein inequality \cite[Theorem 2.8.1]{versh_book} and the fact that the product of two sub-Gaussian random variables is a sub-exponential \cite[Lemma 2.7.7]{versh_book}. A second difference is that LRMC results do not need to deal with the phase error term. (ii)  Standard PR results do have a phase error term and do deal with unbounded random variables using results from \cite{versh_book}. But they do not have to prove concentration using a set of $mq$ measurements that are not identically distributed and, on first glance, may not even be ``similar enough'' to get a useful result. The ``similarity'' that is needed is of the following form: the maximum sub-exponential norm of any of the $mq$ random variables being summed is not much larger than its average value. For each term, we have to carefully exploit the right incoherence assumption to show that this holds.

\subsection{Motivation for studying Low Rank PR (LRPR)} \label{motiv}
Low rank is a commonly used model in many dynamic biomedical imaging applications since (i) such images cannot change too much from one frame to the next, and (ii) these images are taken in controlled settings and so there are no fast changing foreground occlusions to worry about\footnote{Occlusions by moving objects or persons in the foreground are a common feature in computer vision problems such as surveillance or autonomous vehicles etc; for such videos a sparse + low-rank model is more appropriate}. For example, it is an important part of many practically useful fast compressive dynamic MRI solutions, e.g., see \cite{st_imaging}, and follow-up works \cite{dyn_mri1,dyn_mri2}\footnote{These follow-up works exploit both low-rank of the entire sequence and wavelet sparsity of each image to further reduce the number of measurements needed in practice. This is the so-called ``sparse {\em and} low-rank'' model which is very different from sparse+low-rank model where the sparse component models occlusions by foreground moving objects.}.
In a similar fashion, a low sample complexity solution to LRPR can enable fast or low-cost dynamic phaseless imaging in applications such as solar imaging when the sun's surface properties gradually change over time \cite{butala}, or Fourier ptychographic imaging of live biological specimens and other dynamic scenes \cite{holloway,TCIgauri}. Suppose the scene resolution is $n$ and the total number of captured frames is $q$. If the dynamics is approximated to be linear and slow changing, with most of the change being explained by $r$ linearly independent factors, then the matrix formed by stacking the vectorized image frames next to each other can be modeled  as a rank-$r$ matrix plus small modeling error. In typical settings, $r \ll \min(n,q)$ is a valid assumption, making the unknown images' matrix approximately low-rank.

In all the above applications, {\em measurement acquisition is either expensive or slow}. For example, Fourier ptychography is a technique for super-resolution in which each of a set of low resolution cameras measures the magnitude of a different band-pass filtered version of the target high-resolution image. To get enough measurements per image, one either needs many cameras (expensive), or one needs to move a single camera to different locations to acquire the different bands~\cite{holloway}. This can make the acquisition process slow.
By exploiting the low-rank assumption, it is possible to get an accurate reconstruction with using fewer total measurements (fewer cameras in this example). This has been demonstrated experimentally for dynamic Fourier ptychography in our recent work \cite{TCIgauri} and its follow-up \cite{icip20}. 
Moreover, it is indeed {\em practically valid to assume that a different measurement matrix $\A_k$ is used for each different signal/image}. In the ptychography example, this would correspond to using a different randomly selected subset of cameras at different times $k$. Modified cameras can also be designed that save power by switching off a different set of pixels at different time instants. We have explored both settings in \cite{TCIgauri}.

Another practical point that should be mentioned is that, often, {\em in practice, a very small value of rank $r$ suffices}. For example, we used $r = 20$ in all our experiments on image sequences with n = 32400 in \cite{TCIgauri}. In follow-up work \cite{icip20}, we show that just $r = 5$ suffices for the same datasets, as long as a ``modeling error correction step''\footnote{This step applies a few iterations of any standard PR approach column-wise to the output of AltMinLowRaP, in order to recover some of the “modeling error” in the low-rank assumption.} is applied to the output of AltMinLowRaP. 

Lastly, in comparison to sparsity or structured sparsity priors, the low rank prior is a significantly more flexible one since it does not require knowledge of the dictionary or basis in which the signal is sufficiently sparse.  In Table \ref{tab:vid_sp}, we demonstrate this via a simple experiment. We compare AltMinLowRaP with the most recent provable sparse PR algorithm \cite{fastphase}, CoPRAM, applied with assuming wavelet sparsity (which is a generic choice for any piecewise smooth image, but is not necessarily the best choice for the particular image).
As can be seen, AltMinLowRaP has significantly superior performance not just for the real image sequence, but also for its deliberately sparsified version.  The sparsified sequence had sparsity level $s\approx 0.1n \approx100$ and we provided CoPRAM with this ground truth. AltMinLowRaP used just $r=15$ for all three results and still had much lower reconstruction error than CoPRAM. Details of this experiment are provided in Sec. \ref{small_r}. Moreover, low-rank also includes certain types of dynamic sparsity models (those with fixed of very small changes in support over time) as special cases.

\subsection{Review of Related Work} \label{rel_work}



\subsubsection{Linear version of our problem: Compressive PCA} 
While one would think that the linear (with phase) version of our problem would been extensively studied, this is not true.
There have been a few algorithmic solutions for this problem in prior work \cite{hughes_icip_2012,hughes_icml_2014}, and attempts to prove some facts theoretically. Follow-up work consists of an Asilomar 2014 paper \cite{aarti_singh_subs_learn} that solves the general PCA problem for any (not necessarily low rank) matrix $\Xstar$, but does not discuss recovery of $\Xstar$.  We explain these in detail in Sec. \ref{our_linear_details}.


\subsubsection{Our measurement model, but with same set of $m$ measurement vectors used for all signals, and its linear version}
The covariance sketching problem, e.g., see \cite{cov_sketch}, assumes that measurements satisfying \eqref{eq:problem}, but with $\a_\ik = \a_i$, are available. One aggregates these over $k$ to get $\y_i  :=  \sum_k \y_\ik =  \a_i' (\sum_k \xstar_k \xstar_k{}')  \a_i = \a_i' \Xstar \Xstar' \a_i$.  This aggregation is what ensures that the memory complexity of storing the measurements is order $m$ and not $mq$ (which is what we need). Also, the aggregated $\y_i$ is a function of $\Xstar \Xstar'/q$ only in the $\a_\ik=\a_i$ setting, otherwise it is a meaningless quantity. 
Assuming random zero mean iid signals $\xstar_k$, $\Xstar \Xstar'/q$ is the empirical covariance matrix of a signal. The question is can we recover $\Xstar \Xstar'/q$ from the scalar sketches $\y_i, i=1,2,\dots m$ with using $m$ much smaller than $nq$, when $\Xstar$ is low rank (or has other structure)? When $\Xstar$ is rank $r$, the result of \cite{cov_sketch} proves that $m$ of order $(n+q)r$ suffices to estimate the empirical covariance from the aggregated measurements $\y_i$ if one solves an appropriately defined nuclear norm minimization problem. For solving LRPR, we need a much smaller $m$ than this. The reason is we assume independent $\a_\ik$'s for different $k$, and we assume we have access to each individual $\y_\ik$.



The linear version of the above problem, but with random noise added, is considered in Corollary 3 of \cite{wainwright_linear_columnwise} and in the remark immediately below it. 
In our notation, its measurement model can be written as $\y_\ik = \langle \a_i \bm{e}_k{}', \Xstar \rangle + \bm{w}_\ik$ where $\bm{w}_\ik$ is iid zero mean Gaussian noise with variance $\nu^2$.
This result (specialized to the exact low rank case) shows that, whp, a nuclear norm minimization based solution will recover an estimate $\hat\X$ of $\Xstar$ that satisfies
$\|\Xstar - \hat\X \|_F^2 \le C \ \nu^2 r (n+q) / m.$
In this paper, the focus is on using the low rank property to achieve noise robustness. If the low rank property was not used, and one attempted to recover the columns individually, the recovery error bound would scale as $\nu^2 nq / m$ which is much larger.
%
This paper also studies other settings of recovering an approximately low rank matrix from linear measurements.

\subsubsection{Tangentially related work}
Some other tangentially related work includes: (i) computing the approximate rank $r$ approximation of any matrix (need not be low rank) from its random sketches \cite{sketch_1, sketch_2} (sketched SVD); (ii) compressed covariance estimation using different sketching matrices for each data vector, but without the low-rank assumption \cite{aarti_singh_cov_est};  and (iii) a generalization of low-rank covariance sketching \cite{local_conv_pr}: this attempts to recover an $n \times r$ matrix $\Ustar$ from measurements $\y_i = \|\a_i'\Ustar\|^2$ with $r \ll n$. When $r=1$, this is the standard PR problem. In the general case, this is related to covariance sketching described above, but not to our problem.


\subsubsection{Linear low-rank matrix recovery -- LRMS and LRMC}
Low-rank matrix recovery problems with linear measurements that have been extensively studied can be split into two kinds - those with ``global measurements'' and those without. ``Global measurements'' means that each measurement contains information about the entire structured quantity-of-interest, here the low-rank matrix. Such problems are called ``affine rank minimization problems'' or ``low-rank matrix sensing'' (LRMS)  and involve recovery of $\Xstar$ from $\y_i = \langle \A_i, \Xstar \rangle$ with $\A_i$ being dense matrices (typically iid Gaussian), see for example, \cite{first_lrms_convex,Bresler_matrix,tanner,svp,lowrank_altmin}. More recent work studies the case of  $\A_i = \a_i \a_i{}'$ \cite{procrustes,zheng_lafferty}.
Low-rank Matrix Completion (LRMC) is the completely local measurements' setting that involves recovering $\Xstar$ from measurements of a randomly (iid Bernoulli) selected subset of its entries \cite{matcomp_candes,optspace,lowrank_altmin,mc_luo,rmc_gd,pr_mc_reuse_meas} . Thus  $\A_i$'s are  one-sparse matrices.

A precursor to LRMS is compressive sensing (CS) of sparse signals. This has the same property as LRMS, it involves recovering a sparse  $\xstar$ from $\y_i =  \langle \a_i, \xstar \rangle$ with $\a_i$ being dense (sub-)Gaussian random vectors. Similar to CS, even for LRMS, it is typically possible to prove a simple (sparse or low-rank) restricted isometry property which simplifies the rest of the analysis. Our problem setting is different from, and more difficult than, LRMS. There are no ``global measurements'' of the entire $\Xstar$ and, moreover, the measurements' phase/sign is unknown. In this sense it is closer to LRMC than to LRMS. But, unlike LRMC, we do have column-wise global measurements. This is why, for our problem only incoherence of right singular vectors suffices, while LRMC needs incoherence of both left and right singular vectors. 

In summary, the problem closest to ours that is well studied is LRMC. The first iterative solution to LRMC was \cite{optspace}. However, its guarantee does not bound the required number of algorithm iterations, and thus its time complexity cannot be bounded. The first iterative LRMC solution with bounded time complexity, AltMinComplete \cite{lowrank_altmin}, needs a sample complexity of about $C \kappa^4 \mu^2 nr^{4.5} \log(1/\epsilon)$ and assumes sample-splitting (a different independent set of measurements is used at each iteration).  The most recent work on LRMC \cite{pr_mc_reuse_meas} removes the sample-splitting requirement and has bounded time complexity, but its sample complexity was $C nr^3 \log^6(n)$. The best iterative LRMC solution in terms of sample complexity \cite{rmc_gd} needs $C nr^2 \log^2 n \log^2 \kappa/\epsilon$ samples but needs sample-splitting. In the practical regime $r \in O(\log n)$, clearly our sample complexity for LRPR is comparable to the best LRMC guarantee \cite{rmc_gd}. For all values of $r$, it is slightly better than the first bounded time iterative LRMC solution \cite{lowrank_altmin}. Time-wise
the AltMin algorithm for LRMC is faster than our AltMinLowRaP algorithm for LRPR (see Table \ref{compare_assu}). This is because the LS problem to be solved in each AltMin step of LRMC involves a matrix with a large number of zeros but this is not the case for LRPR.


\subsubsection{Sparse PR} Sparse PR is a somewhat related problem to ours since it involves PR with a different type of structural assumption on the signals. But, as noted earlier, there is a major difference. Sparse PR involves recovery from global measurements of the entire sparse vector. It can be understood as the phaseless version of Compressive Sensing with random Gaussian measurements. The global measurements' setting is typical easier than its non-global counterparts.
%
Provably correct sparse PR approaches include convex relaxation approaches such as $\ell_1$-PhaseLift  \cite{voroninski13}; older combinatorial methods  \cite{jaganathan2012recovery}; and a series of fast iterative approaches: (i) AltMinSparse \cite{pr_altmin}, (ii) Sparse Truncated Amplitude Flow (SPARTA)  \cite{sparta}, (iii) Thresholded WF \cite{cai} and  CoPRAM \cite{fastphase}. 
The first two fast nonconvex approaches -- AltMinSparse and SPARTA -- needed to assume a lower bound on the minimum nonzero entry of $\x$. In follow-up work on Thresholded WF and then on CoPRAM, this extra assumption was removed. All four results need at least order $s^2 \log n$ measurements.
A summary of comparison of our work with LRMC, sparse PR, and standard PR is provided in Table \ref{compare_assu}.




\subsection{Organization} We present our algorithm and guarantee  along with a detailed discussion of the novelty of our proof techniques in Sec. \ref{sec:ph_co_lrmr} given next.
The  overall proof is given in Sec. \ref{proof_main}. The lemmas introduced in Sec. \ref{proof_main} are proved in Appendix \ref{proof_lems}. Numerical experiments are provided in Sec. \ref{sec:expts}.
We develop extensions to phaseless subspace tracking in Sec. \ref{sec:pst}. We  conclude in Sec. \ref{conclude} with a detailed discussion of ongoing and future work.

\section{Low Rank PR: Algorithm and Guarantee} \label{sec:ph_co_lrmr}

\subsection{AltMinLowRaP algorithm} \label{altminlowrap}
The complete algorithm is summarized in Algorithm \ref{lrpr_th}. We explain its main idea next and then explain the details.

\subsubsection{Main idea} \label{altminlowrap_algo}
AltMinLowRaP minimizes the following
\begin{align} \label{optprob}
	\sum_{k=1}^q \| \  \y_k -  |\A_k{}' \checkU \checktb_k| \ \|^2  
\end{align}
alternatively over $\checkU, \checktB$ with the constraint that $\checkU$ is a basis matrix. To initialize, we develop a spectral initialization for $\Span(\Ustar)$ explained below. At a top level, the alternating minimization (AltMin) can be understood as alternating PR: minimize \eqref{optprob} over $\checktB$ keeping $\checkU$ fixed at its current value and then vice versa. But there are important differences between the two PR problems and how they can be solved.
\bi
\item  Given an estimate of $\Span(\Ustar)$, denoted $\U$, and assuming that $\U$ contains orthnormal columns and is independent of the measurement vectors, the recovery of each $\tb_k$ is an $r$-dimensional ``standard PR'' problem. We can use either of \cite{twf,rwf} to solve it. The estimate that we get, denoted $\bhat_k$, is actually an estimate of $\g_k:=\U{}'\Ustar \tb_k$ which is a rotated version of $\tb_k$. If $\SE(\U,\Ustar)\le \delta$, then, by the noisy PR result of \cite{rwf}, we can see that $\|\bhat_k - \g_k\| \le C \delta \|\bstar_k\|$.

\item Given a previous estimate of $\tB$, the update of $\Ustar$, or equivalently of its vectorized version, $\Ustar_{vec}$, is a  significantly non-standard PR problem for two reasons. First, the ``measurement vectors'' for this PR problem are no longer independent or identically distributed. Second, and more importantly, by using the previous estimates of $\Ustar$ and of $\tb_k$, with accuracy level $\delta$, we can get an estimate, $\xhat_k = \U \bhat_k$, of $\xstar_k$ with the same accuracy level. With this, we can also get  an estimate of the phase/sign of the measurements, $\cb_\ik:= \mathrm{phase}(\a_\ik{}' \xstar_k)$ with the same accuracy level. As a result, obtaining a new estimate of $\Ustar_{vec}$ becomes a much simpler Least Squares (LS) problem rather than a PR problem.

\item We stress here that an argument similar to the above {\em does not apply} when recovering $\tb_k$'s. The reason is, with a new estimate of $\Ustar$, denoted $\U^+$, the previous estimate of $\tb_k$ becomes useless: (i) it is close to $\U'\Ustar \tb_k$, and not to $\U^+{}' \Ustar \tb_k$; and (ii) we only estimate the span of the columns of $\Ustar$ accurately, so $\U$ or $\U^+$ are close to $\Ustar$ (and hence to each other) only in the subspace error $\SE$ (and not in spectral or Frobenius norm). As a result  an estimate of the form $\U^+ \bhat_k$ cannot be shown to be close to $\xstar_k$.\footnote{To understand this point easily, suppose both $\U$ and $\U^+$ are perfect estimates of $\Ustar$ in terms of the subspaces they span.  Suppose $\U^+ = \Ustar  \R_2$ and $\U = \Ustar  \R_1$ where $\R_1,\R_2$ could be any rotation matrices. Then it is easy to see that $\|\xstar_k - \U^+ \bhat_k\|  \ge \|(\U  - \U^+) \bhat_k\|  - \|\xstar_k - \U \bhat_k\| \ge \| \Ustar( \R_1- \R_2) \bhat_k \| - \delta \|\xstar_k\|$. Since $\R_1, \R_2$ can be any rotation matrices, e.g, one could have $\R_2 = -\R_1$; in this case, the above error is lower bounded by $2\|\bhat_k\| -  \delta \|\xstar_k\| \ge (2-  C\delta) \|\xstar_k\|$.  The last inequality follows since $\|\bhat_k\| = \|\xhat_k\|$ and $\dist(\xstar_k, \xhat_k) \le  C \delta\|\xstar_k\|$ by Lemma \ref{lem:bounding_distb}.
	The error can thus be  even higher than using a zero vector to estimate $\xstar_k$.
}.
\ei

\subsubsection{Details} \label{altminlowrap_algo_details}
A different way to understand AltMinLowRaP is to split it into a three-way AltMin problem over $\Ustar$, $\tb_k$'s, and $\cb_\ik$'s.  This discussion assumes ``sample-splitting'': a new set of $mq$ measurements is used for each update of $\Ustar$ and another new set for each update of $\tB$. Thus the total number of measurements used is $2mq$ times the number of iterations.
\ben
\item
At each new iteration, we first obtain a new estimate of $\Ustar$ using previous estimates of $\tb_k$'s and of the measurements' phases  $\cb_\ik$'s. This is a LS problem, see line 10 of Algorithm \ref{lrpr_th}.  The output of the LS step may not have orthonormal columns; this is easily resolved by a QR decomposition step after it (line 11).

\item Given a new estimate of $\Ustar$, we recover each $\tb_k$, by solving easy individual $r$-dimensional standard PR problems. These are easy because $m$ of order $r$ suffices.

\item Given a good estimate of $\tb_k$ and of $\Ustar$, we can get an equally good estimate of $\xstar_k$ and hence of the  signs/phases $\cb_\ik$'s.
\een

Consider the PR step to update $\tb_k$'s. Observe that we can rewrite $\y_\ik$ as $\y_\ik = |\langle \a_\ik{}, \Ustar \tb_k \rangle | = | \langle ({\Ustar}'\a_\ik), \tb_k \rangle |$.  If $\Ustar$ were known, we would have a noise-free standard PR problem. If, at the $t$-the iteration, instead, we have a good  estimate of $\Span(\Ustar)$, denoted $\U^t$, we can still recover the $\tb_k$'s by solving a noisy version of the same problem. Due to sample splitting, $\U^t$ is independent of $\a_\ik$'s and so the design vectors $({\U^t}'\a_\ik)$ are still iid standard Gaussian. Any standard PR solution can be used, here we use Reshaped Wirtinger flow (RWF) \cite{rwf}. The noise seen by RWF is proportional to $\SE(\U^t, \Ustar)$. The error in the output of RWF cannot be lower than this value \cite{rwf}. Thus, one needs to use just enough iterations of RWF so that the error at the end of the final RWF iteration is proportional to $\SE(\U^t, \Ustar)$. Since we prove geometric convergence of $\SE(\U^t, \Ustar)$, we can let $T_{RWF,t}$ grow linearly with $t$.

\subsubsection{Initialization}
To obtain the initialization, we develop a careful modification of the truncated spectral initialization idea from \cite{twf,lrpr_tsp}. First assume that $r$ is known. We initialize $\Uhat$ as the top $r$ left singular vectors of the following matrix:
\beq \label{def_YU}
\Y_U = \frac{1}{mq} \sum_{k=1}^q \sum_{i=1}^m \y_{ik}^2 \a_{ik}\a_{ik}' \indic_{ \left\{ \y_{ik}^2 \leq C_Y \frac{1}{mq}\sum_{ik} \y_{ik}^2 \right\}  }.
\eeq
where $C_Y$ is a constant that decides the truncation threshold (which measurements are too large in magnitude compared to the mean energy of the measurements and should be discarded). For our guarantee, we set it equal to $9 \kappa^2 \mu^2$. In practice a good value can be chosen by experimentation and cross-validation or by using the ideas in \cite{coherence_est}. An alternative approach is to set it differently for each $k$ as done in \cite{lrpr_tsp}. This approach does not require knowledge of $\kappa$ or $\mu$, but it results in a worse lower bound on just $m$.  We discuss the effect of this choice in Remark \ref{diff_k_thresh} and also in Sec. \ref{our_prev_details}.

To understand why the above approach works, first consider the above matrix {\em with the indicator function removed}. Then it is not hard to see that its expected value equals $(1/q) [\Ustar ({\bm\Sigma^*}^2) {\Ustar}' + 2 \trace({\bm\Sigma^*}^2) \I]$, and so its span of top $r$ singular vectors equals $\Span(\Ustar)$. Hence, with large enough $mq$, the same should approximately hold for the original matrix. However, when using $\Y_U$ with the indicator function removed, a few ``bad'' measurements (those with very large magnitude $\y_\ik^2$ compared to their empirical mean over $i,k$) can heavily bias its value.
To mitigate this effect, and get a good initialization in spite of it, we will need a larger value of $mq$. Using the indicator function helps truncate the summation to only sum over the ``good'' measurements, and as a result a smaller value of $mq$ suffices. Mathematically, this helps ensure that $\Y_U$ is close to a matrix that can be written as $\sum_{ik} \w_\ik  \w_\ik{}'$ with $\w_\ik$'s being iid sub-Gaussian vectors (instead of sub-exponential in the case without truncation) \cite{twf}.

We can also use $\Y_U$ to correctly estimate $r$ whp by using the fact that, when $m$ and $q$ are large, the gap between its $r$-th and $(r+1)$-th singular value is close to $\sigmin^2/q$.  With this idea, we estimate $r$ as given in the first step of Algorithm \ref{lrpr_th}. As explained in \cite{lrpr_tsp}, another way of estimating the rank is to set $\hat{r} = \arg\max_j (\sigma_j(\Y_U)-\sigma_{j+1}(\Y_U) )$. This approach does not require knowledge of any model parameters. Hence it is easily applicable for real data (even without training samples being available). However, it works under the assumption that consecutive nonzero singular values of $\X^*$ are close (do not have significant gap), see \cite[Corollary 3.7]{lrpr_tsp} for one precise statement of this claim.

The main idea of the initialization step explained above was first developed in our previous work \cite{lrpr_tsp}, we explain the difference later in Sec. \ref{our_prev_details}. Also, as pointed out by an anonymous reviewer, a matrix that is related to $\Y_U$ was used in earlier work \cite{hughes_icip_2012,hughes_icml_2014} to try to solve what can be called the linear version of our problem; see Sec. \ref{our_linear_details}. We were not aware of this work when developing our approach. Our approach was developed independently in \cite{lrpr_tsp} as a modification of the truncated spectral initialization idea from PR literature \cite{twf} and then modified here.

We summarize the complete algorithm in Algorithm \ref{lrpr_th}. As is commonly done in existing literature, e.g., see \cite{lowrank_altmin,pr_altmin}, in order to obtain a provable guarantee in a simple fashion, we assume sample-splitting. Since we prove geometric convergence of the iterates, this increases the required sample complexity by a factor of only $\log(1/\epsilon)$. In our empirical evaluations, we reuse the same set of measurements.

\begin{algorithm}[t!]
	{ \caption{AltMin-LowRaP: Alt-Min for Phaseless Low Rank Recovery}
		\label{lrpr_th}
		%
		\begin{algorithmic}[1]
			\STATE Parameters: $T$, $T_{RWF,t}$, $\omega$.
			\STATE  Partition the $m_\tot$ measurements and design vectors  for each $\xstar_k$ into one set for initialization and $2T$ disjoint sets for the main loop.
			\STATE Set $\hat{r}$ as the largest index $j$ for which $\lambda_j(\Y_U) - \lambda_n(\Y_U) \ge \omega$ where $\Y_U$ is in \eqref{def_YU}.
			
			\STATE  $\U^0 \gets \Uhat^0 \gets$ top $\hat{r}$ singular vectors of $\Y_U$ defined in \eqref{def_YU}.   \hspace{5cm} \COMMENT{{\color{red} Initialize $\U$}}
			
			\FOR{$ t = 0: T$}
			\STATE $\bhat_k^t \gets \mathrm{RWF}( \{ \y_k^{(t)}, \U^{t}{}' \A_k^{(t)} \}, T_{RWF,t})$ for each $k =1,2, \cdots, q $ \  {\em ($\mathrm{RWF}$: Reshaped WF \cite{rwf})}.  \hspace{5cm}\COMMENT{{\color{red} Update $\hat{\B}$} }
			
			\STATE   $\xhat_k^t \gets \U^t \bhat_k^t$ for each $k =1,2, \cdots, q $. 
			\STATE $\Chat_k \gets \mathrm{Phase}\left(\A_k^{(T+t)}{}'\xhat_k^t \right)$ for each $k =1,2, \cdots, q $ .  \hspace{5cm}\COMMENT{{\color{red} Update $\Chat_k$'s}}
			
			\STATE Get $\B^t$ by QR decomp: $\hat{\B}^t \qreq \R_B^t \B^t $. 
			\STATE   $\Uhat^{t+1} \leftarrow \arg\min_{\tilde\U} \sum_{k =1}^q\| \Chat_k \y_k^{(T+t)} - \A_k^{(T+t)}{}' \tilde\U \b_{k}^{t}\|^2$.
			\STATE Get $\U^{t+1}$ by QR decomp: $\Uhat^{t+1} \qreq \U^{t+1} \R_U^{t+1} $ . \hspace{5cm}
			\COMMENT{{\color{red} Line 10, 11: Update $\U$}}
			\ENDFOR
		\end{algorithmic}
	}
\end{algorithm}

\subsection{Guarantee} \label{sec:main_res1}
We have the following guarantee.
\begin{theorem}
	\label{thm:main_res} 
	Consider Algorithm \ref{lrpr_th}. 
	Assume that the $\y_\ik$'s satisfy \eqref{eq:problem} with $\a_\ik$ being iid standard Gaussian; and $\Xstar$ is an $n \times q$ rank-$r$ matrix that satisfies right-incoherence with parameter $\mu$. 
	Set $T := C\log(1/\epsilon)$, $T_{RWF,t} = C (\log r + \log \kappa + t (\log(0.7)/\log(1-c)))$, $\omega = 1.3 \sigmin^2/q$, and $C_Y = 9\kappa^2 \mu^2$ in \eqref{def_YU}.
	Assume that, for the initialization step and for each new update, we use a new set of $m$ measurements with $m$ satisfying $mq \ge C \kappa^{12} \mu^4 \cdot nr^4$ and $m \ge C \max(r, \log q, \log n)$.
	Then, with probability (w.p.) at least $ 1- C n^{-10}$,
	\[
	\SE(\U^*,\U^T) \le \epsilon, \ \matdist(\hat\X^T, \X^*) \leq \epsilon \|\X^*\|_F
	\]
	and $\dist(\xhat_{k}^T, \x_{k}^*) \leq \epsilon \|\x_k^* \|$ for each $k$. Moreover, after the $t$-th iteration,
	\begin{align*}
		& \SE(\U^*, \U^{t}) \le 0.7^{t}  \deltinit, t=0,1,2,\dots, T
	\end{align*}
	where $\deltinit = \frac{c}{\kappa^2 r}$. Similar bounds also hold on the error in estimating $\xstar_k$s.
	The time complexity is $m q nr \log^2(1/\epsilon)$. 
\end{theorem}

Proof: We prove this in Sec. \ref{proof_main}.

\begin{remark}\label{diff_k_thresh}
	If we are willing to tolerate another lower bound on just $m$ of $m \ge C  r^4$, we can improve the dependence on $\kappa$, $\mu$ to $\kappa^8 \mu^2$. This will require the following change: define $\Y_U$ as done in \cite{lrpr_tsp}: use $9 \sum_{i=1}^m \y_\ik^2 / m$ as the threshold inside the indicator function. A second advantage of using this is that it allows one to use $C_Y=9$ instead of $C_Y$ that depends on $\kappa,\mu$. The disadvantage is of course a more stringent lower bound on $m$.
\end{remark}

\begin{remark}
	To understand the time complexity, observe that the most expensive step at each iteration is the update of $\Ustar$. This requires solving an LS problem of recovering an $nr$ length vector from $mq$ measurements. One can solve this by conjugate gradient descent with a cost of $mq \cdot nr \cdot  \log(1/\epsilon)$ \cite{pr_altmin}. This times the total number of iterations, $T= C \log (1/\epsilon)$ gives the complexity of the algorithm after initialization.  Consider the initialization step. Observe that the top $r$ singular vectors of $\Y_U$ are also the top $r$ singular vectors of an $mq \times n$ matrix $\G$ whose columns are given by $\a_\ik \y_\ik $ times the indicator function used in \eqref{def_YU}. Thus $\Y_U = \G \G'$. Since computing the $r$-SVD of an $a \times b$ matrix to $\delta$ accuracy needs time of order $ab r \log (1/\delta)$  \cite{robpca_nonconvex}, the SVD needed for initialization  can be computed in time  $C mq \cdot n \cdot r \cdot \log (1/\deltinit) = C mq nr \log r$ where $\deltinit$ is the error level to which the initialization needs to be accurate.  As explained later, $\deltinit = c/r$ suffices.
\end{remark}

Theorem \ref{thm:main_res} implies that one can achieve geometric convergence as long as the sample complexity $m_\tot:= (2T+1)m$ satisfies $m_\tot q \ge C \kappa^{12} \mu^4 nr^4 \log(1/\epsilon)$ along with $m_\tot  \ge C \max(r,\log q, \log n) \log(1/\epsilon)$. The second lower bound is very small and essentially redundant\footnote{We need this lower bound because we recover the $q$ $\tb_k$'s individually by solving a standard PR problem for each. This step works correctly w.p. at least $1-2q \exp(r - c m)$.} except when $q \ge C nr^4$.

Notice that the LRPR sample complexity is significantly better than that of standard (unstructured) PR methods which necessarily need $m = C n$ samples per signal (matrix column).  For fixed $m$ and $q$, LRPR time complexity is about $r$ times worse than that of standard PR. But, if we use the smallest value of $mq$ needed by each method to get an $\epsilon$-accurate estimate, AltMinLowRaP is actually faster when $r$ is small: its needs time of order $n^2 r^5 \log(1/\epsilon)$ while standard PR methods need time of oder $n^2 q$. We demonstrate this fact experimentally in Fig \ref{fig:time_comp}.

The minimum number of samples needed to recover an $n\times q$ matrix of rank $r$ is $(q+n)r$. Thus, in general, our sample complexity is $r^3 \log(1/\epsilon)$ times worse than its order-optimal value. As noted earlier, in problem settings like ours, where the measurements are not global, non-convex algorithms typically do need more than the order-optimal number of samples. 
Since neither our problem nor its linear version have any complete provable guarantees for correct recovery in existing work, LRMC is the closest problem to ours that has been extensively studied and that also uses non-global measurements. Sparse PR is the other somewhat related problem to ours but it is easier because it involves recovery from global measurements of a sparse vector. Table \ref{compare_assu} provides a comparison of our guarantee with the first and best results for both problems. More details are in Sec. \ref{rel_work}.
As can be seen from the table, our sample complexity compares favorably with that for the first  non-convex solution for LRMC that has bounded time complexity (bounds the required number of iterations) \cite{lowrank_altmin}. In the practical regime of $r$ being order $\log n$, it even compares with the best iterative LRMC result \cite{rmc_gd}. 
Also, the the first guarantee for iterative solutions to both problems is sub-optimal (needs more samples or more assumptions) compared to the best one that appeared later.%

We should reiterate that, (i) in many practical applications, a small  value of $r$ suffices, e.g., we used $r=5$ for images with $n=32400$ in \cite{icip20} followed by a few iterations of model error correction via standard PR; and (ii) low rank is a more flexible model for dynamic imaging than sparsity. Also see Tables \ref{tab:vid}  and  \ref{tab:vid_sp}.

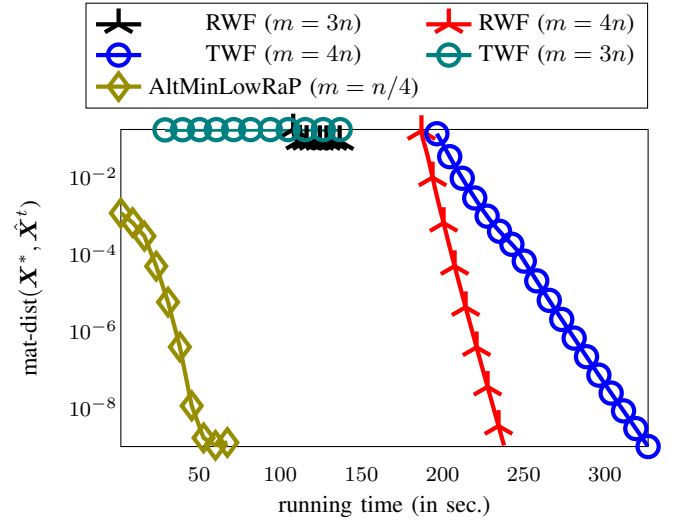
\begin{figure}[t!]
	\begin{center}
		\begin{tikzpicture}
		\begin{groupplot}[
		group style={
			group size=1 by 1,
			x descriptions at=edge bottom,
			y descriptions at=edge left,
		},
		my stylecompare,
		width=.97\linewidth,
		height=5.8cm,
		enlargelimits=false,
		]
		\nextgroupplot[
		legend entries={
			RWF ($m=3n$),
			RWF ($m=4n$),
			TWF ($m=4n$),
			TWF ($m=3n$),
			AltMinLowRaP ($m = n/4$)
		},
		legend style={at={(1,1.4)}},
		legend columns = 2,
		legend style={font=\small},
		ymode=log,
		xlabel={\small{running time (in sec.)}},
		ylabel={\small{$\matdist(\Xstar, \hat\X^t$)}},
		xticklabel style= {font=\footnotesize, yshift=-1ex},
		yticklabel style= {font=\footnotesize},
		]
		\addplot [black, line width=1.6pt, mark=Mercedes star,mark size=6pt, mark repeat=10, select coords between index={0}{70}] table[x index = {2}, y index = {3}]{\timecomp};
		\addplot [red, line width=1.6pt, mark=Mercedes star,mark size=6pt, mark repeat=10, select coords between index={0}{75}] table[x index = {0}, y index = {1}]{\timecomp};
		\addplot [blue, line width=1.6pt, mark=o,mark size=4pt, mark repeat=10, select coords between index={0}{170}] table[x index = {4}, y index = {5}]{\timecomp};
		\addplot [teal, line width=1.6pt, mark=o,mark size=4pt, mark repeat=10, select coords between index={0}{100}] table[x index = {8}, y index = {9}]{\timecomp};
		\addplot [olive , line width=1.6pt, mark=diamond,mark size=5pt, select coords between index={0}{9}] table[x index = {6}, y index = {7}]{\timecomp};
		\end{groupplot}
		\end{tikzpicture}
	\end{center}
	\vspace{-.4cm}
	\caption{\small{Recovery error versus time-taken plot with time in seconds.  The ``time'' here is the time taken in seconds to reach a certain error level. We generate this plot as explained in Sec. \ref{simdata}: plot the time taken until end of iteration t on the x-axis and plot the error at the end of iteration t on the y-axis.
			{\em This figure illustrates the fact that if we use the lowest allowed value of $m$ for each approach, AltMinLowRaP is in fact faster than RWF or TWF, for any given value of desired recovery error.}
			Notice that RWF and TWF fail for $m=3n$, but work for $m=4n$. AltMinLowRaP works with just $m=n/4$ and with this value of $m$ it is roughly 5-times faster than both RWF (with $m=4n$) and TWF (with $m=4n$) for any desired level of error $\epsilon$. 
	}}
	\label{fig:time_comp}
\end{figure}



\subsection{Linear version of our problem: Compressive PCA}\label{our_linear_details}
Consider the linear (with phase) version of our problem -- recover a low rank matrix $\Xstar$, or its column span, from
\[
\y_\ik := \langle \a_\ik, \xstar_k \rangle, \  i=1,2,\dots,m, \ k=1,2,\dots,q.
\]
This problem can also be understood as ``compressive PCA'' although compressive PCA typically allows $\Xstar$ to also be only approximately low rank \cite{hughes_icip_2012,hughes_icml_2014,aarti_singh_subs_learn}.
Clearly, both our algorithm and our guarantee, Theorem \ref{thm:main_res}, directly apply to this simpler special case as long as the $\a_\ik$'s are iid Gaussian. Even when phases are available one could use the magnitude measurements and solve our harder problem instead. We can state the following corollary. 
\begin{corollary}
	Consider the problem of recovering an $n \times q$ rank-$r$ matrix $\Xstar$ from $\y_\ik := \langle \a_\ik, \xstar_k \rangle, i=1,2,\dots,m, \ k=1,2,\dots,q$. As long as right incoherence given in \eqref{right_incoh_1} holds, one needs $m \ge C \frac{n}{q} r^4 \log(1/\epsilon)$ to recover $\Xstar$ and its column span to $\epsilon$ accuracy (precisely defined in Theorem \ref{thm:main_res}).
\end{corollary}
Thus our work also provides a simple and fast AltMin solution that provably converges geometrically to the compressive PCA solution. 
Of course, when the phase/sign is known, a simpler version of our algorithm should work and a better guarantee should be obtainable. We discuss this further in the conclusions' section.

In terms of existing related work,  \cite{hughes_icip_2012} provided an approach for solving the above problem when $\Xstar$ is only approximately low-rank. A follow-up paper \cite{hughes_icml_2014} studied its modification where $\a_\ik$'s are sparse random vectors, e.g., ``sparse Bernoulli" as they call it (each entry takes values $-1, 0, 1$ with probabilities $1/2s,1-1/s,1/2s$). 
Their proposed approach is related to our initialization step: one computes an estimate of the principal subspace of the column span of $\Xstar$ by computing the top $r$ singular vectors of a matrix that is similar to $\Y_U$ defined in \eqref{def_YU}. The difference is that the indicator function is absent, and they also sum over terms of the form $\a_{ik} \y_{ik}\y_{jk} \a_{jk}{}' $ for all $j \neq i$ which, in this linear setting, have nonzero expected value. To be precise, they compute the top $r$ singular vectors of  $\sum_k \A_k \y_k \A_k{}' \y_k{}'$. Denote such a matrix by $\Y_{U,modified}$. 
These works treat the columns $\xstar_k$ as random vectors and assume that they have {\em nonzero mean} and provide a simple intuitive approach to estimate the mean. This is of course only possible because the measurements are linear, and cannot be done in our setting. 
Both works \cite{hughes_icip_2012,hughes_icml_2014} show that the expected value of $\Y_{U,modified}$ is equal to $c_1 \I + c_2 \E[\Xstar \Xstar'/q]$, thus, the span of top r eigenvectors of the expected value equals the desired subspace: principal subspace of the population covariance, $\E[\Xstar \Xstar'/q]$.
They also claim that as $q \rightarrow \infty$, $\Y_{U,modified}$ converges to its expected value, but do not provide a rate of convergence. Finally, they also upper bound the expected value of the Frobenius norm of the error between $\Y_{U,modified}$ and its expected value for any value of $q$. 
Since this result also only bounds the expected value of the error, it cannot be used to get any useful information about the sample complexity $m$ that is required even for the initialization step. Moreover, of course these works do not provide any guarantee on how to recover the entire matrix $\Xstar$.
%
%

A later work in Asilomar 2014 \cite{aarti_singh_subs_learn} attempted to solve a generalization of the above problem: it provided an approach (which is again somewhat related to only our initialization step) and a guarantee for recovering the top $r$ singular vectors of a general matrix $\Xstar$. Specialized to the exactly rank $r$ case, their result proves that, if each column of $\Xstar$ is bounded, and if
$m \ge \frac{1}{\epsilon} \max \left(\frac{n}{\sqrt{q}} \sqrt{r}\kappa,   \frac{n}{q} r^2 \kappa^2 \right) $, one can obtain an $\epsilon$-accurate recovery of the subspace $\Span(\Ustar)$.  This is a much weaker result than ours: (i) its sample complexity $m$ depends on $1/\epsilon$ instead of on $\log(1/\epsilon)$; and (ii) $m$ needs to grow as $n/\sqrt{q}$ instead of as $n/q$. The reason it is weaker of course, is because they  used a single step approach instead of an iterative algorithm.



\subsection{Discussion of Proof Techniques and Reason for Worse Dependence on $\kappa, r$} \label{discuss}

\subsubsection{Proof techniques}   
Since even the linear version of our problem has not been studied theoretically (except  for a convex solution for the $\a_\ik=\a_i$ case which is a significantly different problem), it is not possible to directly modify proof techniques from existing work. We do borrow some ideas from LRMC \cite{lowrank_altmin} or from standard PR results \cite{rwf,twf}. But, as explained earlier in Sec. \ref{novelty}, concentration bounds need to applied in a significantly different way than for either of these problems.
%
The algebra for obtaining an expression that bounds the subspace recovery error between $\Uhat^t$ and $\Ustar$ uses the overall approach of \cite{lowrank_altmin}. After this, the details are different because, for LRMC, the random variables are Bernoulli, while in our case they are unbounded sub-Gaussian or sub-exponential.
Also, in LRPR, the  sign/phase are unknown and this introduces an extra term that needs to be bounded, see $\mathrm{Term2}$ defined in Lemma \ref{lem:key_lem}. To bound this we first use Cauchy-Schwarz to bound it by a product of two terms. One term is easy to deal with. For the second term, we borrow a lemma from the RWF paper \cite{rwf} on standard PR, but the rest of our approach is different because of the need to prove concentration of a set of $mq$ measurements that are not identically distributed (we need to carefully exploit right concerence to ensure that they are ``similar enough'' to apply the concentration bounds for sums of products of sub-Gaussians).  

Our initialization step uses the truncated spectral initialization approach, first introduced for PR in \cite{twf}. The proof for it also uses the overall approach of \cite{twf} but there are many important
differences in proving the concentration bounds (see above).  
Our initialization is also significantly different from that of LRMC or LRMS. In these cases, the phase is known and thus one can come up with a matrix whose expected value equals $\Xstar$. This is not possible for LRPR. The matrix  whose top $r$ singular vectors we compute has expected value $2 \beta_2 \Xstar \Xstar' + \beta_1 \I$. This point has important implications for the sample complexity dependence on $\kappa,r$, we discuss this next.


\subsubsection{Worse dependence on $\kappa$, $r$}
Our sample complexity  is comparable to that of \cite{lowrank_altmin} in terms of its dependence on $n$ and $r$. However, our result has a worse dependence on $\kappa$ because we have access only to phaseless measurements. Because of this, (1) our initialization step  needs to use the matrix $\Y_U$ and find its top $r$ eigenvectors in order to get an initial estimate of the column span $\Ustar$. For simplicity, consider $\Y_U$ without the truncation (without the indicator function). Then, its expected value is $\Ustar {\bm\Sigma^*}^2 \Ustar{}' + 2 \trace({\bm\Sigma^*}^2) \I$. Notice that the condition number of the first term of this matrix (the term of interest) is $\kappa^2$. Because of this, when analyzing this step, we end up with a dependence of $mq$ on $\kappa^8 \mu^4 nr^2 /\deltinit^2$ (see Claim \ref{lemm:bounding_U} and Sec. \ref{proof_init_lem} where this claim is proved). Here $\deltinit$ is the subspace error bound for the initialization step.
Instead, the expected value of the matrix used for initialization of AltMinComplete \cite{lowrank_altmin} has expected value equal to $\Xstar$ and thus its condition number is just $\kappa$. (2) A second issue is as follows: because of magnitude-only measurements, we are having to deal with phase error (sign error) in each LS step that updates the estimate of $\Ustar$. In bounding the phase error term -- $\mathrm{Term2}$  defined in Lemma \ref{lem:key_lem} in Sec. \ref{proof_claim} -- we need to use the Cauchy-Schwarz inequality; see proof of Lemma \ref{Show}. Because of this, when using the bound on $\mathrm{Term2}$ from Lemma \ref{Show} to prove the main descent claim, Claim \ref{lem:descent}, we end up with a bound of the form $\SE(\U^{t+1}, \Ustar) \le C \delta_t (\epsilon_1 +  \sqrt{\deltapt + \epsilon_2}) \sqrt{r} \kappa$ where $\deltapt$ is the bound on the subspace error from the previous step and $\epsilon_i$ are quantities used in the concentration bounds for Term1 and Term2 respectively. Thus, to ensure that the $(t+1)$-th step error is below $0.7 \deltapt$ (decays geometrically), we need $\deltapt \le c / (\kappa^2 r)$ for each iteration $t$. Since we show $\deltapt \le \deltinit$ for all $t$, this is ensured if  $\deltinit =  c / (\kappa^2 r)$.
We also need $\epsilon_1 \le c/ kappa \sqrt{r}$ and $\epsilon_2 \le c / (\kappa^2 r)$. 
This, along with the fact that the initialization step needs $mq \ge \kappa^8 \mu^4 nr^2 /\deltinit^2$, implies that our sample complexity per iteration becomes $C \kappa^{12} \mu^4 nr^4$.

If Cauchy-Schwarz were not used, we would only need $mq \ge C \kappa^{10} \mu^4 r^3$. Similarly, if we somehow did not use the loose bound $\|\X\| \le \|\X\|_F$ for rank $r$ matrices at a few different places in the proof, we could remove another factor of $r$.

\subsubsection{Our previous work} \label{our_prev_details}
The only other work that also studies our problem is our previous work \cite{lrpr_tsp}. This introduced a series of heuristics and evaluated them experimentally. It also provided a guarantee for the initialization step of one of them. If we compare their main result (their Theorem 3.2) with ours, it required the following lower bound on just $m$: $m \ge C \max(\sqrt{n},r^4)/ \epsilon^2$ in addition to a lower bound on $mq$ that also depends on $1/\epsilon^2$.  
We remove the $1/\epsilon^2$ dependence by analyzing the complete algorithm. 

The first two requirements on just $m$ are also significantly relaxed in our work because we study a significantly modified version of our previous algorithm.
The most important algorithmic difference is that, both for initialization and for later iterations, we recover $\tb_k$'s by solving the standard PR problem either fully (or, at least for enough iterations so that the error in recovering $\tb_k$'s is of the same order as the subspace error in the estimate of $\Ustar$). This is what allows us to replace the strong requirement $m \ge C \sqrt{n}$ that \cite{lrpr_tsp} needed by just $m \ge C r$. This is also what enables us to get a complete guarantee for the entire algorithm. The algorithm in \cite{lrpr_tsp} used only one iteration of AltMinPhase \cite{pr_altmin}  for obtaining a new estimate of $\tb_k$'s using a new estimate of $\Ustar$. With this, it was not possible to show that the recovery error of $\tb_k$'s is of the same order as that of $\Ustar$.

Our approach for initializing $\Ustar$ is taken from \cite{lrpr_tsp}, but with a simple, but important, difference: the threshold in the indicator function used for defining $\Y_U$ in \eqref{def_YU} now takes an average over all $mq$ measurements (instead of over only the $m$ measurements of the $k$-th column in \cite{lrpr_tsp}). This simple change allows us to use concentration over all the $mq$ measurements (and design vectors) in every step of deriving the initialization guarantee for $\Ustar$. This is what helps us eliminate the requirement of $m \ge C r^4$ on just $m$ that was needed in \cite{lrpr_tsp}.

\section{Proof of Theorem \ref{thm:main_res}}\label{proof_main}

The proof borrows ideas from past works -- \cite{twf,lrpr_tsp} (for initialization of $\Ustar$), \cite{lowrank_altmin} (the overall approach for getting a subspace error bound given in the Appendix), \cite{candes2009tight} for careful $\epsilon$-net arguments for unit Frobenius norm matrices, and  \cite{rwf} (for recovering $\tb_k$'s, and in one step of trying to show that the phase error is small). We cite the relevant reference again where it is used.
In Sec. \ref{sec:key_lem} next, we provide the two main claims (one for initialization and one for the descent), the two other auxiliary lemmas needed for proving Theorem \ref{thm:main_res} and the theorem's proof  using these. In Sec. \ref{proof_init_lem}, we give the key lemmas needed for proving the initialization claim and also prove it. The same is done for the descent claim in Sec. \ref{proof_claim}. Each of these subsections also provides the main ideas (intuition) used for proving the lemmas.
We then prove all but one of the lemmas from this entire section in Appendix \ref{proof_lems}. Lemma \ref{lem:key_lem} which uses the overall approach of \cite{lowrank_altmin} is proved in Appendix \ref{proof_key_lem}.%

\subsection{Overall lemmas and proof of Theorem \ref{thm:main_res}} \label{sec:key_lem}

\begin{claim}[Rank estimation and Initialization of $\U^*$]
	\label{lemm:bounding_U}
	Let $\U_\init = \Uhat^0$.
	Pick a $\deltinit < 0.25$.
	Assume $mq \ge \kappa^8 \mu^4 nr^2 /\deltinit^2 $. 
	Set the  rank estimation threshold $\omega = 1.3 \sigmin^2 / q$ (we can actually set the multiplier to any number between $0.025$ and $1.5$).
	Then, w.p. at least $1 - 6 n^{-10}$, the rank is correctly estimated and
	\begin{align*}
		\SE(\U_\init,\Ustar) \leq \deltinit.
	\end{align*}
\end{claim}

Define
\begin{eqnarray}
\g_k^t := (\U^t)' \xstar_k \  \text{and} \  \e_k^t := (\I - \U^t {\U^t}') \xstar_k.
\label{def_g}
\end{eqnarray}
It is easy to see that $\xstar_k  = \U^t \g_k^t + \e_k^t$ and so $\y_\ik = |( (\U^t){}'\a_\ik)'\g_k^t + \a_\ik{}'\e_k^t|$. Thus, we have a noisy PR problem to solve with the noise magnitude proportional to $\|\e_k^t{}\|$ with $\|\e_k^t{}\| \le \SE(\U^t, \Ustar)\|\xstar_k\|$. We use RWF to solve it.
RWF provides an estimate, $\bhat_k^t$, of $ \g_k^t$. Observe that
\[
\g_k^t =  (\U^t)' \xstar_k = ( (\U^t)'\Ustar) \tb_k
\]
is just a rotated version of $\tb_k$. We show in the next lemma that, whp, the error in the RWF estimate, $\dist(\g_k^t,\bhat_k^t)$, is proportional to $\SE(\U^t, \Ustar)$; and the same is true for the error in $\xhat_k^t := \U^t \bhat_k^t$.

\begin{lemma}[Recovery of $\tb_k$'s]\label{lem:bounding_distb}
	At iteration $t$, assume that $\SE(\Ustar, \U^t) \le \deltapt$.
	Pick a $\delta_b<1$. If $m \geq C r$, and if we set $T_{RWF,t} = C \log \deltapt/ \log(1-c)$,
	then, w.p. at least $1 - 2q \exp\left( -c \delta_b^2 m \right)$, the following is true for each $k = 1, 2, \cdots, q$
	\bea
	\dist\left( \g_k^{t}, \bhat_k^{t} \right) 
	& \le & C \deltapt \|\tilde{\b}_k^*\| = C \deltapt \|\xstar_k\|, \nonumber \\
	\matdist(\G^{t} , \hat\B^{t}) & \le & C \deltapt \|\tB\|_F = C \deltapt \|\Xstar\|_F,  \nonumber \\
	\dist(\xhat_{k}^{t}, \x_{k}^*) & \le &  (C+1) \deltapt \|\x_k^* \|. 
	\eea
	with $C = \sqrt{1+\delta_b}+1$.
	
	Thus, if $m \ge C \max(r,\log n, \log q)/\delta_b^2$, then the above bounds hold w.p. at least $1 - n^{-10}$.
\end{lemma}

From above, $\bhat_k^t$ is close to $\g_k^t$ (which is a rotated version of $\tb_k$) for each $k$.
We thus expect $\hat\B^t$, or equivalently $\B^t$,  to also satisfy the incoherence assumption. We show next that this is indeed true if $\deltapt$ is small enough.
Recall that $\hat\B^t \qreq \R_B \B^t$.

\begin{lemma}[Incoherence of $\Bstar$ implies incoherence of $\B^t$]
	\label{incoherencebhat}
	Pick a $\delta_b <1/10$ and assume that $m \ge C \max(r,\log n, \log q)/\delta_b^2$.
	At iteration $t$, assume that $\SE(\Ustar, \U^t) \le \deltapt$ with $\deltapt \leq \frac{0.25}{C \sqrt{r} \kappa}$.
	If $\Bstar$ is $\mu$-incoherent,  then,  w.p. at least $1 - n^{-10}$,
	%
	$\B^t$ is $\hat\mu$-incoherent with $\hat\mu = C \kappa \mu $. 
\end{lemma}

Finally, the next claim shows that the LS step to update $\U$ reduces its error by a factor of $0.7$ at each iteration. Its proof relies on the previous two lemmas and the fact that $\xhat_k^t$ close to $\xstar_k$ implies that, with large probability, the phases (signs) of $(\a_\ik{}'\xhat_k^t)$ and $(\a_\ik{}'\xstar_k)$ are equal too.

\begin{claim}[Descent Lemma] \label{lem:descent}
	At iteration $t$, assume that $\SE(\Ustar, \U^t) \le \deltapt$ and $\deltapt \le \deltinit \le  c/r \kappa^2$.
	If
	$mq \ge C \kappa^6 \mu^2 nr^3$ and $m \ge C \max(r, \log n , \log q)  $
	then  w.p. at least	$1 - C \exp(-nr) - n^{-10}$,
	\[
	\SE( \U^{t+1},\Ustar)  \le 0.7 \deltapt:= \deltaptplus.
	\]
\end{claim}

\begin{proof}[Proof of Theorem \ref{thm:main_res}]
	The $\SE(\U^t ,\Ustar)$ bounds are an immediate consequence of Claims \ref{lemm:bounding_U} and \ref{lem:descent}, along with setting $\delta_\init = c/\kappa^2 r$. 
Claim \ref{lemm:bounding_U} needs $mq \ge C \kappa^8 \mu^4 nr^2 /\deltinit^2 =  C \kappa^8 \mu^4 nr^4$ along with $m \ge C \max(r,\log n, \log q)$ while Claim \ref{lem:descent} needs $mq \ge C \kappa^6 \mu^2 nr^3$ and $m \ge C \max(r, \log n , \log q)$.
The first lower bound on $mq$ dominates. Thus the required sample complexity is $m_{tot} q \ge C \kappa^8 \mu^4 \cdot nr^4 \log(1/\epsilon)$ and $m_{tot} \ge C \max(r,\log n, \log q) \log(1/\epsilon)$.

The other bounds of Theorem \ref{thm:main_res} follow by Lemma \ref{lem:bounding_distb}.%
\end{proof}

We prove Claims \ref{lemm:bounding_U} and \ref{lem:descent} next in Sec. \ref{proof_init_lem} and \ref{proof_claim}. The proof of  Lemmas \ref{lem:bounding_distb} and \ref{incoherencebhat} and of the lemmas needed for proving these two claims is postponed to Appendix \ref{proof_lems}.

\subsection{Proof of Claim \ref{lemm:bounding_U}} \label{proof_init_lem}
In this section, we let $\a_{ik}: = \a_{ik}^{(0)}$ and $\y_{ik}:=\y_\ik^{(0)}$.

The overall idea for proving this is inspired by the approach in \cite{lrpr_tsp} which itself borrows ideas from \cite{twf}. But there are many important differences because we define $\Y_U$ differently in this work, see \eqref{def_YU}: the threshold in the indicator function now takes an average over all $mq$ measurements (instead of over only the $m$ measurements of the $k$-th column as in \cite{lrpr_tsp}). This simple change enables us to get a significantly improved result. It lets us  use concentration over all the $mq$ measurements (and design vectors) in each of the three steps of the proof. This is what helps eliminate the lower bound $m \ge C r^4$ on just $m$ that was needed in \cite{lrpr_tsp}. However, this also means that the proofs are much more involved (more quantities now vary with $k$). 

Recall the expression for $\Y_U$ from earlier, and define  matrices $\Y_- (\epsilon_1)$ and $\Y_+(\epsilon_1)$ as
\begin{align*}
	&\Y_U = \frac{1}{mq} \sum_{ik} |\a_{ik}{}'\x_k^*|^2 \a_{ik}\a_{ik}' \indic_{\left\{|\a_{ik}{}'\x_{k}^*|^2 \leq \frac{9\mu^2 \kappa^2}{mq}\sum_{ik} |\a_{ik}{}'\x_{k}^*|^2 \right\} }\\
	&\Y_-(\epsilon_1) =  \\
	& \frac{1}{mq} \sum_{ik} |\a_{ik}{}'\x_k^*|^2 \a_{ik}\a_{ik}' \indic_{\left\{|\a_{ik}{}'\x_{k}^*|^2 \leq \frac{9\mu^2 \kappa^2(1-\epsilon_1)}{q} \|\X^*\|_F^2  \right\} }.
\end{align*}
Define $\Y_+(\epsilon_1)$ similarly but with $(1-\epsilon_1)$ replaced by $(1+\epsilon_1)$ in the indicator function. We will show that $\Y_U$ is sandwiched between $\Y_-$ and $\Y_+$. This, along with showing that  $\Y_-$ and $\Y_+$ are close, will help us show that $\Y_U$ is close to $\Y_-$ and, hence, also to its expected value. After this, use of the $\sin \theta$ theorem will give us the desired bound.

Adapting the approach of \cite{twf,lrpr_tsp}, 
\begin{align}
	\E\left[\Y_-(\epsilon_1)\right] = \frac{1}{q} \left\{ \sum_k \beta_{1,k}^- \xstar_k {\xstar_k}' + \left(\sum_k \beta_{2,k}^- \|\xstar_k\|^2 \right) \I\right\}
	\label{EYminus}
\end{align}
where
\begin{align*}
	& \beta_{1,k}^-(\epsilon_1):= \E[(\xi^4 - \xi^2) \indic_{\xi^2 \le (1-\epsilon_1)\gamma_k }] \\
	& \beta_{2,k}^-(\epsilon_1):= \E[\xi^2 \indic_{\xi^2 \le (1-\epsilon_1)\gamma_k}], \\
	& \gamma_{k}:= \frac{9 \mu^2 \kappa^2 \|\Xstar\|_F^2}{q \|\xstar_k\|^2},
\end{align*}
and $\xi$ is a scalar standard Gaussian random variable.
The expression for $\E\left[\Y_+(\epsilon_1)\right] $ is similar but with $ (1-\epsilon_1)$ replaced by $ (1+\epsilon_1)$ in the expression for $\beta_{1,k}^+$, $\beta_{2,k}^+$.


Observe that $\E\left[\Y_-(\epsilon_1)\right] $ can be simplified as 
\[
\E\left[\Y_-(\epsilon_1)\right]  = \frac{1}{q} [ \Ustar (\sum_k \beta_{1,k}^- \tb_k \tb_k{}') \Ustar{}' + (\sum_k \beta_{2,k}^- \|\tb_k\|^2) \I ]
\]
Thus, the span of its top $r$ eigenvectors (same as singular vectors) equals $\Span(\Ustar)$. Hence, we can use the $\sin \Theta$ theorem \cite{davis_kahan} stated below in a fashion similar to \cite{lrpr_tsp} (Sec 6).
\begin{lemma}[Davis-Kahan $\sin \Theta$ theorem] \label{sinthetath}
	Given two symmetric matrices $\bm{D}$ and $\hat{\bm{D}}$. Let $\Ustar$ ($\U$) be the matrix of top eigenvectors of $\bm{D}$ ($\hat{\bm{D}}$).
	If $\lambda_r(\bm{D}) - \lambda_{r+1}(\bm{D}) - \|\bm{D} - \hat{\bm{D}}\| > 0$, then
	\[
	\SE(\U,\Ustar) \le \frac{\|\bm{D} - \hat{\bm{D}}\| }{\lambda_r(\bm{D}) - \lambda_{r+1}(\bm{D}) - \|\bm{D} - \hat{\bm{D}}\|}.
	\]
\end{lemma}

Using Lemma \ref{sinthetath} with $\hat{\bm{D}}= \Y_U$ and $\bm{D} = \E\left[\Y_-\right]$,
\begin{align*}
	&\SE(\U_\init,\U^*) \\
	&\leq \frac{\|\Y_U - \E\left[\Y_-(\epsilon_1)\right]  \|}{\lambda_r(\E\left[\Y_-(\epsilon_1)\right] ) - \lambda_{r+1}(\E\left[\Y_-(\epsilon_1)\right] ) - \|\Y_U - \E\left[\Y_-(\epsilon_1)\right]  \|} \\
\end{align*}
Moreover,
\begin{align*}
	\lambda_r(\E\left[\Y_-\right] ) - \lambda_{r+1}(\E\left[\Y_-\right] ) & = \frac{1}{q} \lambda_{\min}\left(\sum_k \beta_{1,k}^- \tb_k \tb_k{}' \right)\\
	& \geq  ( \min_k \beta_{1,k}^-  ) \frac{\sigmin^2}{q} 
\end{align*}
Now we just need to upper bound $\|\Y_U - \E\left[\Y_-(\epsilon_1)\right]  \|$ and lower bound $\min_k \beta_{1,k}^-$. Both these follow by combining the three lemmas given next and triangle inequality.
\begin{lemma}
	\label{lemm:bound_Yu}
	We have that, w.p. at least
	$1 - \exp(-\epsilon_1^2 \frac{m q}{\mu^2 \kappa^2}) $,
	\[ \Y_-(\epsilon_1) \preceq \Y_U \preceq \Y_+(\epsilon_1)\]
	and so $\|\Y_U - \Y_-(\epsilon_1)\| \le \|\Y_+(\epsilon_1) - \Y_-(\epsilon_1)\|$.
\end{lemma}

\begin{lemma}
	\label{lemm:bound_Sigma+-}
	Let $\Y_+ = \Y_+(\epsilon_1)$ and $\Y_- = \Y_-(\epsilon_1)$. We have
	\[
	\|\E[\Y_+] - \E[\Y_-] \| \leq \frac{ 9 \epsilon_1 \mu^2 \kappa^2 \|\Xstar\|_F^2 }{q} \le \frac{ 9 \epsilon_1 \mu^2 \kappa^2  r{\sigma_{\max}^*}^2}{q}
	\]
	and, assuming $\epsilon_1 < 0.01$,
	\[
	\min_k \beta_{1,k}^-(\epsilon_1) \geq 1.5.
	\]
\end{lemma}

\begin{lemma}
	\label{lemm:bound_Y-sigma-}
	We have that, \\ w.p. at least $1 - 2\exp\left(n\log 9  -c\epsilon_2^2 mq \right)$,
	\[
	\|\Y_-(\epsilon_1) - \E\left[\Y_-(\epsilon_1)\right]  \| \leq \frac{1.5 \epsilon_2 \mu^2 \kappa^2\ r \sigmax^2}{q}
	\]
	We get the exact same claim also for $ \|\Y_+ - \E\left[\Y_+\right]  \|$.
\end{lemma}
We prove the above lemmas in Sec. \ref{sec:init_lem_proofs}.

Using triangle inequality, $\|\Y_U - \E[\Y_-]  \| \le \|\Y_U - \Y_-\| + \|\Y_- - \E[\Y_-] \|$. Moreover, using Lemma \ref{lemm:bound_Yu},  $\|\Y_U - \Y_-\| \le \|\Y_+ - \Y_-\|$. Using these and again using triangle inequality, $\|\Y_U - \E[\Y_-]  \| \le 2 \|\Y_- - \E[\Y_-] \| + \|\Y_+ - \E[\Y_+] \|  + \|\E[\Y_+] - \E[\Y_-] \|$. Thus, combining bounds from the above lemmas and setting $\epsilon_1=\epsilon_2 =  \frac{\deltinit}{C (\kappa^2 \mu^2) \kappa^2 r}$ for a $\deltinit <1$, we conclude that
\\
w.p. $1 - 2\exp\left(n  - \frac{c \deltinit^2 mq}{ \kappa^8 \mu^4 r^2}  \right) -2\exp\left(-\frac{c \deltinit^2  m q}{  \kappa^8 \mu^4 r^2 } \right)$,
\bea
\|\Y_U - \E\left[\Y_-\right] \| \le \frac{ 0.25 \deltinit  {\sigmin}^2 }{ q }.
\label{yu_bnd}
\eea
Using $\min_k \beta_{1,k}^- \ge 1.5$, from Lemma \ref{lemm:bound_Sigma+-}, and \eqref{yu_bnd}, since $\deltinit<1$,
\begin{align*}
	& \SE( \U_\init,\U^*) \\
	& \leq \frac{\|\Y_U - \E\left[\Y_-\right]  \|}{\lambda_r(\E\left[\Y_-\right] ) - \lambda_{r+1}(\E\left[\Y_-\right] ) - \|\Y_U - \E\left[\Y_-\right]  \|} \\
	& \le \frac{ 0.25 \deltinit  {\sigmin}^2 / q  }{ 1.5  \sigmin^2/q -  0.25 \deltinit  \sigmin^2 / q  } < \deltinit.
\end{align*}

\subsubsection{Proof that rank is correctly estimated}
Consider the rank estimation step.
This requires lower bounding $\lambda_r(\Y_U) - \lambda_n(\Y_U)$ and upper bounding $\lambda_{r+1}(\Y_U) - \lambda_n(\Y_U)$. Both bounds follow using (i) \eqref{yu_bnd}, along with Weyl's inequality, and (ii) the lower bound on $\beta^-:= \min_k \beta_{1,k}^-$ from Lemma \ref{lemm:bound_Sigma+-}: $\beta^- > 1.5$.

We have $\lambda_r(\Y_U) - \lambda_n(\Y_U) \ge \lambda_r(\E\left[\Y_-\right]) - \lambda_n(\E\left[\Y_-\right])  - 2\|\Y_- - \E\left[\Y_-\right]  \| \ge \beta^- \lambda_r(\U^* {\Sigma^*}^2 {\U^*}' ) - 2\|\Y_U - \E\left[\Y_-\right] \| \ge 1.5(\sigmin^2/q) - 0.25 \deltinit  \sigmin^2 / q = (1.5-0.25 \deltinit)  \sigmin^2 / q > 1.4  (\sigmin^2 / q)$ as long as $\deltinit < 0.1$.

Also, $\lambda_{r+1}(\Y_U) - \lambda_n(\Y_U) \le \lambda_{r+1}(\E\left[\Y_-\right]) - \lambda_n(\E\left[\Y_-\right]) + 2\|\Y_U - \E\left[\Y_-\right] \| = 2\|\Y_- - \E\left[\Y_-\right] \| \le 0.25 \deltinit  ( \sigmin^2 / q ) < 0.025  ( \sigmin^2 / q )$ as long as $\deltinit < 0.1$. 

In summary, as long as $\deltinit < 0.1$, $\lambda_r(\Y_U) - \lambda_n(\Y_U) \ge 1.4  \sigmin^2 / q $ and $\lambda_{r+1}(\Y_U) - \lambda_n(\Y_U) \le 0.025  \sigmin^2 / q$.
Thus by setting the threshold $\omega = C \sigmin^2/q$ with $C$ being any constant between 0.025 and 1.4, we can ensure that the rank is correctly estimated whp.
%

\subsection{Proof of Claim \ref{lem:descent}} \label{proof_claim}
In this section, we remove the superscript $^t$ except where essential. Also, we let $\a_{ik}: = \a_{ik}^{(T+t)}$ and $\y_{ik}:=\y_\ik^{(T+t)}$.

We first use the overall approach of \cite{lowrank_altmin} to get the following deterministic bound on the subspace error of the $(t+1)$-th estimate of $\Ustar$, $\U^{t+1}$. The proof requires some messy algebra and hence we give it in Appendix \ref{proof_key_lem}.
\begin{lemma}\label{lem:key_lem}
	We have
	\begin{align}
		\SE(\U^{t+1}, \Ustar) 
		\leq \frac{ \mathrm{MainTerm} }{\sigma_{\min}(\U^* \bm\Sigma^* \B^* \B') - \mathrm{MainTerm}}
		\label{SEU_bnd_0}
	\end{align}
	where $\mathrm{MainTerm}:=$ 
	\[
	\frac{ \max_{\W \in \mathcal{S}_W}|\mathrm{Term1}(\W)| + \max_{\W \in \mathcal{S}_W}|\mathrm{Term2}(\W)|}{ \min_{\W \in \mathcal{S}_W} \mathrm{Term3}(\W)},
	\]
	\begin{align*}
		\mathrm{Term1}(\W) & :=  \sum_{ik} \b_k{}' \W' \a_\ik \a_\ik{}' \Ustar (\tB \B' \b_k - \tb_k), \\
		\mathrm{Term2}(\W) & :=  \sum_{ik} (\cb_\ik \hat\cb_\ik - 1) (\a_\ik{}' \W \b_k) (\a_\ik{}' \xstar_k), \\
		\mathrm{Term3}(\W) & :=  \sum_{ik}  (\a_\ik{}' \W \b_k)^2,
	\end{align*}
	\[
	\mathcal{S}_W := \{\W \in \mathbb{R}^{n\times r}:\ \|\W\|_F = 1 \}
	\]
	is the space of all $n \times r$ matrices with unit Frobenius norm,
	and $\cb_\ik, \hat\cb_\ik$ are the phases (signs) of $\a_\ik{}' \xstar_k$  and $\a_\ik{}' \xhat_k$.
\end{lemma}	
%
We obtain high probability bounds on the three terms above in the three lemmas that follow, Lemmas \ref{sigmaminG}, \ref{product}, \ref{Show}. All three lemmas first bound the terms for a fixed $\W$, followed by using a carefully developed epsilon-net argument to extend the bounds for all unit Frobenius norm $\W$'s. This is inspired by similar arguments in \cite{candes2009tight}.

Consider a fixed $\W$. To bound $\mathrm{Term1}$, we first show $\E[\mathrm{Term1}] = 0$. Next, we use Lemmas \ref{lem:bounding_distb} and \ref{incoherencebhat} to show that $\| \Xstar \B' \b_k - \xstar_k\| \le  C\delta_t \|\Xstar\|_F \max(\|\bstar_k\|, \|\b_k\|) \le  C\delta_t \|\Xstar\|_F \hat\mu \sqrt{r/q}$ and  $\|\Xstar (\B'\B - \I)\| \le C \deltapt \|\Xstar\|_F$. 
Finally, we use these two facts and the concentration bound of Lemma \ref{ProductsubG} (bounds sums of products of sub-Gaussian random variables) to show that, if $mq$ is large enough, whp,
$
|\mathrm{Term1}| \le C m \epsilon_1  \deltapt \|\Xstar\|_F
$ for any $\deltapt < 0.1$. This is followed by a careful epsilon-net argument to extend the bound for all unit Frobenius norm $\W$'s.

To bound $\mathrm{Term2}$ for a fixed $\W$, we  first use  Cauchy-Schwarz. This implies that
\begin{align*}
	|\mathrm{Term2}(\W)| &  \le \sqrt{\mathrm{Term3}(\W)} \sqrt{\mathrm{Term22}}, \text{ where} \\
	\mathrm{Term22} &  := \sum_{ik} (\cb_\ik \hat\cb_\ik - 1)^2 (\a_\ik{}' \xstar_k)^2
\end{align*}
We explain how to upper bound $\mathrm{Term3}(\W)$ in the next paragraph.
Consider $\mathrm{Term22}$. Notice that $(\cb_\ik \hat\cb_\ik - 1)^2$ takes only two values - zero or four. It is zero when the signs are equal, else it is four. To start bounding $\E[\mathrm{Term22}]$, we can use Lemma 1 of \cite{rwf}. This shows that the probability that the signs are unequal is upper bounded by a term that is directly proportional to the ratio $\dist^2(\xstar_k,\xhat_k)/(\a_\ik{}' \xstar_k)^2$. The probability bound is thus large when this ratio is large and small otherwise.
Moreover, it is easy to see that $\E[(\a_\ik{}' \xstar_k)^2] = \|\xstar_k\|^2$ and by Lemma \ref{lem:bounding_distb}, whp, $\dist^2(\xstar_k,\xhat_k) \le \deltapt^2 \|\xstar_k\|^2$. Relying on these ideas, we can argue that, on average, $\mathrm{Term22}$ is very small: $\E[\mathrm{Term22}]  \le C m \deltapt^3 \|\Xstar\|_F^2$. Use of concentration bound of Lemma \ref{ProductsubG} then implies that, if $mq$ is large enough, $\mathrm{Term22}$ is bounded by $C m (\epsilon_2 + \deltapt)\deltapt^2 \|\Xstar\|_F^2$ whp.

To upper and lower bound $\mathrm{Term3}$, notice first that $\E[\mathrm{Term3}] = m \|\W \B\|_F^2 = m$. Also, each summand in this term is sub-exponential with sub-exponential norm bounded by $\|\W \b_k\|^2$; and $\|\W \b_k\|^2 \le \|\b_k\|^2 \le \hat\mu^2 r/ q $ (by Lemma \ref{incoherencebhat}), and $\sum_{ik} \|\W \b_k\|^2 = m $. Using these facts and Lemma \ref{ProductsubG}, we can show that $\mathrm{Term3}$ concentrates around $m$ whp.

\begin{lemma}
	\label{sigmaminG}
	Pick a $\delta_b <1/10$ and assume that $m \ge C \max(r,\log n, \log q)/\delta_b^2$.
	Under the conditions of Theorem \ref{thm:main_res}, for a $\deltapt < 1/10$,
	w.p. at least
	$1 - 2\exp\left(nr  (\log 17)  - c \frac{\epsilon_3^2 m q}{\hat\mu^2 r}\right) - n^{-10}$,
	\begin{align*}
		& \min_{\W \in \mathcal{S}_{\W}}  \mathrm{Term3}(\W)  \ge 0.5 (1-\epsilon_3 )m
	\end{align*}
	and
	\begin{align*}
		&  \max_{\W \in \mathcal{S}_{\W}} \mathrm{Term3}(\W) \leq 1.5(1+\epsilon_3) m.
	\end{align*}
\end{lemma}

\begin{lemma}
	\label{product}
	Pick a $\delta_b <1/10$ and assume that $m \ge C \max(r,\log n, \log q)/\delta_b^2$.
	Under the conditions of Theorem \ref{thm:main_res} and assuming that $\SE(\Ustar,\U) \le \deltapt$, with $\deltapt < 1/10$,
	w.p. at least
	$1 - 2 \exp\left(nr (\log 17) - c \frac{ \epsilon_1^2 mq}{\kappa^3 \mu^2 r} \right) - n^{-10}$.
	\begin{align*}
		\max_{\W \in \S_{\W}} \mathrm{Term1}(\W) \leq m \epsilon_1 \deltapt  \|\Xstar\|_F.
	\end{align*}
	In proving the above, we also show that
	\[
	\| \tB \left( \I - \B' \B  \right) \|_F \le C\deltapt  \|\Xstar\|_F
	\]
\end{lemma}

\begin{lemma}
	\label{Show}
	Pick a $\delta_b <1/10$ and assume that $m \ge C \max(r,\log n, \log q)/\delta_b^2$.
	Under the conditions of Theorem \ref{thm:main_res} and assuming $\SE(\Ustar,\U) \le \deltapt$ with $\deltapt < 1/10$, w.p. at least
	$1 - 2\exp\left(nr (\log 17)  - c \frac{\epsilon_3^2 m q}{\hat\mu^2 r}\right) - 2\exp\left(- c\epsilon_2^2 m q\right) - n^{-10}$,
	\begin{align*}
		\max_{\W\in \mathcal{S}_W} \mathrm{Term2}(\W) \le m \ \sqrt{1+\deltapt} \sqrt{\deltapt + \epsilon_2} \deltapt \|\Xstar\|_F.
	\end{align*}
\end{lemma}

Finally, we lower bound the first denominator term of \eqref{SEU_bnd_0}. This can be done by using the bound on
$\| \tB \left( \I - \B {}' \B  \right)\|_F$ from Lemma \ref{product}.
This, in turn, implies a lower bound on the minimum singular value of $\B^* \B'$, and hence the following.
\begin{lemma}\label{lem:bound_RU}
	Pick a $\delta_b <1/10$ and assume that $m \ge C \max(r,\log n, \log q)/\delta_b^2$.
	Under the conditions of Theorem \ref{thm:main_res}, if $\SE(\Ustar,\U) \le \deltapt$ with $\deltapt \leq \frac{1}{4C \sqrt{r} \kappa}$, then, w.p. at least $1 -n^{-10}$,
	$\sigma_{\min}(\U^* \bm\Sigma^* \B^* \B')  \ge 0.9\sigmin.
	$
\end{lemma}

We prove the above lemmas in Appendix \ref{sec:key_lem_proofs}. 

\begin{proof}[Proof of Claim \ref{lem:descent}]
	Combining  Lemmas \ref{lem:key_lem} and \ref{lem:bound_RU}, if $\deltapt < c / \sqrt{r} \kappa$,
	\begin{align}
		\SE(\U^{t+1}, \Ustar)  \le \frac{\mathrm{MainTerm}}{0.9\sigmin - \mathrm{MainTerm}}.
		\label{SEU_bnd}
	\end{align}
	Set $\delta_b=1/11$.
	Combining Lemmas  \ref{product}, \ref{Show} and \ref{sigmaminG}, and using  $\|\Xstar\|_F \le \sqrt{r}\sigmax$, we conclude that, w.p. at least
\\
$1 - 2 \exp\left(nr (\log 17) - c \frac{ \epsilon_1^2 mq}{\kappa^3 \mu^2 r} \right) - 2\exp\left(nr (\log 17)  - c \frac{\epsilon_3^2 m q}{\hat\mu^2 r}\right) - 2\exp\left(- c\epsilon_2^2 m q\right) - n^{-10}$
	\begin{align*}
		\mathrm{MainTerm} \le C \frac{1}{0.5(1-\epsilon_3)} (\epsilon_1 + \sqrt{\deltapt + \epsilon_2})  \deltapt \sqrt{r} \sigmax.
	\end{align*}
In order to ensure that $\mathrm{MainTerm} \le 0.7 \deltapt  \sigmin$, we need to set $\epsilon_3 = 0.1$, $\epsilon_1 = c / \sqrt{r} \kappa$,   $\sqrt{\epsilon_2}= c / \sqrt{r} \kappa$, and  $\sqrt{\deltapt} \le c / \sqrt{r} \kappa$. Since we prove $\deltapt \le \deltinit$, the last bound is ensured if $\sqrt{\deltinit} \le c / \sqrt{r} \kappa$.
With these settings, if $mq \ge C nr^2 /\epsilon_1^2  = C nr^3 $, $mq \ge C nr^2 /\epsilon_3^2 = C nr^2$, $mq \ge C nr /\epsilon_2^2 = C nr^3$, and  $\deltinit = c/r \kappa^2$, then
\\ w.p. $1 - n^{-10} - 4 \exp(-c nr)$,
	\begin{align*}
		\SE(\U^{t+1}, \Ustar) 
		\le  0.7  \deltapt
	\end{align*}
\end{proof}



\begin{figure*}[t!]
	\centering
	\begin{tikzpicture}
	\begin{groupplot}[
	group style={
		group size=3 by 1,
		horizontal sep=2cm,
	},
	my stylecompare,
	enlargelimits=false,
	width = .3\linewidth,
	height=5.5cm,
	enlargelimits=false,
	]
	\nextgroupplot[
	my legend style compare,
	legend style={
		at={(0.2,1.4)},
		anchor=north west,
	},
	legend columns=7,
	legend style={font=\small},
	ymode=log,
	xlabel={\small{running time (in sec.)}},
	ylabel={\small{$\matdist(\Xstar, \hat\X^t$)}},
	title={\small{(a) $m=80$, $n=200$, $q=400$ }},
	]
	\addplot table[x index = {0}, y index = {1}]{\rankestsmall};
	\addplot table[x index = {2}, y index = {3}]{\rankestsmall};
	\addplot table[x index = {4}, y index = {5}]{\rankestsmall};
	\addplot table[x index = {6}, y index = {7}]{\rankestsmall};
	\addplot table[x index = {10}, y index = {11}]{\rankestsmall};
	
	\nextgroupplot[
	ymode=log,
	xlabel={\small{running time (in sec.)}},
	title={\small{(b) $m=150$, $n=600$, $q=1000$}},
	]
	\addplot table[x index = {0}, y index = {1}]{\rankestlarge};
	\addplot table[x index = {2}, y index = {3}]{\rankestlarge};
	\addplot table[x index = {4}, y index = {5}]{\rankestlarge};
	\addplot table[x index = {6}, y index = {7}]{\rankestlarge};
	\addplot table[x index = {10}, y index = {11}]{\rankestlarge};
	
	\nextgroupplot[
	ymode=log,
	xmode=log,
	xlabel={\small{running time (in sec.)}},
	title={\small{(c) $m=1000$, $n=200$, $q=200$}},
	]
	\addplot table[x index = {0}, y index = {1}]{\allwork};
	\addplot table[row sep=crcr]{%
		nan nan\\
	};
	\addplot [blue, line width=1.6pt, mark=o,mark size=4pt, select coords between index={0}{5}]table[x index = {2}, y index = {3}]{\allwork};
	\addplot [olive, line width=1.6pt, mark=diamond,mark size=5pt,  select coords between index={0}{5}]table[x index = {4}, y index = {5}]{\allwork};
	
	\end{groupplot}
	\end{tikzpicture}
	\vspace{-.2cm}
	\caption{\small{Error versus time plot with time in seconds. 
			We compare with LRPR2 which is the only other existing Low-Rank Phase Retrieval algorithm \cite{lrpr_tsp}, RWF \cite{rwf} and projected RWF. 
			The first step (initialization) of LRPR2 is slower than our proposed method. This is likely due to the fact that the estimates of the rank are different for the two algorithms and thus the errors are also slightly different in the two cases. For the purpose of better illustration, we only plot the error and time at the end of every $10$ iterations for RWF and proj-RWF.
	}}
	\label{fig:rank_est}
\end{figure*}
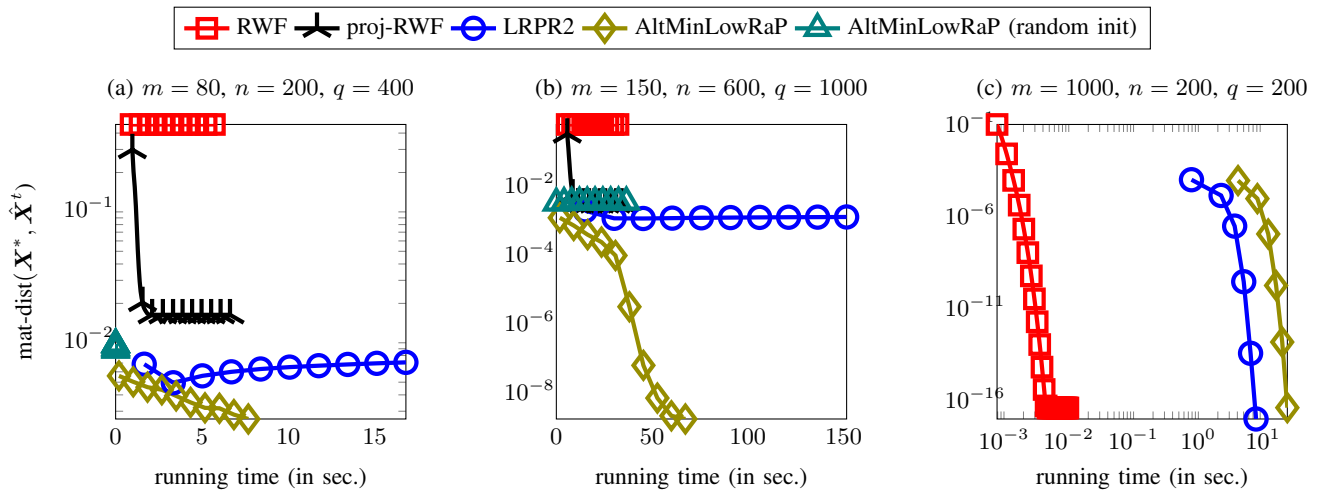

\section{Numerical Evaluation}\label{sec:expts}
In this section we provide detailed description of the numerical evaluation of our algorithms on synthetic and real data. All time comparisons are performed on a single Desktop Computer with Intel$^{\textsuperscript{\textregistered}}$ Xeon E$3$-$1240$ $8$-core CPU @ $3.50$GHz and $32$GB RAM.
{We must mention that for all the algorithms, (a) we compare on the same system, and we do not run other memory and compute intensive programs while performing the time comparison, (b) we use the most efficient sub-routines that are provided by the respective authors and these implementations rely on the ``anonymous functions'' feature of MATLAB. This provides a significant speed up since the major chunk of computation time involves matrix-vector/matrix-matrix products. This ensures uniformity to a large extent for all algorithms.  
}

%

\subsection{Synthetic Data Experiments} \label{simdata}
We demonstrate the effectiveness of AltMinLowRaP over existing work using two synthetic experiments. For both, we generate an error versus time-taken plot as follows: for each $t=0,1,\dots,T$, we plot the matrix recovery error (at the end of that iteration) and time-taken (until the end of that iteration) on the y- and x-axes respectively.
All synthetic data experiments are performed for $100$ independent trials, and for each algorithm, we plot the average error over the best $90$ trials (drop the 10 trials with the largest error). This is done because all algorithms are guaranteed to work well only with high probability.

In the {\em first} experiment, we compare the performance of several algorithms for different values of $m$. We generated data as follows: $\Xstar = \Ustar \tB$ where $\Ustar \in \mathbb{R}^{n \times r}$ is generated by orthonormalizing a iid standard Gaussian matrix. The entries of $\tB \in \mathbb{R}^{r \times q}$ are chosen from another iid standard Gaussian distribution.  Thus, in this setting, $\sigmin^2/q \approx \sigmax^2/q \approx 1$. Measurements were generated using \eqref{eq:problem} with $\a_\ik$'s being iid Gaussian. We compare AltMinLowRaP with LRPR2 (best heuristic from \cite{lrpr_tsp}), RWF \cite{rwf}, and with what we call projected-RWF or proj-RWF (for all $t$, after the $t$-th RWF iteration, we project the matrix $\hat{X}^t_{RWF}$ onto the space of rank-$r$ matrices, with $r$ known).
We provide results for three cases: (i) $n = 200$, $q = 400$, $r=4$, and $m = 0.4n$; (ii)  $n=600$, $q = 1000$,  $r=4$, and $m=0.25n$; and (iii) $n=200$, $m=5n = 1000$, $q=200$ and $r=4$. The results are summarized in Fig. \ref{fig:rank_est}. In the first two settings (Figs. \ref{fig:rank_est}(a),(b)), we show that by exploiting the low-rank structure, AltMinLowRaP is able to outperform the unstructured PR methods. Notice that, proj-RWF performs better than RWF (which does not work at all since $m \ll n$). Secondly, AltMinLowRaP is significantly better than proj-RWF, and for reasons explained earlier, it is also better than LRPR2. AltMinLowRaP and LRPR2 estimate the rank, whereas proj-RWF is provided the true rank. We observed that, in practice, the estimated rank is higher than the true rank in all the cases.  In both these settings, $ m \ll n$, and thus, unstructured algorithms fail. The third setting is a setting with a very large value of $m$: we are using $m=5n$ measurements. This is a case where all of AltMinLowRaP, TWF, and RWF work, but AltMinLowRaP is much slower as shown in Fig. \ref{fig:rank_est}(c).
{ Additionally, for the first two cases, based on a reviewer's comment, we also empirically evaluate a random initialization scheme for the AltMinLowRaP algorithm. We observed that although the initial estimates are approximately in the order of $10^{-2}$, the algorithm itself fails to improve this with subsequent iterations. A possible workaround to this could be through using a gradient descent based algorithm (inspired by \cite{pr_mc_reuse_meas}) but this requires a detailed analysis, and we will study this as part of future work.}.

Our {\em second} experiment illustrates the  time complexity discussion given in Sec. \ref{sec:main_res1}. For a given $m$, AltMinLowRaP is about $r$ times slower than the best provably correct regular (unstructured) PR methods - TWF and RWF. But, if for each algorithm, we use the minimum $m$ needed for the algorithm to achieve $\epsilon$ accuracy, then, theoretically, AltMinLowRaP should be faster if the rank $r$ is small enough. We tested this empirically as follows. We generated data as before with $n=600$, $q=1000$ and $r=4$. We implemented TWF and RWF using two values of $m$, $m=3n$ and $m=4n$. We evaluated AltMinLowRaP with using $m=n/4$. The error-at-iteration-$t$ versus time-taken-until-iteration-$t$ plot is shown in Fig. \ref{fig:time_comp} for all these cases. As can be seen, using $3n$ measurements, neither of TWF or RWF works. Using $m=4n$, both work. But if we compare the time taken (x-axis value) for any value of error level $\epsilon$, both are at least 5 times slower than AltMinLowRaP ($m=n/4$). For all algorithms, we repeat the expriments for $100$ independent trials, and plot the mean taken over the best $90$ trials to illustrate the high probability results.

\subsubsection{Algorithm parameters}
For both experiments, AltMinLowRaP was implemented as Algorithm \ref{lrpr_th} but using the same set of measurements (does not require sample-splitting), and with the following parameters: $T_{RWF,t}$ scales linearly  from 5 to 30,  $\omega =  1.3\sigmin^2/q \approx 1.3$ and $C_Y=9$. All parameters are as suggested in the theorem.
For LRPR2 we used the default parameters mentioned in the documentation. We set the maximum number of outer-loop iterations, $T_{max}=10$ for both.
For RWF and TWF, we used the default parameters suggested by the authors with the exception that we let the maximum number of iterations $T_{max}=300$ (to try to see if its error reduces with more iterations). However, as we plot the time-taken at the end of each iteration, this is not an unfair implementation of RWF; it only means that we have 300 data points to plot on our graph. 
We should point out that the recovery error for AltMinLowRaP is sensitive to the choice of $\omega$. If $\omega$ is too small (for given values of $m$ and of $C_Y$ used in computing $\Y_U$), the algorithm will significantly over-estimate the rank. This is especially problematic when $m$ is small (or $C_Y$ is large for a given $m$). Thus, as a thumb rule, for a lower value of $m/n$, the threshold $\omega$ should be larger. Of course if it is too large, it will underestimate the rank\footnote{This is a bigger problem when $\kappa=1$ as in the simulated data above. It is a lesser problem for real approximately low-rank data (e.g., slow changing videos) with a larger $\kappa$, since in those cases, the missed directions will be the ones with smaller singular values.}

\subsection{Real Videos with Simulated CDP measurements: Small $r$ suffices} \label{small_r}
We demonstrate the effectiveness of AltMinLowRaP for recovering two real video sequences (these are only approximately low-rank) from simulated Coded Diffraction Pattern (CDP) measurements. These measurements can be represented as $\bm{Y} = |\mathcal{F}(\bm{D} \Xstar)|$ where $\mathcal{F}(\cdot)$ is the DFT operation and the matrix $\bm{D}$ represents a diagonal mask matrix whose diagonal entries are chosen uniformly at random from $\{\pm 1, \pm \sqrt{-1}\}$ and modulate the intensity of the input. We generate CDP measurements of each frame of the video (the $k$-th frame vectorized is $\xstar_k$). We compared our algorithm with LRPR2 and RWF. We present the quantitative results in Table \ref{tab:vid} and the visual comparisons in Fig. \ref{fig:vid_frames} (given in the beginning), and in Fig. \ref{fig:vid_frames_plane_2n}. Notice that, in this case, even with $m = 5n$ measurements, RWF is unable to accurately recover the video and AltMinLowRaP has a slightly better performance w.r.t. LRPR2. The algorithm parameters are set as in the synthetic data experiments, swith the exception that we now set $T_{max} = 30$ for ALtMinLowRaP and LRPR2. AltMinLowRaP implementation used all the speed-up ideas for Fourier measurements explained in \cite{lrpr_tsp} for LRPR2 and so did LRPR2.

{
	We tested AltMinLowRaP with three possible values of rank, $r=15, 20, 25$. As can be seen, even $r=15$ suffices to get a significantly better reconstruction error than RWF for $m=5n$ CDP measurements. For the plane video, $n = 6912$ and $q  = 105$ and for the mouse video, $n = 5182$ and  $q = 90$, and thus in both cases, $r \approx 0.003n$. If $q$ is larger, a natural idea would be to use similar parameter settings, but instead implement the tracking variant of AltMinLowRaP (Algorithm \ref{pst_th}).
}

{
	\subsection{Real Videos with Simulated CDP measurements: Low-Rank versus (Wavelet) Sparse Models} \label{lr_vs_s}
	To justify the low-rank assumption on videos, we compare with CoPRAM \cite{fastphase}, a state-of-the-art, provable algorithm for compressive phase retrieval. Since the videos are not sparse in the spatial domain, as suggested in \cite{fastphase}, we use the Haar wavelet as the sparsifying basis\footnote{We also experimented with the Daubhechies-3 wavelet as the sparsifying basis in our experiments. However, we noticed that, for the plane video, irrespective of the wavelet basis, the number of coefficients necessary to preserve $\approx 90\%$ of the energy of the video required $s \approx 0.1n$ and thus the choice of wavelet basis is not detrimental in this experiment.}. 
	
	
	As can be seen from Table \ref{tab:vid_sp} and Fig. \ref{fig:vid_frames_plane_sparse}, the low-rank prior gives a much better reconstruction error in all three cases in this table including the exact sparse case. Since the video is not exactly wavelet sparse, we also performed a comparison on the sparsified video, wherein, for each image frame, we truncate the wavelet coefficients such that approximately $90\%$ of the energy in each frame is preserved. We refer to this as the {\em sparsified video} in this experiment. Sparsifying the video significantly improves the performance of CoPRAM, but AltMinLowRaP ($r=15$) is still better.
	Even with $m=10n$ measurements, the sparse model is unable to capture the finer details in the video. We also observed that a standard complex mask does not work very well for CoPRAM and hence for this experiment, we report the results when the entries of the CDP mask, $\bm D$ are chosen uniformly at random from $\{\pm 1\}$. We reshaped each video frame into size $32 \times 32$ since the online implementation of their code only works for small sized data. We provide the quantitative results in Table \ref{tab:vid_sp} and qualitative results for the plane video in Fig. \ref{fig:vid_frames_plane_sparse}.
}

%
%

\begin{table}[t!]
	\caption{$\matdist(\hat{\X}, \Xstar)$ and time comparison for the {\em mouse} and {\em plane} videos. We generate the measurements using the CDP model and consider two different number of settings. Notice that AltMinLowRaP is slightly better than LRPR2 but is slower than in the simulated data experiments.}
	\centering
	\resizebox{1\linewidth}{!}{
		
		\begin{tabular}{c c c c c c c} \toprule
			Algorithm & \multicolumn{3}{c}{$\matdist(\hat{\X}, \Xstar)$ (Running Time  in sec)}  \\ \midrule
			& \multicolumn{1}{c}{m = $5n$ (mouse)} & \multicolumn{1}{c}{m = $5n$ (plane)} & \multicolumn{1}{c}{m = $2n$ (mouse)} \\ \midrule
			
			RWF \cite{rwf} & $0.65$ ($0.35s$) & $0.65$ ($0.35$s) & $1.36$ ($0.25$s)  \\
			LRPR2 \cite{lrpr_tsp} ($r=25$) & $0.48$ ($81.8$s) & $0.10$ ($122.3$s) & $0.61$ ($31.0$s)  \\
			AltMinLowRaP ($r=25$) & $0.39$ ($297.6$s) & $0.09$ ($467.6$s) & $0.52$ ($122.1$s) \\
			AltMinLowRaP ($r= 15$) & $0.57$ ($277.2$s) & $0.15$ ($418.9$s) & $0.60$ ($113.6$s) \\
			AltMinLowRaP ($r=10$) & $0.70$ ($262.7$s) & $0.23$ ($409.0$s) & $0.72$ ($97.4$s) \\
			\bottomrule
		\end{tabular}
	}
	\label{tab:vid}
\end{table}

\begin{table}[t!]
	\centering
	\caption{Low rank versus Sparse PR: We compare with a recent state of the art algorithm for provable sparse phase retrieval. $\matdist(\hat{\X}, \Xstar)$ and time comparison for the {\em mouse} and {\em plane} videos using {\em real valued} CDP measurements. We conclude that low-rank is a better model than wavelet-sparse for slowly changing videos.}
	\resizebox{1\linewidth}{!}{
		
		\begin{tabular}{cccc} \toprule
			Algorithm & \multicolumn{3}{c}{Video} \\ \midrule
			& $m=5n$ (plane) & $m=10n$ (plane) & $m=5n$ (sparsified-plane) \\ \midrule
			CoPRAM \cite{fastphase}  & $2.113$  & $1.019$ & $0.3104$ \\
			RWF & $0.6531$ & $0.4134$ & $0.6514$ \\
			AltMinLowRaP ($r=15$) & $0.111$  & $0.109$ & $0.1427$ \\ \bottomrule
		\end{tabular}
	}
	\label{tab:vid_sp}
\end{table}

\color{black}

\begin{figure*}[t!]
	\begin{center}
		\resizebox{.7\linewidth}{!}{
			\begin{tabular}{cccc}
				\\    \newline
				\includegraphics[scale=1.3, trim={.1cm, .1cm, .1cm, .1cm}, clip=true]{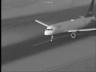}
				&
				\includegraphics[scale=1.3, trim={.1cm, .1cm, .1cm, .1cm}, clip=true]{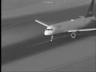}
				&
				\includegraphics[scale=1.3, trim={.1cm, .1cm, .1cm, .1cm}, clip=true]{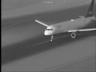}
				&
				\includegraphics[scale=1.3, trim={.1cm, .1cm, .1cm, .1cm}, clip=true]{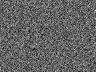}
				\\    \newline
				\includegraphics[scale=1.3, trim={.1cm, .1cm, .1cm, .1cm}, clip=true]{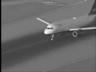}
				&
				\includegraphics[scale=1.3, trim={.1cm, .1cm, .1cm, .1cm}, clip=true]{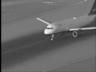}
				&
				\includegraphics[scale=1.3, trim={.1cm, .1cm, .1cm, .1cm}, clip=true]{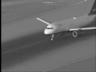}
				&
				\includegraphics[scale=1.3, trim={.1cm, .1cm, .1cm, .1cm}, clip=true]{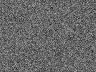}
				\\    \newline
				\subcaptionbox{Original\label{1}}{\includegraphics[scale=1.3, trim={.1cm, .1cm, .1cm, .1cm}, clip=true]{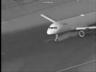}}
				&
				\subcaptionbox{AltMinLowRaP\label{1}}{\includegraphics[scale=1.3, trim={.1cm, .1cm, .1cm, .1cm}, clip=true]{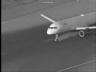}}
				&
				\subcaptionbox{LRPR2\label{1}}{\includegraphics[scale=1.3, trim={.1cm, .1cm, .1cm, .1cm}, clip=true]{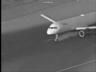}}
				&
				\subcaptionbox{RWF\label{1}}{\includegraphics[scale=1.3, trim={.1cm, .1cm, .1cm, .1cm}, clip=true]{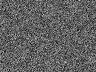}}
			\end{tabular}
		}
		
		\caption{Comparison of visual performance for the plane video with $m = 2n$. The images are shown at $k = 20, 60, 100$.}
		\label{fig:vid_frames_plane_2n}
	\end{center}
\end{figure*}

\begin{figure*}[t!]
	\begin{center}
		\resizebox{.7\linewidth}{!}{
			\begin{tabular}{c@{}c@{}c@{}c}
				\\    \newline
				\includegraphics[scale=4.1, trim={2.5cm, 1.7cm, 2.5cm, 1cm}, clip=true]{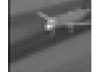}
				&
				\includegraphics[scale=4.28, trim={2.5cm, 1.7cm, 2.5cm, 1cm}, clip=true]{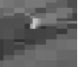}
				&
				\includegraphics[scale=4.23, trim={2.5cm, 1.7cm, 2.5cm, 1cm}, clip=true]{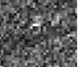}
				&
				\includegraphics[scale=4.1, trim={2.5cm, 1.7cm, 2.5cm, 1cm}, clip=true]{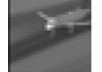}
				\\    \newline
				\subcaptionbox{Original\label{1}}{\includegraphics[scale=4.28, trim={2.5cm, 1.7cm, 2.5cm, 1cm}, clip=true]{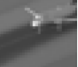}}
				&
				\subcaptionbox{CoPram\label{1}}{\includegraphics[scale=4.1, trim={2.5cm, 1.7cm, 2.5cm, 1cm}, clip=true]{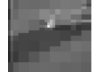}}
				&
				\subcaptionbox{RWF\label{1}}{\includegraphics[scale=4, trim={2.5cm, 1.7cm, 2.5cm, 1cm}, clip=true]{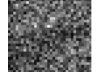}}
				&
				\subcaptionbox{AltMinLowRaP\label{1}}{\includegraphics[scale=4.4, trim={2.5cm, 1.7cm, 2.5cm, 1cm}, clip=true]{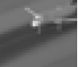}}
			\end{tabular}
		}
		
		\caption{ Comparison of visual performance for the plane video for $m=5n$ at $k=10,60$. We observe that RWF fails visually although numerical error is better, and CoPRAM is able to recover the global features but cannot recover the fine details due to the high level of sparsity required. AltMinLowRaP outperforms both methods, confirming the validity of the low-rank model.}
		\label{fig:vid_frames_plane_sparse}
	\end{center}
\end{figure*}

\newcommand{\J}{\mathcal{J}}

\section{Phaseless Subspace Tracking} \label{sec:pst}

When the matrix $\Xstar$ consists of a time sequence of signals $\xstar_k$, then the column-wise measurements appear one column at a time (sequentially).
Hence, there is benefit in trying to develop a  mini-batch algorithm that works with measurements of short batches of $\alpha$ consecutive columns. Moreover, for long data sequences, the subspace from which the data are generated could itself change with time. Detecting and being able to track such subspace changes is important for long sequences.
Interestingly the algorithm that works for this purpose is a simple modification of the static case idea along with a carefully designed subspace change detection step. 

\subsection{Problem setting}
The low-rank assumption is equivalent to assuming that  $\xstar_k = \Ustar \tb_k$ where $\Ustar$ specifies a fixed $r$-dimensional subspace. For long signal/image sequences, a better model (one that allows the required subspace dimension $r$ to be smaller) is to let the subspace change with time. As is common in time-series analysis, the simplest model for time-varying quantities is to assume that they are piecewise constant with time. We adopt this approach here.
Moreover, in order to easily borrow ideas from the static setting, we will assume that we now have a total of $\qfull$ signals (matrix columns) and we will denote the $n \times \qfull$ matrix formed by all these columns by $\Xstar_\full$. Our algorithm will operate on measurements of $\alpha$-consecutive-column sub-matrices of $\Xstar_\full$.

Let $k_0=1$, and let $k_j$ denote the $j$-th subspace change time, for $j=1,2,\dots, J$ and let $k_{J+1} = \qfull$. We have the following model
\begin{align}
	\xstar_k = \Ustar_{\sub,(j)} \td_k,  \ \text{for all} \ k_j \le k \le k_{j+1}
\end{align}
where $\Ustar_{\sub,(j)}$ is an $n \times r$ ``basis matrix'' for the $j$-th subspace and $\td_k$ is the coefficients' vector at time $k$.
%

The goal is to track the subspaces $\Span(\Ustar_{\sub,(j)})$ on-the-fly; of course, ``on-the-fly'' for subspace tracking means with a delay of at least $r$. Once this can be done accurately enough, it is easy to also recover the matrix columns $\xstar_k$ (by solving a simple $r$-dimensional PR problem to recover the $\td_k$'s).

The reason we use a different notation here (the subscript $\sub$ and use of $\td_k$ instead of $\tb_k$) is as follows. Consider an $\alpha$-column sub-matrix formed by $\alpha$ consecutive signals. Let us call it $\Xstar$ and let $\Xstar \svdeq \Ustar \bm\Sigma^* \Bstar$. If all the $\xstar_k$'s forming this matrix are generated from the same subspace, say $\Ustar_{\sub,(j)}$, then $\Span(\Ustar) = \Span(\Ustar_{\sub,(j)})$ and there is no need for a different notation. However, if a subspace change occurred inside this interval, then we cannot say anything simple like this.
All we can say is that $\Xstar = [\Ustar_{\sub,(j-1)} \bm{D}_{(j-1)} , \Ustar_{\sub,(j)} \bm{D}_{(j)} ]$ and so $\Span(\Ustar) \subseteq \Span(\Ustar_{\sub,(j-1)}) \cup \Span(\Ustar_{\sub,(j)})$.


%


\subsection{Basic PST algorithm and extensions} 
As noted earlier the PST algorithm is a simple modification of the static case algorithm (AltMinLowRaP) along with a carefully designed change detection strategy. In the static case,  in each iteration, we used a set of $mq$ measurements of a single $n \times q$ matrix $\Xstar$. For obtaining the guarantees, we assumed a new (independent) set of $2mq$ measurements of the {\em same} matrix $\Xstar$ were used in each iteration ($mq$ for updating the estimate of $\Bstar$ and another $mq$ for $\Ustar$).
For the tracking setting, using a mini-batch size of $\alpha$, we proceed as follows: each new update iteration uses $2m \alpha$  measurements of a {\em new} $\alpha$-consecutive-column sub-matrix of $\Xstar_\full$. The input to the update iteration is the subspace estimate from the previous iteration. Under the assumption that the subspace remains constant for at least $T \alpha$ time instants after a subspace change has been detected, this approach works: with $T  = C\log(1/\epsilon)$, we can show that, after $T \alpha$ time instants, we get an $\epsilon$-accurate estimate of the $j$-th subspace.

We summarize the algorithm in Algorithm \ref{pst_th}. This toggles between a ``detect'' and an ``update'' mode. It starts in the ``update'' mode (described above) and remains in it for the first $T \alpha$ time instants. At this time it enters the ``detect'' mode. We are able to guarantee that, when the algorithm enters this detect mode, the previous subspace has been estimated to $\epsilon$ error whp. In the detect mode, the algorithm {\em does not} perform any subspace updates. This is done to simplify our analysis; it ensures that, in the interval during which the subspace change occurs, the subspace is not updated. This is what allows us to use our previous two main claims (Claims \ref{lemm:bounding_U} and \ref{lem:descent}) without change to analyze the update mode.
Practically, this is of course wasteful. We develop an improvement below.

To understand the change detection strategy, let $\hat{k}_j$ denote the estimated change times. 
Consider an $\alpha$-length interval, $\J_\alpha$, contained in $[k_j, k_{j+1})$. Assume that an $\epsilon$-accurate estimate of the previous subspace $\Ustar_{\sub,(j-1)}$ has been obtained by $\hat{k}_{j-1} + T \alpha$ and that this time is before $k_j$. Let $\U_{\sub,(j-1)} $ denote this estimate.
Define the matrix
\begin{align*}
	& \Y_{U,det,big}:= \\ & (I - \U_{\sub,(j-1)} \U_{\sub,(j-1)}{}') \Y_U (I - \U_{\sub,(j-1)} \U_{\sub,(j-1)}{}')
\end{align*}
with $\Y_U = \Y_U(\J_\alpha)$. This means that $\Y_U$ is as defined earlier in \eqref{def_YU} with the $k$ summation being over all $k \in \J_\alpha$ (it is
using measurements for all the columns within this $\alpha$-length interval).
With a little bit of work (see Lemma \ref{lem:changedet} and its proof), one can show that, in this interval, the matrix $\Y_{U,det}:=\U_{\sub,(j-1),\perp}{}'\Y_{U,det,big}\U_{\sub,(j-1),\perp}{}$ is close to a matrix $\tSigma$ whose eigenvalues satisfy
\begin{align*}
	&\lambda_{\max}(\tSigma) - \lambda_{\min}(\tSigma) \\ &  \ge 1.5 (\SE(\U_{\sub,(j-1)},\Ustar_{\sub,(j)}) - 2 \epsilon)^2 \frac{ \sigmin^2 }{ \alpha}.
\end{align*}
On the other hand, in an $\alpha$-length interval contained in $[\hat{k}_{j-1} + T \alpha, k_j)$,
\begin{align*}
	&\lambda_{\max}(\tSigma) - \lambda_{\min}(\tSigma) \\ &  \le \SE^2(\U_{\sub,(j-1)},\Ustar_{\sub,(j-1)}) \sigmax^2/\alpha \le \epsilon^2 \frac{ \sigmax^2 }{\alpha}.
\end{align*}
Thus, this quantity is small when the $j$-th change has not occurred (before $k_j$), and is large when the subspace has changed (after $k_j$). By using a large enough lower bound on the product $m \alpha$, the same can be shown for the difference between the maximum and minimum eigenvalues of $\Y_{U,det}$ (these are equal to the maximum and $(n-r)$-th eigenvalues of $\Y_{U,det,big}$).

%
Once we have an $\epsilon$-accurate estimate of the current subspace, it is straightforward to also recover the corresponding signals $\xstar_k$. This can simply be done by solving a standard PR problem to recover the coefficients vector. See last line of Algorithm \ref{pst_th}. This borrows a similar idea from \cite{rrpcp_icml}. 

\subsubsection{Improved algorithm: PST-all}
Notice from Theorem \ref{thm:pst} that Algorithm \ref{pst_th} can only provably detect and track subspace changes that are larger than a small threshold. While this makes sense for detection, it should be possible to track all types of changes.
By including a simple modification in Algorithm \ref{pst_th} (include the ``update'' step during the detection mode as well), we can empirically demonstrate that this is indeed true. We demonstrate this in Fig \ref{fig:pst}(a). Moreover, PST-all also removes the other limitation of basic PST (not using the detect phase samples for improving the subspace estimate). Thus, even for large changes that basic PST can detect, PST-all has better tracking performance; see Fig \ref{fig:pst}(b).

\subsection{Guarantee for basic PST}
We can prove the following about Algorithm \ref{pst_th} (basic PST).

\begin{corollary}[PST algorithm]\label{thm:pst}
	Consider Algorithm \ref{pst_th}.
	Pick any value of $m \ge C \max(r, \log n, \log \qfull)$. For this $m$, set $\alpha = \frac{C \kappa^{12} \mu^4 \cdot nr^4 }{ m }$.
	Set $T := C\log(1/\epsilon)$, and the detection threshold $\omega_{det} = c/(\kappa^2 r)$.
	Assume that $k_{j+1} - k_j \ge  (T+3) \alpha$ and that $\SE(\Ustar_{\sub,(j-1)}, \Ustar_{\sub,(j)})^2  > \frac{2 c}{\kappa^2 r}$.
	%
	Then, w.p. at least $ 1- C n^{-10}$,
	\begin{enumerate}[noitemsep]
		\item we can detect the change with a delay of at most $2\alpha$, while ensuring no false detections, i.e., $k_j \le \hat{k}_j \le k_j + 2 \alpha$;
		\item for any $\epsilon > 0$, we can get an $\epsilon$-accurate estimate of the $j$-th subspace with a delay of at most $(T + 3) \alpha$ from $k_j$ (when the subspace changed);
		\item we have the following subspace error bounds:
		let $\U_{\sub,(j)}^{(-1)} = \U_{\sub,(j-1)}:= \U_{\sub,(j-1)}^{(T)}$, and let $\U_{\sub,(j)}^{(\ell)}$, $\ell=0,1,\dots, T$, be the $\ell$-th estimate;
		\begin{align*}
			&\SE(\U_{\sub,(j)}^{(\ell)}, \Ustar_{\sub,(j)}) \le \\ & \left\{
			\begin{array}{ll}
				\SE(\Ustar_{\sub,(j-1)}, \Ustar_{\sub,(j)}) + \epsilon  & \text{ if }  \ell=-1 \\
				(0.7)^{\ell-1} \frac{c}{\kappa^2 r}   & \text{ if }  \ell=0,1,2,\dots T,  \\
				\epsilon  & \text{ if } \ell = T  \\
			\end{array}
			\right.
		\end{align*}
		\\
		Offline PST returns $\hat\X$ that satisfies $\matdist(\hat\X,\Xstar) \le \epsilon$.
	\end{enumerate}
\end{corollary}

We provide a proof sketch in Appendix \ref{proof_pst}.

The above result shows that, if the subspace remains constant for  at least $\alpha \log(1/\epsilon)$ time instants, and if the amount of subspace change (largest principal angle of subspace change) is of order $1/\sqrt{r}$ or larger, then we can both detect the change and track the changed subspace to $\epsilon$ error within a delay of order $\alpha \log 1/\epsilon$. Moreover, for only at most $3\alpha$ time instants after a change, the subspace error does not reduce and is essentially bounded by the amount of change. After this, it decays exponentially every $\alpha$ time instants.

Notice from the expression for $\alpha$ that, if we pick the smallest allowed value of $m$, then the required $\alpha$ (and hence the required delays) will be large. However, we are allowed to tradeoff $m$ and $\alpha$. If we let $m$ grow linearly with $n$, then we will only need $\alpha \approx r^4$, which is, in fact, close to the minimum required delay of $r$. This also matches what is seen in existing works on provable subspace tracking (ST) in other settings (e.g., robust ST, ST with missing data, or streaming PCA with missing data) \cite{rrpcp_dynrpca,rrpcp_icml,streamingpca_miss}. These are able to allow close to optimal detection and tracking delays but all these assume that $m$ increases linearly with $n$.
We can also pick any value of $m$ in between the two extremes of $m = C r$ or $m=C n$. For example, if $m = C n/r$, then $\alpha = r^5$ and so on.




\begin{algorithm}[ht!]
	{
		\caption{PST: detect and track large subspace changes}
		\label{pst_th}
		\begin{algorithmic}[1]		
			\STATE Set $r$ equal to the largest index $j$ for which $\lambda_j(\Y_U) - \lambda_n(\Y_U) \ge \omega$. 	
			\STATE $\hat{k}_0 \gets 0, j \gets 0, \ell \gets 0$
			\STATE $\mbox{Mode} \gets \mbox{update}$
			\FOR{$k \geq 0$}
			\IF{$\mbox{Mode} = \mbox{update}$}
			\IF{$k = \hat{k}_j + (\ell+1) \alpha $}
			\IF{$\ell =0$}
			\STATE $\U_{\sub,(j)}^{\ell} \gets$ top $r$ singular vectors of $\Y_U$.
			\ENDIF
			\STATE  $\bhat_\tau \gets RWF((\y_{\tau}, \U_{\sub,(j)}^{\ell}{}' \A_{\tau}), T_{RWF,\ell})$, for $\tau \in [k-\alpha + 1, k]$
			\STATE QR decomposition $\hat{\B} \qreq  \R_B \B $
			\STATE $\Chat_\tau \gets \text{Phase}\left(\A_\tau'\U_{\sub,(j)}^{\ell}\bhat_\tau\right)$, for $\tau \in [k-\alpha + 1, k]$
			\STATE $\Uhat_{\sub,(j)}^{\ell+1} \gets \arg\min_{\U} \sum_{\tau \in [k-\alpha + 1, k]}\| \Chat_\tau \y_\tau - \A_\tau{}' \U \bhat_{\tau}\|^2$
			\STATE  QR decomposition $\Uhat_{\sub,(j)}^{\ell+1} \qreq  \U_{\sub,(j)}^{\ell+1} \R_U $
			\STATE $\ell \gets \ell+1$
			\ENDIF
			\IF{$\ell = T$}
			\STATE $\U_{\sub,(j)} \gets \U_{\sub,(j)}^{T}$, $\mbox{Mode} \gets  \mbox{detect}$
			\ENDIF
			\ENDIF
			\IF{$\mbox{Mode} =  \mbox{detect}$ }
			\IF{$ \lambda_{\max}(\Y_{U,det,big}) - \lambda_{n-r} ( \Y_{U,det,big} ) \geq \omega_{det} $}
			\STATE $j\gets j+1, \hat{k}_j \gets k, \ell \gets 0$, $\mbox{Mode} \gets  \mbox{update}$
			\ENDIF
			\ENDIF
			\STATE Output $\U_{\sub,(j)}^\ell$
			\ENDFOR
			
			Offline PST: For each $k \in [\hat{k}_j, \hat{k}_{j+1})$,  output $\xhat_k =  \U \hat{\tilde{\dd}}^*_{k}$ where $ \hat{\tilde{\dd}}^*_{k}$ is a (at most) $2r$-length vector obtained by RWF applied on $\{ \y_{ik}, (\U'\a_{ik}), i=1,2,\dots, m \}$ with $\U = basis([\U_{\sub,(j)},\U_{\sub,(j+1)}])$. Here $basis(\U_1, \U_2)$ means a matrix with orthonormal columns that span the subspace spanned by the columns of $\U_1$ and  $\U_2$. We need to use the union of both subspace estimates because the actual subspace change time, $k_{j+1}$, is not known. Corollary \ref{thm:pst} implies that, whp, it is contained in $[\hat{k}_j, \hat{k}_{j+1})$.
		\end{algorithmic}
	}
\end{algorithm}

\subsubsection{Related Work}
For  Phaseless Subspace Tracking (PST) the only works before this work was our first huristic versions \cite{lrpr_globalsip}, and \cite{lrpr_icassp19}.
Other subspace tracking (ST) problems that have been extensively studied include dynamic compressive sensing \cite{stab_jinchun_jp} (a special case of ST where the subspace is defined by the span of a subset of $r$ vectors from a known dictionary matrix), dynamic robust PCA (or robust ST), see \cite{rrpcp_dynrpca,rrpcp_icml} and references therein, streaming PCA with missing data \cite{streamingpca_miss,sslearn_jmlr}, and ST with missing data \cite{grouse,petrels,local_conv_grouse,chi_review,rrpcp_tsp19}. 
In terms of works with complete provable guarantees, there is the nearly optimal robust subspace tracking via recursive projected compressive sensing approach  \cite{rrpcp_dynrpca,rrpcp_icml,rrpcp_tsp19} and its precursors; recent papers on streaming PCA with missing data \cite{streamingpca_miss,sslearn_jmlr}, and older work on dynamic compressive sensing (CS) \cite{stab_jinchun_jp}. For robust ST, the problem setting itself implies $m=n/2$. In the streaming PCA case, the availability of  $m= \rho n$ measurements, with $\rho <1$, is assumed. This is why both achieve close to optimal tracking delays (at least when the added unstructured noise is nearly zero). As noted earlier, our method can also achieve a delay of order $r^4$ if we let $m$ grow linearly with $n$.

Dynamic CS (like basic CS) is able to detect support changes (with sufficiently nonzero magnitude) immediately even  with a small value of $m = C r \log n$ measurements; here $r$ is the sparsity level (support size). This is because it is a much simpler special case of ST: in this case, one just needs to be finding the correct subset of basis vectors from a large provided set (dictionary matrix).



\subsection{Numerical experiments}
This experiment evaluates the PST algorithm (Algorithm \ref{pst_th}) and PST-all algorithms from Sec. \ref{sec:pst}. We generate the true data for the first subspace $\Xstar_0 = \Ustar_{\sub, (0)} \D_0^*$ where $\Ustar_{\sub, (0)} \in \mathbb{R}^{n \times r}$ with $n = 300$, $r = 2$ is generated by orthonormalizing the columns of a $n \times r$ iid standard normal matrix. The entries of $\D_0^* \in \mathbb{R}^{r \times t_1}$ with $t_1 = 2992$ are also generated from an i.i.d. standard normal distribution. We generate the true data from the second subspace similarly and set $\Xstar_1 = \Ustar_{\sub,(1)} \D_1^*$ and we set $q = 6000$. Notice that $\kappa \approx 1$. The subspace $\Ustar_{\sub, (1)}$ is generated using the idea of \cite{rrpcp_icml} as $\Ustar_{\sub, (1)} = e^{-\gamma \bm{M}} \Ustar_{\sub, (0)}$ in order to control the subspace error. Here $\bm{M}$ is a skew-symmetric matrix and $\gamma$ controls the amount of subspace change. We study two cases in which we set $\gamma = 0.08, 0.001$ which roughly translates to $\SE(\Ustar_{\sub, (0)}, \Ustar_{\sub, (1)}) = 0.8, 0.01$. We generate the measurement matrices $\A_k (\in \mathbb{R}^{n \times m}) \overset{i.i.d.}{\sim} \mathcal{N}(0, \I)$ with $m=100$ for $ i = 1, \cdots q$. We then implemented PST (Algorithm \ref{pst_th}) and PST-all. PST requires large-enough change in order to ensure good results, and PST-all which works even with small changes.
We chose the algorithm parameters as follows. We set $\alpha = 250$ and $L = 8$. For the detection, and initialization steps of both algorithms we set $\initm = m$. We set the threshold for detection, $\omega = 0.6$ through cross-validation. The results for the two algorithms are shown in Fig. \ref{fig:pst}. Notice that for the small change case, since PST is always in the detect mode, it does not improve the estimation error whereas PST-all does. However, when the change is large enough, both algorithms converge to a small error. The results are averaged over $100$ independent trials.

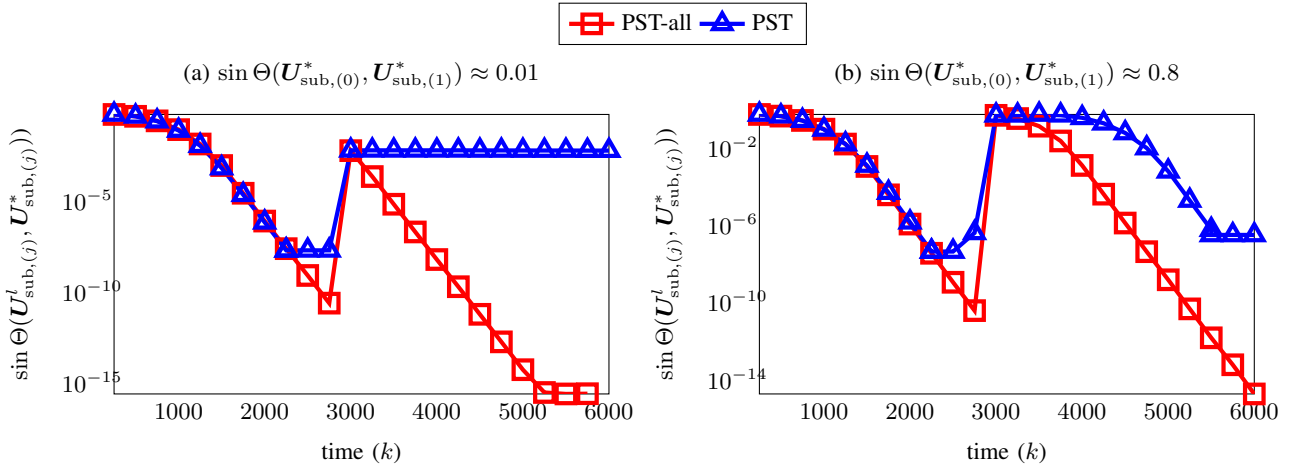
\begin{figure*}[t!]
	\centering
	\begin{tikzpicture}
	\begin{groupplot}[
	group style={
		group size=2 by 1,
		horizontal sep=2cm,
	},
	my stylecompare,
	enlargelimits=false,
	width = .45\linewidth,
	height=5.3cm,
	enlargelimits=false,
	]
	\nextgroupplot[
	legend entries={
		PST-all,
		PST
	},
	legend style={at={(1.4,1.4)}},
	legend columns = 2,
	legend style={font=\small},
	ymode=log,
	xlabel={\small{time ($k$)}},
	ylabel={\small{$\SE(\U_{\sub, (j)}^l, \Ustar_{\sub, (j)})$)}},
	title={\small{(a) $\SE(\Ustar_{\sub, (0)}, \Ustar_{\sub, (1)}) \approx 0.01$}},
	]
	\addplot [red, line width=1.6pt, mark=square,mark size=3.5pt, select coords between index={0}{22}] table[x index = {0}, y index = {1}]{\pstsmall};
	\addplot [blue, line width=1.6pt, mark=triangle,mark size=4pt] table[x index = {2}, y index = {3}]{\pstsmall};
	\nextgroupplot[
	ymode=log,
	xlabel={\small{time ($k$)}},
	ylabel={\small{$\SE(\U_{\sub, (j)}^l, \Ustar_{\sub, (j)})$)}},
	title={\small{(b) $\SE(\Ustar_{\sub, (0)}, \Ustar_{\sub, (1)}) \approx 0.8$}},
	]
	\addplot [red, line width=1.6pt, mark=square,mark size=3.5pt, select coords between index={0}{23}] table[x index = {0}, y index = {1}]{\pstlarge};
	\addplot [blue, line width=1.6pt, mark=triangle,mark size=4pt] table[x index = {2}, y index = {3}]{\pstlarge};
	
	\end{groupplot}
	\end{tikzpicture}
	\vspace{-.2cm}
	\caption{Plot of subspace error versus time at each $\alpha$ frames. Notice that for the cases where $\SE(\Ustar_{\sub, (0)}, \Ustar_{\sub, (1)}) = 0.8$ both algorithms are able to detect and track changes whereas when $\SE(\Ustar_{\sub, (0)}, \Ustar_{\sub, (1)}) = 0.01$ only the PST-all algorithm works. We perform the experiment for $100$ independent trials, and plot the average taken over the best $90$ trials.}
	\label{fig:pst}
\end{figure*}

\section{Conclusions and Future Work} \label{conclude}
This work introduced the first simple, fast, and provably correct, algorithm for Low Rank Phase Retrieval  -- low-rank matrix recovery from different (mutually independent) column-wise phaseless linear projections -- AltMinLowRaP. Moreover, since, even the linear version of our problem has not been studied, this work also provides the first fast and provable solution to the linear version, ``Compressive'' or ``Sketched PCA''.
AltMinLowRaP relies on a careful spectral initialization followed by alternating minimization. We showed that its required sample complexity is about $r^3$ times the order-optimal value of $nr$. We also developed its dynamic extension that is relevant for datasets where we would like to develop a mini-batch solution that recovers the current sub-matrix of $\Xstar$ without waiting for all the measurements of all the signals (columns) to arrive.


In ongoing work we are (i) exploring how to remove the sample-splitting requirement by studying an alternating gradient descent solution, and attempting to borrow the leave-one-out ideas from \cite{pr_mc_reuse_meas}; (ii) how to reduce the dependence of our sample complexity on $r$ and on $\kappa$. Some thoughts are provided in Sec. \ref{discuss}; and (iii) how to analyze an easy modification of AltMinLowRaP to get a better guarantee for the linear version of our problem. The algorithm modification is easy, it just involves replacing the standard PR step for recovering $\tb_k$'s by a simpler LS step, and of course remove the phase/sign estimation step before updating $\U$. In terms of analysis, (a) we can possibly simplify the analysis of the initialization step because in the linear case, $\E[\a_ik \y_ik] = \xstar_k$; and (b) in the iterations, there will be no phase error term, $\mathrm{Term2}$, and hence, no need for Cauchy-Schwarz. This latter change itself will reduce the sample complexity to $nr^3$ instead of $nr^4$.

Open questions for future work include (i) study if we can exploit the right incoherence assumption in the algorithm itself, for example, by using a projected GD approach inspired by \cite{rpca_gd}; (ii) can phaseless LRMS be solved (this would be the other possible LRPR problem alluded to in the introduction), and (iii) develop a fast algorithm and a guarantee for exploiting both low-rank (as we do) and column-wise sparsity. This type of modeling been used very successfully in the MRI literature to come up with practical algorithms to reduce the sample complexity empirically, see for example, \cite{dyn_mri1,dyn_mri2}. It has also been studied theoretically in the linear setting  \cite{lee2017near}.
(iv) Another open question is how to analyze the improved tracking algorithm PST-all that we currently only empirically evaluate. In experiments, it is clearly much better than the simpler version we analyze.

\appendices
\counterwithin{theorem}{section}

\section{Proofs of the Lemmas from Sec. \ref{proof_main}}\label{proof_lems}	

\subsection{Simple facts for various proofs}\label{facts}
Our proofs will use the following facts: for two arbitrary matrices $\A, \bm{H}$,

\begin{enumerate} 
	\item $\sigma_{\max}(\A + \bm{H}) \le \sigma_{\max}(\A)  + \|\bm{H} \|$
	\item $\sigma_{\min}(\A + \bm{H}) \ge \sigma_{\min}(\A) -  \| \bm{H} \|$.
	\item $\sigma_{\min}(\A  \bm{H}) \ge \sigma_{\min}(\A) \sigma_{\min}(\bm{H})$.
	\item For two basis matrices, $\U_1, \U_2$, $\sigma_{\min}^2(\U_1{}' \U_2) = 1 - \SE^2(\U_1, \U_2)$.
	
	\item For any matrix $\bP$, $\|\bP\M\|_F \leq \|\bP\| \|\M\|_F$ and  $\|\M\bP\|_F \le \|\bP\| \|\M\|_F$.
	
	\item For an invertible matrix $\bP$, \\
	$\|\M\|_F = \|\bP^{-1}\bP \M\|_F \leq \|\bP^{-1}\|\|\bP\M\|_F = \frac{1}{\sigma_{\min}\left( \bP\right)} \|\bP\M\|_F$.
	
\end{enumerate}

The following lemma is a simple but useful modification of [Lemma 5.16]\cite{vershynin}. The proof follows by combining Lemma 2.7.7. and Theorem 2.8.1. of \cite{versh_book}.
The sub-Gaussian norm for a random vector $X$, in this lemma, can be defined as follows:
\begin{align*}
	& \|X\|_{\psi_2} = \sup_{x \in \S^{n-1}} \|\langle X, x \rangle\|_{\psi_2},
\end{align*}
where $\|z\|_{\psi_2}$ is the sub-Gaussian norm of a scalar $z$ \cite[Definition 5.7]{vershynin}.
\begin{lemma}
	\label{ProductsubG}
	Let $X_{i}, Y_{i}$ be independent sub-Gaussian random variables with sub-Gaussian norm $K_{X_i}$ and $K_{Y_i}$ respectively and with $\E[X_iY_i] = 0$. Then
	\begin{align*}
		&	\Pr \left\lbrace   |\sum_{i}  X_iY_i| \geq t \right\rbrace \\
		& \leq 2\exp{\left( -c \min{\left(\frac{t^2}{\sum_i K_{X_i}^2 K_{Y_i}^2 }, \frac{t}{\max_i{|K_{X_i}K_{Y_i} |}} \right) } \right)}
	\end{align*}
	When $X_i = Y_i$, this simplifies to Lemma 5.16 of \cite{vershynin}.
\end{lemma}

\subsection{Proof of the lemmas for Claim \ref{lemm:bounding_U}} \label{sec:init_lem_proofs}
In this section, we let $\a_{ik}: = \a_{ik}^{(0)}$ and $\y_{ik}:=\y_\ik^{(0)}$.
\begin{proof}[Proof of Lemma \ref{lemm:bound_Yu}]
	Observe that we will be done if we can show that, whp, $\frac{1}{mq}\sum_{ik} \left( \a_{ik}{}' \x_k^*\right)^2$ lies in the interval $[(1-\epsilon_1) \|\X^*\|_F^2/q, (1+\epsilon_1) \|\X^*\|_F^2/q]$.
	Using Lemma \ref{ProductsubG}, with $ K_{X_{ik}} = K_{Y_{ik}} =  \|\x_k\|$, with probability more than $ 1 - 2\exp\left(\frac{ -C\epsilon_1^2 m q}{ \mu^2 \kappa^2 } \right)$, we have
	\begin{align*}
		&\lvert \sum_{ik} \left( \a_{ik}{}' \x_k^*\right)^2 - m\|\X^*\|_F^2\rvert\leq \epsilon_1 m \|\X^*\|_F^2 .
	\end{align*}
	Details for obtaining this bound: using  $\sum_{k}\|\xstar_k\|^4 \le \max_k \|\xstar_k\|^2 \sum_{k}\|\xstar_k\|^2$ and right incoherence,
	\begin{align*}
		& \frac{t^2}{\sum_{ik} K_{X_{ik}}^4}  \ge \frac{ \epsilon_1^2 m^2 \|\X^*\|_F^4}{m \max_{k}\|\x_k^*\|^2 \|\X^*\|_F^2}
		\geq \frac{ \epsilon_1^2 m q}{ \mu^2 \kappa^2 }, \\  
		&	 \frac{t}{\max K_{X_{ik}}^2}  = \frac{ \epsilon_1 m \|\X^*\|_F^2}{\|\x_k^*\|^2 }\geq
		\frac{ \epsilon_1 m q}{\mu^2 \kappa^2 } .
	\end{align*}
\end{proof}

\begin{proof}[Proof of Lemma \ref{lemm:bound_Sigma+-}]
	It is easy to see that
	\begin{align*}
		&\|\E\left[\Y_+\right] - \E\left[\Y_-\right] \| \leq \frac{1}{q} \sum_{k}\left(\beta_{1,k}^+ -  \beta_{1,k}^-\right) \|\xstar_k\|^2 \\
		&\qquad + \frac{1}{q} \sum_{k} \left(\beta_{2,k}^+ - \beta_{2,k}^-\right) \|\xstar_k\|^2.
	\end{align*}
	Recall $\gamma_k = {9\|\Xstar\|_F^2\ \mu^2 \kappa^2}/{(q \|\xstar_k\|^2)}$. Using $x^3e^{-x^2/2}\leq 3\sqrt{3}e^{-3/2}$, we have
	\begin{align*}
		&\beta_{1,k}^+ -  \beta_{1,k}^- = \E\left[\left(\xi^4 - \xi^2 \right)\indic_{\{ (1-\epsilon_1)\gamma_k  \leq \xi^2 \leq (1+\epsilon_1)\gamma_k  \}} \right]\\
		&= \frac{2}{\sqrt{2\pi}} \int_{\sqrt{(1-\epsilon_1)\gamma_k}}^{\sqrt{(1+\epsilon_1)\gamma_k}}  x^2\left(x^2-1\right) e^{-x^2/2}dx \\
		&\leq \frac{2}{\sqrt{2\pi}} \int_{\sqrt{(1-\epsilon_1)\gamma_k}}^{\sqrt{(1+\epsilon_1)\gamma_k}}  x^4 e^{-x^2/2} dx\\
		& \leq \frac{6\sqrt{3}e^{-3/2}}{\sqrt{2\pi}} \int_{\sqrt{(1-\epsilon_1)\gamma_k}}^{\sqrt{(1+\epsilon_1)\gamma_k}}  x dx \\
		&= \frac{6\sqrt{3}e^{-3/2}}{\sqrt{2\pi}}  \gamma_k \epsilon_1 \leq \gamma_k \epsilon_1.
	\end{align*}
	Similarly, using $xe^{-x^2/2}\leq e^{-1/2}$,
	\begin{align*}
		&\beta_{2,k}^+ -  \beta_{2,k}^- = \E\left[\xi^2 \indic_{\{ (1-\epsilon_1)\gamma_k  \leq \xi^2 \leq (1+\epsilon_1)\gamma_k  \}} \right]\\
		&= \frac{2}{\sqrt{2\pi}} \int_{\sqrt{(1-\epsilon_1)\gamma_k}}^{\sqrt{(1+\epsilon_1)\gamma_k}}  x^2e^{-x^2/2}dx \\
		&\leq \frac{2e^{-1/2}}{\sqrt{2\pi}} \int_{\sqrt{(1-\epsilon_1)\gamma_k}}^{\sqrt{(1+\epsilon_1)\gamma_k}}  xdx \\
		&= \frac{2e^{-1/2}}{\sqrt{2\pi}} \gamma_k \epsilon_1 \leq \gamma_k \epsilon_1.
	\end{align*}
	Therefore,
	\begin{align*}
		&\|\E\left[\Y_+\right]- \E\left[\Y_-\right]\|  \leq \frac{\epsilon_1}{q} \sum_k \gamma_k \|\xstar_k\|^2\\
		&= 9 \frac{\epsilon_1 \mu^2 \kappa^2}{q}  \sum_k \frac{\|\Xstar\|_F^2}{q} =  9 \epsilon_1 \frac{\mu^2 \kappa^2 \|\Xstar\|_F^2}{q}.  
	\end{align*}
	
	To lower bound $\beta_{1,k}^-$, we will use right incoherence which implies that $\gamma_k \ge 9$.
	\begin{align*}
		\beta_{1,k}^-  &= \E \left[\xi^2\left(\xi^2 -1\right)\indic_{\{\xi^2 \leq (1-\epsilon_1)\gamma_k  \}} \right]\\
		&= \E\left[\left(\xi^4-\xi^2\right) \right] -  \E \left[\xi^2\left(\xi^2 -1\right)\indic_{\{\xi^2 \geq (1-\epsilon_1)\gamma_k  \}} \right]\\
		&= 2 - 2\int_{\sqrt{(1-\epsilon_1)\gamma_k}}^{\infty} x^2(x^2-1) \frac{1}{\sqrt{2\pi}}e^{-x^2/2}dx \\
		&\geq 2 - \frac{2}{\sqrt{2\pi}}\int_{\sqrt{(1-\epsilon_1)\gamma_k}}^{\infty} x^4 e^{-x^2/2}dx \\
		&\geq 2 -\frac{7}{\sqrt{2\pi}}\int_{\sqrt{(1-\epsilon_1)\gamma_k}}^{\infty}x e^{-x^2/4}dx \\
		&=2-\frac{14}{\sqrt{2\pi}}\exp\left(-(1-\epsilon_1)\gamma_k/4  \right) > 1.5,
	\end{align*}
	where we used the fact that $x^3e^{-x^2/4} \leq 3.5$ for any $x$; $\gamma_k \ge 9$ (follows by right incoherence); and $\epsilon_1<0.01$.
\end{proof}


\begin{proof}[Proof of Lemma \ref{lemm:bound_Y-sigma-}]
	Let us define
	\begin{align*}
		\w_{ik} =  \lvert   \a_{ik}{}' \x_k^* \rvert \a_{ik}\indic_{ \left\lbrace \left( \a_{ik}{}' \x_k^*\right)^2 \leq \frac{9\left( 1-\epsilon\right) \|\X^*\|_F^2 \mu^2 \kappa^2 }{q}\right\rbrace}.
	\end{align*}
	As argued in \cite{lrpr_tsp}, which itself borrows the key idea from \cite{twf}, we can show that the $\w_\ik$s are sub-Gaussian random variables with sub-Gaussian norm $K = C \mu \kappa \|\X^*\|_F /\sqrt{q}$.
	Notice that we have defined $\Y_U$ differently in this paper (in order to be able to exploit concentration over $mq$) as compared to that in \cite{lrpr_tsp} and hence only the above argument is similar.
	
	Observe that
	\\ $mq \|\Y_- - \E\left[\Y_-\right]  \| = \max_{\z:\|\z\|=1} | \z' \sum_{ik} (\w_{ik} \w_{ik}{}' - \E[\w_{ik} \w_{ik}{}]) \z |$.
	
	First consider a fixed unit vector $\z$. Observe that $\z' \w_{ik}$ is sub-Gaussian with sub-Gaussian norm  $K = C \mu \kappa  \|\X^*\|_F /\sqrt{q}$. Thus, using Lemma \ref{ProductsubG} with $t = \epsilon_2 m \mu^2 \kappa^2 \|\X^*\|_F^2$, and $K_{X_{ik}} = K_{Y_{ik}}  =  \frac{ \mu \kappa \|\X^*\|_F }{\sqrt{q}}$, we can conclude that w.p. at least $\geq 1 - 2\exp\left(-c\epsilon_2^2 mq \right)$,
	\begin{align*}
		& \lvert \z' ( \sum_{ik}\w_{ik} \w_{ik}{}' -  mq\E\left[\Y_-\right] ) \z\rvert \leq \epsilon_2 m \mu^2 \kappa^2 \|\X^*\|_F^2,
	\end{align*}
	After this, we can use a standard epsilon-net argument to extend the bound to all unit vectors $\z$.
	With it, we can conclude that, w.p. at least $1 - 2\exp\left(n\log 9  -c\epsilon_2^2  mq \right)$, 
	\begin{align*}
		& \|\Y_- - \E\left[\Y_-\right] \| \leq \frac{1.5 \epsilon_2 \mu^2 \kappa^2  \|\X^*\|_F^2}{q} 
	\end{align*}
	%
\end{proof}

\subsection{Clarifying the sign inconsistency issue}\label{sign_issue}
Recall that we had defined $\g_k^t := (\U^t)' \xstar_k$ in \eqref{def_g}.
Since the solution of phase retrieval always comes with a phase (sign) ambiguity, at each iteration $t$, for each $k$, the output of RWF, $\bhat_k^t$, may be closer to either $\g_k^t$ or $-\g_k^t$.
This is what decides whether $\dist(\g_k^t,\bhat_k^t)$ equals $\|\g_k^t - \bhat_k^t\|$ or $\|\g_k^t + \bhat_k^t\|$. However, bound both in each proof is cumbersome. Instead we can proceed as follows. {\em Re-define} $\g_k^t$ as%
\[
\g_k^t = \begin{cases}
+(\U^t)' \xstar_k \ \text{if} \ \| (\U^t){}'\xstar_k - \bhat_k^t\| \le \| (\U^t{})'\xstar_k + \bhat_k^t\| \\
- (\U^t)' \xstar_k \ \text{otherwise}
\end{cases}
\]
and define the matrix
\[
\G^t:= [\g_1^t, \g_2^t, \dots, \g_q^t].
\]
With these new definitions,
$
\dist(\g_k^t,\bhat_k^t) = \|\g_k^t - \bhat_k^t\|
$
and $\matdist(\G^t, \hat\B^t) = \|\G^t - \hat\B_t\|_F$.

As an aside, we should point out that, even if some columns of a matrix change sign (are multiplied by $(-1)$), its singular values do not change. Thus, the minimum singular value of $\G^t$ remains the same with or without the above re-definition.

We need to do something similar to the above for $\xstar_k$'s as well. 
%
Define ${\tilde\xstar_k}^t = \xstar_k$ if $\|\xhat_k^t - \xstar_k \| \leq  \|\xhat^t_k + \xstar_k \|$ and ${\tilde\xstar_k}^t = - \xstar_k$ otherwise. Define the corresponding matrix $\tilde\Xstar^t$. 

Clearly $\matdist(\tilde\Xstar^t, \Xstar) = 0$. So, in the rest of the writing in this section, to reduce notation, we will {\em re-define}
\[
\Xstar  := \tilde\Xstar^t.
\]
With this, we can define the error/perturbation in $\xhat_k$ as just
\[
\h_k:= \xhat^t_k - \xstar_k
\]
and we have $\dist(\xhat_k^t, \xstar_k) = \|\xhat_k^t - \xstar_k \|  = \|\h_k\|$.

\subsection{Proof of Lemmas \ref{lem:bounding_distb} and \ref{incoherencebhat}} \label{proof_bhat_lems}
In this section,  we let $\a_{ik}: = \a_{ik}^{(t)}$ and $\y_{ik}:=\y_\ik^{(t)}$. Also, everywhere below, we remove the superscripts $^t$ for ease of notation.
Recall that $\xstar_k = \U \g_k + \e_k$ with $\e_k := (\I - \U \U') \xstar_k$. 

\begin{proof}[Proof of Lemma \ref{lem:bounding_distb}]
	To estimate $\b_k$, we first need to estimate $\g_k$ which requires measurements of the form $\a_{ik}{}' \U \g_k $. Our measurements satisfy
	\[
	\y_\ik = \lvert \a_{ik}{}' \U g_k \rvert + \nu_{ik},\] where $\nu_{ik} =\lvert \a_{ik}{}'  \x_k^* \rvert - \lvert \a_{ik}{}' \U \g_k \rvert$ is the noise. We use these to obtain the estimate $\bhat_k$ using RWF. By Theorem 2 of \cite{rwf}, if $m \ge Cr$, w.p. at least $1- \exp(-c m)$,
	\begin{align*}
		& \dist\left( \g_k, \bhat_k \right) \leq \frac{\|\mathbf{\nu}_k\|}{\sqrt{m}} + \left( 1-c_1 \right)^{T_{RWF,t}} \|\g_k\|,
	\end{align*}
	where $c_1$ is a constant less than one.
	For our problem,
	\begin{align*}
		&  |\nu_{ik}| \leq \lvert\ \lvert\a_{ik}{}' \U \g_k + \a_{ik}{}' \e_k   \rvert - \lvert \a_{ik}{}' \U \g_k \rvert \ \rvert\leq \lvert \a_{ik}{}' \e_k   \rvert\\
		& \|\bm{\nu}_k\|^2 = \sum_{i} \nu_{ik}^2 = \sum_{i}  \lvert \a_{ik}{}' \e_k  \rvert^2.
	\end{align*}
	Clearly, $\E[\|\bm{\nu}_k\|^2 ] = m \| \e_k \|^2$ and $ \| \e_k\|^2 \le \SE(\U,\Ustar) \|\tb_k\|$.
	Using Lemma \ref{ProductsubG} with $t = m  \delta_b   \|\e_k\|^2$, $K_{X_{i}} = K_{Y_{i}} = \|\e_k\|$, and summing over $i=1,2,\dots,m$, we conclude that,
	w.p. at least $1-\exp\left(  -c\delta_b ^2 m\right)$,
	\begin{align*}
		{\|\mathbf{\nu}_k\|^2}  \leq m (1+\delta_b) \|\e_k\|^2 \le m (1+\delta_b) \deltapt^2 \|\tb_k\|^2.
	\end{align*}
	where the last inequality used $\|\e_k\| \le \deltapt \|\tb_k\|$.
	Thus, using the above and $\|\g_k\| \le \|\tb_k\|$,
	\[
	\dist\left( \g_k, \bhat_k \right) \leq \sqrt{1+\delta_b} \deltapt \|\tilde{\b}_k^*\| + \left( 1-c_1 \right)^{T_{RWF,t}}  \|\tilde{\b}_k^*\|.
	\]
	By setting $T_{RWF,t}$  so that $(1-c)^{T_{RWF,t}} \le \deltapt$, we get that $\dist\left( \g_k, \bhat_k \right) \le C \deltapt \|\tilde{\b}_k^*\| = C \deltapt \|\xstar\|$ with $C=(\sqrt{1+\delta_b}+1) $.
	The above bound holds w.p. at least $1- \exp\left(  -c\delta_b ^2 m\right)$ for a given $k$. By union bound, it holds for all $k=1,2,\dots,q$, w.p. at least $1- q\exp\left(  -c\delta_b ^2 m\right)$. Hence, with this probability,
	\[
	\matdist(\G, \hat\B) \le C \deltapt \| \Xstar\|_F .
	\]
	For proving the third claim, recall that $\xhat_k = \U \bhat_k$ and $\xstar_k = \U \g_k + \e_k$. Let $\h_k: = \xstar_k - \xhat_k$. We can rewrite $\h_k$ as $\h_k = \xstar_k - \U \g_k + \U \g_k - \U \bhat_k$. Thus, by triangle inequality, and using $\|\e_k\| \le \SE(\U,\Ustar) \|\g_k\| \le \deltapt \|\tb_k\|$,
	\begin{align*}
		\|\h_k\| \le \|\e_k\| + \|\U\| \|\g_k - \bhat_k\| \le (1+C) \deltapt \|\tb_k\|.
	\end{align*}
\end{proof}

\begin{proof}[Proof of Lemma \ref{incoherencebhat}]
	%
	Recall that $\hat\B \qreq \R_B \B$ and so $\b_k = \R_B^{-1} \bhat_k$.
	Using Lemma \ref{lem:bounding_distb}, $\|\g_k\| \le \|\tb_k\|$, and right incoherence (which implies that $\|\tb_k\|^2 \le \sigmax^2 \mu^2 r/q$),
	\begin{align*}
		\|\b_k\|  & = \| \R_B^{-1} \left(\g_k -\bhat_k + \g_k \right) \|\\
		& \leq \| \R_B^{-1}\| \left( \dist(\bhat_k, \g_k) + \|\g_k\| \right)\\
		& \leq \| \R_B^{-1}\|  (1 + C\deltapt) \|\tb_k\|  \\
		& \leq \frac{(1 + C\deltapt) \sigmax \mu \sqrt{r/q}}{\sigma_{\min}(\R_B)} \le  \frac{1.5 \sigmax \mu \sqrt{r/q}}{\sigma_{\min}(\R_B)}
	\end{align*}
	To lower bound $\sigma_{\min}(\R_B)$, observe that $\sigma_{\min}(\R_B) = \sigma_{\min} (\hat{\B} )$. Using Lemma \ref{lem:bounding_distb}, the discussion of Sec. \ref{sign_issue},  facts from Sec. \ref{facts}, and $\SE(\U,\Ustar) \le \deltapt$,
	\begin{align*}
		\sigma_{\min}(\hat{\B}) &\geq \sigma_{\min}\left(\G\right) - \|\G - \hat{\B}\| \\
		&	\geq \sigma_{\min}(\U'\U^*) \sigma_{\min}( \tB ) - \|\G - \hat{\B}\|_F \\
		& \ge  \sqrt{1- \SE^2(\U,\Ustar)} \sigmin  - C \deltapt \|\tilde{\B}^*\|_F \\
		& \ge  \sqrt{1- \deltapt^2} \sigmin  -C \deltapt \sqrt{r} \sigmax.
	\end{align*}
	Using $\deltapt \le c/ \kappa \sqrt{r}$, $\sigma_{\min}(\R_B)=\sigma_{\min}\left(\hat{\B} \right) \ge 0.9 \sigmin$.
	Thus,
	\[
	\|\b_k\| \le  \frac{1.5 \sigmax \mu \sqrt{r/q} }{ 0.9 \sigmin } \le 2 \kappa \mu \sqrt{r/q} := \hat\mu \sqrt{r/q}.
	\]
	All of the above bounds used the bound from Lemma \ref{lem:bounding_distb}. Thus the above bounds hold w.p. at least $1-n^{-10}$ as long as $m \ge C \max(r,\log q, \log n)$.
\end{proof}

\subsection{Proof of the lemmas for Claim \ref{lem:descent}}\label{sec:key_lem_proofs}


In this section, we let $\a_{ik}: = \a_{ik}^{(T+t)}$ and $\y_{ik}:=\y_\ik^{(T+t)}$. Also, at almost all places, we remove the superscript $^t$.

All the proofs in this section use incoherence of $\B$ with parameter $\hat\mu = C \kappa \mu$ (by Lemma \ref{incoherencebhat}).  This holds w.p. at least $1 -  n^{-10}$ as long as $m \ge C \max(r,\log n, \log q)/\delta_b^2$.


\begin{proof}[Proof of Lemma \ref{sigmaminG}]
	Recall that $\mathrm{Term3}(\W)  :=  \sum_{ik}  (\a_\ik{}' \W \b_k)^2$.
	We have
	\begin{align*}
		& \E\left[  \sum_{ik} | \a_{ik}{}' \W \b_k |^2  \right] = m \| \W \B \|_F^2 = m.
	\end{align*}
	Let $X_{ik}  =  | \a_{ik}{}' \W \b_k |$. $X_{ik}$ is sub-Gaussian with sub-Gaussian norm $\|\W\b_k\|$. We use Lemma \ref{ProductsubG} with $Y_\ik = X_\ik$ and $t=\epsilon_3 m $, along with the following facts: 
	\ben
	\item $\sum_k \| \W\b_k \|^4 \le \max_k \| \W\b_k \|^2 \sum_k \| \W\b_k \|^2$,
	\item  $\sum_{k}\| \W\b_k \|^2 = \|\W\B\|_F^2 = \trace \left( \W\B\B{}' \W' \right) = \|\W\|_F^2 = 1$, and
	\item  $\max_k \| \W\b_k \|^2 \leq \|\W\|^2 \max_k  \|\b_{k}\|^2 \leq \|\W\|_F^2  \max_k\|\b_{k}\|^2 \leq \max_k\|\b_k\|^2 \le \hat\mu^2 r / q$,  w.p. at least $1 - n^{-10}$ as long as $m \ge C \max(r,\log n, \log q)/\delta_b^2$ by Lemma \ref{incoherencebhat}.
	\een
	Using Lemma \ref{ProductsubG} and the above facts, for a fixed $\W$,
	\begin{align*}
		& \Pr\left\{\lvert \sum_{ik} | \a_{ik}' \W \b_k |^2 - m \rvert \geq \epsilon_3 m  \right\} \leq 2\exp{\left(-c \frac{\epsilon_3^2 mq}{\hat\mu^2 r}\right)}.
	\end{align*}
	Now we develop an epsilon-net argument to complete the proof. This is inspired by similar arguments in \cite{candes2009tight}.
	Recall that $\S_{\W} = \left\lbrace \W \in \mathbb{R}^{n\times r}, \|\W\|_F=1  \right\rbrace$. By \cite{vershynin}(Lemma 5.2), there is a set (called epsilon-net), $\bar{\S}_W \subseteq \S_{W}$ so that for any $\W$ in $\S_{\W}$, there is a $\bar{\W} \in \bar{\S}_{\W} $, such that
	\[
	\|\bar{\W} - \W\|_F \leq \epsilon_{net}
	\]
	and
	\[
	|\bar{\S}_{\W}| \leq \left(1+\frac{2}{\epsilon_{net}} \right)^{nr}.
	\]
	Pick $\epsilon_{net} = 1/8$ so that $|\S_W| \le 17^{nr}$. Also, define
	\[
	\mathbf{\Delta}\W: = \bar{\W} - \W
	\]
	so that $\|\mathbf{\Delta}\W \|_F \le \epsilon_{net}=1/8$.
	
	Using a union bound over all entries in the finite set $\bar\S_W$,
	\begin{align}
		& \Pr\left( \lvert \sum_{ik} | \a_{ik}' \W \b_k |^2 - m \rvert \leq \epsilon_3  m, \ \text{for all} \ \bar{\W} \in \bar{\S}_W  \right)  \nonumber \\
		& \geq 1 - 2 |\bar{\S}_W| \exp{\left(-c \frac{\epsilon_3^2 mq}{\hat\mu^2 r} \right) } \nonumber \\
		& \geq 1 - 2 \exp{\left( nr (\log 17) -c \frac{\epsilon_3^2 mq}{\hat\mu^2 r} \right) }.
		\label{prob_epsnet}
	\end{align}
	Next we extend the above to obtain lower and upper bounds over the entire hyper-sphere, $\S_W$.
	Define
	\begin{align*}
		&\theta_W =  \max_{\W \in \S_{W}} \sum_{ik} |\a_{ik}{}' \W\b_k|^2,
	\end{align*}
	as the maximum of $\mathrm{Term3}(\W)$ over $\S_W$.
	Since $\frac{\Delta \W}{ \|\Delta \W\|_F} \in \S_W$,
	\[
	\sum_{ik}  | \a_{ik}{}' \Delta{\W} \b_k|^2 \le \theta_W  \|\Delta \W\|_F^2 \le \theta_W \epsilon_{net}^2.
	\]
	Using this, \eqref{prob_epsnet}, and Cauchy-Schwarz, w.p. at least $1 - 2 \exp{ \left( nr (\log 17) -c \frac{\epsilon_3^2 mq}{\hat\mu^2 r} \right) }$
	\begin{align}
		&	\sum_{ik} | \a_{ik}{}' \W \b_k |^2  \nonumber  \\
		&=\sum_{ik}  | \a_{ik}{}' \bar{\W} \b_k|^2  + \sum_{ik}  | \a_{ik}{}' \Delta{\W} \b_k|^2  \nonumber \\
		& + 2\sum_{ik} \left(\a_{ik}{}'\bar{\W}\b_{k} \right)\left(\a_{i,k}' \Delta\W \b_{k} \right)  \nonumber \\
		&\leq  (1+\epsilon_3) m + \epsilon_{net}^2  \theta_W + 2\sqrt{m (1+\epsilon_3)}\sqrt{\theta_W} \epsilon_{net} \nonumber \\
		& = (1+\epsilon_3) m + (1/64)  \theta_W + (1/4) m \sqrt{1+\epsilon_3}\sqrt{\theta_W/m}
		\label{eq:eqtmp}
	\end{align}
	The last equality just used $\epsilon_{net}=1/8$ and re-arranged the third term.
	
	If $\theta_W/m < 1$, we are done because then $\theta_W \le m$.
	Otherwise, $\theta_W/m \ge 1$ and so $\sqrt{\theta_W/m} \leq \theta_W/m$. Using this and taking $\max_{\W \in \S_{W}}$ of \eqref{eq:eqtmp},
	\begin{align*}
		\theta_W \leq  (1+\epsilon_3) m +  \theta_W ( (1/64)+ (1/4) \sqrt{1+\epsilon_3} ).
	\end{align*}
	By assumption, $\epsilon_3 < 1/10$, and so the above implies that $\theta_W \le 1.25(1+ \epsilon_3) m $.
	
	Thus, w.p. $1 - 2 \exp{\left( nr (\log 17) -c \frac{\epsilon_3^2 mq}{\hat\mu^2 r}\right) }$,
	\[
	\theta_W := \max_{\W \in \S_W} \mathrm{Term3} \le 1.25(1+ \epsilon_3) m \le 1.5 m.
	\]
	
	We now obtain the lower bound on the minimum  of $\mathrm{Term3}$ over the entire hyper-sphere. This uses \eqref{prob_epsnet}, Cauchy-Schwarz, and the upper bound on $\theta_W$ from above.
We have
	\begin{align*}
		&	\sum_{ik} | \a_{ik}{}' \W \b_k |^2  \\
		&\geq \sum_{ik}  | \a_{ik}{}' \bar{\W} \b_k|^2  + 2\sum_{ik} \left(\a_{i,k}{}'\bar{\W}\b_{k} \right)\left(\a_{i,k}{}' \Delta\W \b_{k} \right)  \\
		&\geq \sum_{ik}  | \a_{ik}{}' \bar{\W} \b_k |^2  - 2| \sum_{ik} \a_{i,k}{}'\bar{\W}\b_{k}   \a_{i,k}{}' \Delta\W \b_{k}|  \\
		&\geq m(1-\epsilon_3) - 2\sqrt{\sum_{ik} |\a_{i,k}{}'\bar{\W}\b_{k} |^2 } \sqrt{\sum_{ik}  |\a_{i,k}{}' \Delta\W \b_{k}|^2}  \\
		&\geq m(1-\epsilon_3) - 2 \sqrt{m(1+\epsilon_3)} \sqrt{\theta_W\|\Delta{\W}\|^2_F} \\
		&\geq m(1 - \epsilon_3) - 2m (1+\epsilon_3) \sqrt{1.5} \epsilon_{net} \ge m (0.9 - 0.26)=0.64m
		\end{align*}
	w.p. $1 - 2 \exp{\left( nr (\log 17) -c \frac{\epsilon_3^2 mq}{\hat\mu^2 r} \right) }$.
	In the last line we substituted $\epsilon_{net} =1/8$ and used $\epsilon_3 < 1/10$.
	
	All of the above bounds hold on the event in which $\B$ is $\hat\mu$ incoherent.  This holds w.p. at least $1 -  n^{-10}$ as long as $m \ge C \max(r,\log n, \log q)/\delta_b^2$  (this follows by Lemma \ref{incoherencebhat}).
	
	Thus, if  $m \ge C \max(r,\log n, \log q)/\delta_b^2$, w.p. $1 - 2 \exp{\left( nr (\log 17) -c \frac{\epsilon_3^2 mq}{\hat\mu^2 r} \right) } - n^{-10}$,
	\[
	\min_{\W \in \S_W} \mathrm{Term3}(\W) \ge 0.64m.
	\]	
\end{proof}

\renewcommand{\P}{\bm{P}}
\begin{proof}[Proof of Lemma \ref{product}]
Recall that
	\[
	\mathrm{Term1}(\W)= \sum_{ik} \b_k{}' \W{}' \a_{ik} \a_{ik}{}'  \p_k.
	\]
where
\begin{align*}
\p_k := \Ustar \tB \B' \b_k -  \Ustar \tb_k = \Xstar \B' \b_k - \xstar_k
\end{align*}

First we will show that $\E[\mathrm{Term1}(\W)]=0$.
	\begin{align*}
		\E\left[ \mathrm{Term1}(\W) \right] &= m \sum_{k} \b_k {}'  \W' \U^*\bm\Sigma^*(\B^* \B{}' \b_{k}  -  \b_k^*) \\
		& = m  \sum_{k} \mbox{trace} (\W' \U^* \bm\Sigma^*(\B^* \B{}'  \b_{k}  -  \b_k^*) \b_k {}' ) \\
		& = m  \mbox{trace} (\W{}' \U^*\bm\Sigma^* (\B^* \B{}'   \B\B{}'- \B^* \B{}')  )  \\
		&=  0
	\end{align*}
	where we used $\B\B'=\I_r$.
Next, we will bound $\|\p_k\|$ and $\|\P\|_F$ where $\P := [\p_1, \p_2, \dots, \p_q]$ and use these bounds to show that w.h.p. $\mathrm{Term1}(\W)$ is of order $\epsilon_1 \delta_t \|\Xstar\|_F$
%
Using  $\hat\X = \U \hat\B$, $\hat\B \qreq \R_B \B$ , $\B \B' = I$, 
\[
\hat\X \B' \b_k = \U \R_B \B \B'b_k = \U \R_B \b_k = \U \hat\b_k = \xhat_k
\]
Thus,
\begin{align*}
\p_k
& = \Xstar \B' \b_k - \xstar_k    \\
& =  (\Xstar - \hat\X + \hat\X) \B' \b_k - \xstar_k  \\
& = (\Xstar - \hat\X) \B' \b_k   + (\xhat_k - \xstar_k)   
\end{align*}
Thus, using Lemma \ref{lem:bounding_distb},
\begin{align}
\|\p_k\|
& \le \|\xstar_k - \xhat_k\| + \|\Xstar - \hat\X\| \ \|\B\| \ \|\b_k\| \nonumber \\
& \le C\delta_t \|\xstar_k\| + \|\Xstar - \hat\X\|   \|\b_k\| \nonumber \\
& \le C\delta_t \sigmax \|\bstar_k\| + C\|\Xstar - \hat\X\| \|\b_k\|    \nonumber
\end{align}
and, writing $\Xstar - \hat\X = (\U \U' + (\I - \U U') ) (\Xstar - \hat\X)$,
using Lemma \ref{lem:bounding_distb}, $\hat\X = \U \hat\B$, and $\G = \U' \Xstar$, 
\begin{align}
\|\Xstar - \hat\X\|
& = \|\U (\G - \hat\B) + (\I - \U \U') \Xstar \|   \nonumber \\  
& \le \|\G - \hat\B\| + \delta_t \|\Bstar\| \nonumber \\
& \le \|\G - \hat\B\|_F + \delta_t \sigmax  \nonumber \\
& \le C\delta_t \|\Xstar\|_F  + \delta_t \sigmax \le C\delta_t \|\Xstar\|_F  \nonumber
\end{align}
The last inequality used $\sigmax \le \|\Xstar\|_F$.
Thus,
\begin{align}
\|\p_k\|
& \le (C\delta_t \sigmax  + \|\Xstar - \hat\X\| )  \max(\|\bstar_k\|, \|\b_k\|)   \nonumber \\
& \le C\delta_t( \sigmax + \|\Xstar\|_F) \max(\|\bstar_k\|, \|\b_k\|)  \nonumber  \\
& \le C\delta_t \|\Xstar\|_F \max(\|\bstar_k\|, \|\b_k\|)
\label{bnd_pk}
\end{align}
Let
	\[
	\P :=[\p_1, \p_2, \dots, \p_q]=  \Xstar (\B'\B - \I).
	\]
To bound this, we add and subtract $\hat\X = \U \hat\B = \U \R_B \B$ from $\Xstar$ and use the facts that $\B(\B'\B - \I) = 0$ and $\|\B\B' - \I\| \le 2$ (by triangle inequality and $\|\B\|=1$). This gives
\begin{align}
\|\P\|_F
& = \|(\Xstar - \hat\X + \hat\X)(\B'\B - \I)\| \nonumber  \\
& =  \|(\Xstar - \hat\X)(\B'\B - \I)\|  \nonumber \\
& \le 2 \|\Xstar - \hat\X\| \nonumber  \\
& \le 2 \|\Xstar - \hat\X\|_F \le  C\delta_t \|\Xstar\|_F
\label{bnd_P}
\end{align}

We now use above bounds to apply Lemma \ref{ProductsubG}.
Let $X_{ik} = \a_{ik}' \W \b_k$ and $Y_{ik} =\a_{ik}' \p_k$. Both are sub-Gaussian with $K_{X_{ik}} = \|\W\b_k\| \leq \|\W\|_F \|\b_k\| \le \|\b_k\|$, and $K_{Y_{ik}}\leq  \|\p_k \| $.
Also, using Lemma \ref{incoherencebhat}, $\b_k$'s are incoherent w.p. at least $1 - 2q\exp(r \log(17) - \delta_b^2 m) \ge 1-  n^{-10}$ if $m \ge C \max(r,\log n, \log q)/\delta_b^2$, i.e.,
	\[
	\|\b_k\|^2 \le \hat\mu^2 r/q = C\kappa^2 \mu^2 r/q.
	\]
Thus, using this and incoherence of $\bstar_k$, with above probability,
\[
\max(\|\bstar_k\|^2, \|\b_k\|^2) \le C\kappa^2 \mu^2 r/q
\]
Applying Lemma \ref{ProductsubG} with $t = m \epsilon_1 \deltapt  \|\Xstar\|_F$, and using the above bounds on $\|\p_k\|$, $\|\P\|_F$ and $\max(\|\bstar_k\|^2, \|\b_k\|^2)$,
	\begin{align*}
		\frac{t^2}{\sum_{ik} K_{X_{ik}}^2 K_{Y_{ik}}^2}
		& = \frac{m^2 \epsilon_1^2  \deltapt^2 \|\Xstar\|_F^2}{m \sum_k \|\b_k\|^2  \|\p_k \|^2} \\
		& \ge \frac{m \epsilon_1^2 \deltapt^2 \|\Xstar\|_F^2}{ \max_k \|\b_k\|^2  \sum_k  \|\p_k \|^2} \\
		& = \frac{m  \epsilon_1^2 \deltapt^2 \|\Xstar\|_F^2}{ \max_k \|\b_k\|^2 \|\P\|_F^2} \\
		& \ge \frac{m  \epsilon_1^2 \deltapt^2 \|\Xstar\|_F^2}{ \max_k \|\b_k\|^2 C^2 \deltapt^2 \|\Xstar\|_F^2} \\
		& \ge \frac{mq  \epsilon_1^2}{ C \kappa^2 \mu^2 \ r}, \\
		& \text{ and} \\
		\frac{t}{\max_\ik K_{X_{ik}} K_{Y_{ik}}}
		& = \frac{m \epsilon_1 \deltapt \|\Xstar\|_F}{\max_k \|\b_k\| \|\p_k \|} \\
		& \ge \frac{m \epsilon_1 \deltapt \|\Xstar\|_F}{\max_k \|\b_k\| \deltapt \|\Xstar\|_F \max_k \max(\|\bstar_k\|, \|\b_k\|) } \\
& \ge \frac{mq \epsilon_1}{\kappa^2 \mu^2 r}
	\end{align*}
Thus,
	\[
	\Pr \lbrace |\mathrm{Term1}(\W)| \le m \deltapt^2 \|\Xstar\|_F \rbrace \ge 1 - \exp\left(-c \frac{mq \epsilon_1^2}{\kappa^2 \mu^2 r} \right)
	\]

Now we just need to extend our bound for all $\W \in \S_W$. We first extend it to all $\W$ in an epsilon-net of $\S_W$. By \cite{vershynin}(Lemma 5.2), there is a net, $\bar{\S}_W$ so that for any $\W$ in $\S_{W}$, there is a $\bar{\W}$ in $\bar{\S}_{W} $, such that $\|\bar{\W} - \W\|_F \leq \epsilon_{net}$ and $|\bar{\S}_W| \leq \left(1+\frac{2}{\epsilon_{net}} \right)^{nr}$.
	Pick $\epsilon_{net} = 1/8$. With this, $|\bar{\S}_W| \le 17^{nr}$.
	Define $\mathbf{\Delta}\W:= \bar{\W} - \W$. We have $\|\mathbf{\Delta}\W\|_F \le \epsilon_{net}=1/8$.
	Using union bound on the set $\bar{\S}_{W} $,
	\begin{align}
		&\Pr \lbrace |\mathrm{Term1}(\W)| \le m \epsilon_1^2 \|\Xstar\|_F \text{ for all } \bar{\W} \in \bar{\S}_W  \rbrace \nonumber \\
		&\geq 1 - 2|\bar{\S}_W| \exp\left(-c \frac{mq \epsilon_1^2}{\kappa^2 \mu^2 r} \right) \nonumber  \\
		&\geq 1 - 2  \exp\left(nr (\log 17) -c \frac{mq \epsilon_1^2}{\kappa^2 \mu^2 r} \right)
		\label{epsnet_bnd_p}
	\end{align}
	To extend the claim to all $\W \in \S_{W}$, define
	\[
	\theta_W :=  \max_{\W \in \S_{W}}  \sum_{ik} (\a_{ik}{}' \W\b_k)(\a_{ik}' \p_k). 
	\]
	Since $\frac{\Delta \W}{ \|\Delta \W\|_F} \in \S_W$, $\sum_{ik}  ( \a_{ik}{}' \Delta{\W} \b_k )  (\a_{ik}' \p_k) \le \theta_W \|\Delta \W\|_F \le \theta_W \epsilon_{net}$. Thus, using \eqref{epsnet_bnd_p}, for any $\W \in \S_W$,
	\begin{align*}
		&	\sum_{ik}  ( \a_{ik}{}' \W \b_k)  (\a_{ik}{}' p_k) 	 \\
		&=\sum_{ik}  ( \a_{ik}{}' \bar{\W} \b_k )  (\a_{ik}{}' \p_k) + \sum_{ik}  ( \a_{ik}{}' \Delta{\W} \b_k )  (\a_{ik}' \p_k) \\
		& \leq  m \epsilon_1^2 \|\Xstar\|_F + \theta_W \epsilon_{net} = m \epsilon_1^2 \|\Xstar\|_F + (1/8) \theta_W
	\end{align*}
	w.p. at least $1 - 2 \exp\left(nr (\log 17) -c \frac{mq \epsilon_1^2}{\kappa^2 \mu^2 r} \right)$.
	Thus, taking the $\max_{\W \in \S_{W}}$ of the above equation and solving for $\theta_W$, with the above probability,
	\[
	\theta_W \le m \epsilon_1^2 \|\Xstar\|_F / (1-\epsilon_{net}) = (8/7) m \epsilon_1^2 \|\Xstar\|_F.
	\]
	All of the above bounds hold on the event in which $\B$ is $\hat\mu$ incoherent.  
	Thus, if $m \ge C \max(r,\log n, \log q)/\delta_b^2$, w.p. at least $1 - 2 \exp\left(nr (\log 17) -c \frac{mq \epsilon_1^2}{\kappa^2 \mu^2 r} \right) - n^{-10}$,
	$\max_{\W \in \S_{W}} |\mathrm{Term1}(\W)|  \le (8/7) m \epsilon_1^2 \|\Xstar\|_F$.

Finally, \eqref{bnd_P} also implies that 
\[
\|\tB(\B' \B - I\| = \|\P\|\le \|\P\|_F \le C \deltapt \|\Xstar\|_F
\]
We use the above bound in a later proof.
\end{proof}

\begin{proof}[Proof of Lemma \ref{Show}]
	Recall that $\a_{ik}: = \a_{ik}^{(T+t)}$ and same for $\y_{ik}$. Thus, these are independent of the current $\xhat_k$'s.
	
	By Cauchy-Schwarz,
	\begin{align}
		\mathrm{Term2}(\W) & := \sum_{ik} (\cb_{ik} \hat\cb_{ik} -1) (\a_{ik}{}' \W \b_k) (\a_{ik}{}' \x_k^*) \nonumber \\
		& \leq \sqrt{ \sum_{ik} |\a_{ik}{}' \W \b_k|^2   }  \sqrt{ \sum_{ik} |\cb_{ik} \hat\cb_{ik} -1|^2  |\a_{ik}{}' \x_k^*|^2 }
		\label{Cauchy_Term2}
	\end{align}
	We can bound the first term using Lemma \ref{sigmaminG}. Consider the second term. Since $\cb_{ik} = sign(\a_{ik}{}' \x_k^*)$ and $\hat\cb_{ik} = sign(\a_{ik}{}' \xhat_k)$, clearly $(\cb_{ik} \hat\cb_{ik}-1)^2 = (4)\indic_{\left\{ \cb_{ik} \neq \hat\cb_{ik}\right\}}$. To bound this term we use the following result.
	\begin{lemma}[Lemma 1 of \cite{rwf}]
		\label{RWF1}
		\label{lem1_rwf}
		Let $\a_i$ be standard Gaussian random vectors. For any given $\xstar$, and $\xhat$ independent from $\a_i, i = 1, \cdots, m$,
		\begin{align*}
			& \Pr\left(  sign(\a_i{}' \xstar) \neq sign\left( \a_i{}' \xhat \right)   |  \left( \a_i{}' \xstar \right)^2 = z^2 , \xhat \right)  \\
			& \leq \mbox{erfc}\left( \frac{z}{2 \|\xstar - \xhat\|} \right),
		\end{align*}
		for all $\xhat$ that satisfy $\dist(\xhat,\xstar) \le 0.4$. Here $\mbox{erfc}(u) := \frac{2}{\sqrt{\pi}} \int_{u}^{\infty} \exp\left( - \tau^2 \right) d\tau$ is the complementary error function.
	\end{lemma}
	
	Let $Q_{ik} := \indic_{\left\{ \cb_{ik} \neq \hat\cb_{ik}\right\}} \cdot \left( \a_{ik}{}' \x_k^*\right)^2$ and let $Z_{ik} := \a_{ik}{}' \x_k^*$. Recall from Sec. \ref{sign_issue} that $\dist(\x_k,\xstar_k) = \|\h_k\|$ with $\h_k = \x_k^* - \xhat_k$.
	
	We first upper bound $\E[Q_{ik}|\xhat_k]$. For simplicity, we remove the subscripts $i$ and $k$ wherever these are not needed. Consider  $\E[Q]  = \E[\indic_{ \cb \neq \hat\cb} \ Z^2 ]$.
	Observe that $\cb = sign(Z)$ is a function of $Z$ and $Z$ depends on $\a$. Also $\hat\cb$ is a function of $\a$. Thus both of $\cb,\hat\cb$ are dependent on $Z$. Moreover $\hat\cb$ also depends on $\xhat$.
	We first bound $\E[Q|Z^2, \xhat]$ using the lemma stated above. 
	For any $\xhat$ that satisfies $\dist(\xhat,\xstar) \le 0.4$,
	\bea
	\E[Q |Z^2 = z^2, \xhat]
	&= & \E[ \indic_{ \cb  \neq \hat\cb} \ z^2 |Z^2 = z^2, \xhat]  \nonumber \\
	&= &      z^2 \Pr( { \cb \neq \hat\cb} |Z^2 = z^2, \xhat)  \nonumber   \\
	& \le &  z^2 \mbox{erfc}\left( \frac{z}{2 \| \xstar - \xhat\| } \right)  \nonumber  \\
	& \le &  z^2 \exp\left(-\frac{z^2}{4 \|\xstar - \xhat\|^2} \right).  \nonumber
	\eea
	The first inequality follows using Lemma \ref{lem1_rwf}, the second is a standard upper bound on the erfc function \cite{ermolova2004simplified}. 
	Thus, for any $\xhat$ that satisfies $\dist(\xhat,\xstar) \le 0.4$,
	\[	
	\E[Q | \xhat] = \E[ \E[Q |Z^2, \xhat] ]  \le \E\left[Z^2 \exp\left(-\frac{Z^2}{4 \|\xstar - \xhat\|^2} \right) | \xhat \right].
	\]
	Since $Z$ is zero mean Gaussian with variance $\|\x^*\|^2$, $Y:=Z^2/\|\x^*\|^2$ is standard chi-squared with one degree of freedom. Thus, $\E[Y] = 1$. Using this and $\exp(-y/2) < 1$, we get
	\begin{align*}
		&	\E[Q | \xhat]	\\
		& \le   \int_0^\infty y \|\xstar\|^2 \exp\left(-\frac{y \|\x^*\|^2}{4 \| \xstar - \xhat\|^2} \right)  \frac{\exp(-y/2)}{\sqrt{2y} \Gamma(1/2)} dy \\
		& \le   \int_0^\infty y \|\x^*\|^2 \exp\left(-\frac{y \|\x^*\|^2}{4 \| \xstar - \xhat\|^2} \right)  \frac{1}{\sqrt{2y} \Gamma(1/2)} dy \\
		& =   2\sqrt{2} \frac{\|\xstar - \xhat\|^3}{\|\xstar\|}    = 2\sqrt{2} \frac{\dist(\xstar,\xhat)^3}{\|\xstar\|} 
	\end{align*}
	where $\Gamma(1/2)$ is the Gamma function evaluated at $1/2$ (can treat it as a constant).
	We will now use Lemma \ref{lem:bounding_distb} to average over $\xhat$. Let $E$ be the event that $\dist(\xstar_k,\xhat_k) \le \deltapt \|\xstar_k\|$ for all $k=1,2,\dots,q$. By Lemma \ref{lem:bounding_distb}, under the lower bound on $m$, this event occurs w.p. at least $1-n^{-10}$. On the complement event, we do not have any tight bounds on $Q$. However, trivially, $Q \le \left( \a_{i}{}' \xstar\right)^2$ always and so, for any $\xhat$, $\E[Q|\xhat] \le \|\xstar\|^2$.  Thus
	\[
	\E[Q] \le 4 \deltapt^3 \|\xstar\|^2 \cdot (1-n^{-10}) +  \|\xstar\|^2 \cdot n^{-10} 
	\]
	We will eventually set $\deltapt = c/ (\kappa^2 r)$ which is larger than $n^{-10}$. Thus,
	\[
	\E[Q] \le C \max(\deltapt^3, n^{-10}) \|\xstar\|^2 =  C \deltapt^3 \|\xstar\|^2
	\]
	and so, if $m \ge C \max(r,\log n, \log q)$,
	\begin{align}
		\sum_{ik}  \E\left[ Q_{ik} \right] \le m C  \deltapt^3 \|\Xstar\|_F^2.	
		\label{EQsum}
	\end{align}
	As a side-note, we should clarify here that the dependence on $\xhat$ matters in only the above expected value computation because this is the only term where we have upper bounded the expectation using Lemma \ref{lem:bounding_distb}. The expected value of the other two terms is the same for all values of $\xhat$ and hence we ignore the dependence there.
	Everywhere else, we use this lemma only while obtaining the high probability error bounds, and of course we assume the bounds hold on the intersection of the desired event with $E$. 
	
	Next we show that, on the event $E$, whp, $\sum_{ik} Q_\ik$ is of the same order.
	As shown in the proof of Theorem 1 of \cite{rwf}, $\cb_{ik} \neq \hat\cb_{ik}$ implies that\footnote{This follows from $(\a'\xstar)^2 = (\a'\h)^2 + (\a'\xhat)^2  + 2(a'\h) (a'\xhat) =  (\a'\h)^2 -  (\a'\xhat)^2 + 2 (\a'\xstar)  (\a'\xhat) \le (\a'\h)^2$. The inequality holds because $\cb_{ik} \neq \hat\cb_{ik}$ means that the last term is negative.}
	$\left(\a_{ik}{}' \x_k^*\right)^2 \leq \left(\a_{ik}{}' \h_k \right)^2$. Here $\h_k = \x_k^* - \xhat_k$.
	Thus, $Q_{ik} = \indic_{\left\{ \cb_{ik} \neq \hat\cb_{ik}\right\}} \left( \a_{ik}{}' \x_k^*\right)^2 \le \indic_{\left\{ \cb_{ik} \neq \hat\cb_{ik}\right\}} \left( \a_{ik}{}' \h_k \right)^2 \le \left( \a_{ik}{}' \h_k \right)^2$. Thus, it is a sub-exponential r.v., or equivalently it is a product of sub-Gaussian r.v.'s $\sqrt{Q_\ik}$. Thus, we can apply  Lemma \ref{ProductsubG} with $K_{X_\ik} = K_{Y_\ik} =  \|\h_k\|$. Pick  $t= m \epsilon_2 \deltapt^2 \|\Xstar\|_F^2$. Observe that
	\begin{align*}
		\frac{t^2}{\sum_{ik}  \|\h_k\|^4} & =  \frac{m^2 \epsilon_2^2  ( \deltapt^2 \|\Xstar\|_F^2 )^2 }{m  \sum_k \|\h_k\|^4 }\\
		&\ge  \frac{m \epsilon_2^2 \deltapt^4 \|\Xstar\|_F^4}{ \max_k \|\h_k\|^2 \deltapt^2 \|\Xstar\|_F^2 }\\
		&\ge \frac{m \epsilon_2^2  \|\Xstar\|_F^2  }{ \max_k \|\xstar_k\|^2 }\\
		&\ge  \frac{m q \epsilon_2^2 }{ \kappa^2 \mu^2},
\text{ and} \\
		\frac{t}{\max_k \|\h_k\|^2} & \ge \frac{ m \epsilon_2 \deltapt^2  \|\Xstar\|_F}{ \deltapt^2 \max_k \|\xstar_k\|^2}  \ge \frac{\epsilon_2 m q }{ \kappa^2 \mu^2}.
	\end{align*}
	The second inequality used $\|\h_k\|^2 \le \deltapt^2 \|\xstar_k\|^2$ which holds on the event $E$.  The third inequality used the fact that right incoherence implies $\|\xstar_k\|^2 \le \mu^2 \kappa^2 \|\Xstar\|_F^2 / q$.
	Thus, on the event $E$,
	\bea
	&\Pr \left\{ | \sum_{ik}Q_{ik} - \sum_\ik \E[Q_{ik}]| \geq m \epsilon_2 \deltapt^2 \|\Xstar\|_F^2 \right\} \nonumber \\
	& \leq 2 \exp{\left( -c\epsilon_2^2 m q / \kappa^2 \mu^2 \right)} \nonumber
	\eea
	Thus, using \eqref{EQsum}, and $\Pr(E)>1-n^{-10}$ (by Lemma \ref{lem:bounding_distb}),
	if $m \ge C \max(r,\log n, \log q)$, w.p. at least $1 - n^{-10} -  2 \exp{\left( -c\epsilon_2^2 m q / \kappa^2 \mu^2 \right)}$,
	\begin{align}
		\sum_{ik}  Q_{ik} \le  (C+1) m \epsilon_2 \deltapt^2 \|\Xstar\|_F^2.  
		\label{eq:cs_1}
	\end{align}	
	Finally, combining \eqref{Cauchy_Term2}, \eqref{eq:cs_1}, and Lemma \ref{sigmaminG}, if $m \ge C \max(r,\log n, \log q)/\delta_b^2$,
	w.p. at least $1 - n^{-10} - 2 \exp\left(nr - c \epsilon_3^2 \frac{mq}{\hat\mu^2 r}\right) -   2 \exp{\left( -c \epsilon_2^2 \frac{m q}{\kappa^2 \mu^2} \right)}$,
	\begin{align*}
		\max_{\W \in \S_{\W}} \mathrm{Term2}(\W)  
		& \le C m \sqrt{1+\deltapt} \sqrt{\deltapt + \epsilon_2}\deltapt \|\Xstar\|_F.
	\end{align*}
\end{proof}

\begin{proof}[Proof of Lemma \ref{lem:bound_RU}]
	Using facts from Sec. \ref{facts}, and using $\sigma_i(\Ustar) = 1$,
	\begin{align*}
		\sigma_{\min}(\Ustar \bm\Sigma^* \B^* \B')
		&\geq   \sigma_{\min}(\Ustar) \sigma_{\min}(\bm\Sigma^*) \sigma_{\min}(\B^* \B')  \\
		& \ge   \sigmin \sqrt{1-\SE^2(\Bstar', \B')}
	\end{align*}
	To upper bound $\SE(\Bstar', \B')$, first notice that $\Bstar'$ and $\B'$ are basis matrices. Thus, $\SE(\Bstar', \B')  = \| \B^*\left(\I - {\B}'\B  \right) \|$.
	We have upper bounded $\| \tB \left(\I - {\B}'\B  \right) \|$ in Lemma \ref{product}. Also recall that $\tB = \bm\Sigma^* \Bstar$ and $\|\tB\|_F = \|\Xstar\|_F \le \sqrt{r}\sigmax$. Thus,
	\begin{align*}
		\SE(\Bstar', \B') = \|(\bm\Sigma^*)^{-1}  \tB \left(\I - \B'\B \right) \|
		& \leq \frac{C \deltapt  \|\Xstar\|_F}{\sigmin} \\
		& \leq   C\sqrt{r} \kappa \deltapt   
	\end{align*}
	Thus, using $\deltapt \le 0.7/C\sqrt{r} \kappa$, we will have  $\sigma_{\min}(\Ustar \bm\Sigma^* \B^* \B')  \ge 0.95 \sigmin$.
\end{proof}

\section{Proof of  Lemma \ref{lem:key_lem}} \label{proof_key_lem}

We begin by defining a few quantities needed for writing an expression of $\Uhat^{t+1}$ in closed form.
\ben
\item We use the subscript $\M_{vec}$ to refer to the vectorized version of matrix $\M$. Thus, for example, $\Ustar_{vec}$ is an $nr \times 1$ vectorized version of the basis matrix $\Ustar$. When updating its estimate by LS, we first obtain an expression for the vectorized version $\Uhat^{t+1}_{vec}$ and then rearrange it as an $n \times r$ matrix.

\item Scalars such as $b(q)$ refer to the $q$-{th} entry of vector $\b$.

\item Define the diagonal matrix ${\C}_k  := \text{diag}(\cb_\ik)$,  and recall from the algorithm that $\Chat_k= \text{diag}( \hat\cb_\ik )$. Here  $\cb_{ik}:=\text{sign}({\a_{ik}{}'\x_k^*})$ and $\hat\cb_{ik}:=\text{sign}({\a_{ik}{}'\xhat_k})$ as defined in Lemma \ref{lem:key_lem}.

\item For $k =1,\ldots,q$, define
\begin{align*}
	&\B_{k,mat} := [b_k(1)\I_n,\ldots,b_k(r)\I_n ]'\\
	&\dd := \sum_k   \B_{k,mat} \A_k\Chat_k \y_k  \\ 
	&		\GG  :=  \sum_{k} \B_{k,mat}\A_k\A_k' \B_{k,mat}'.
\end{align*}
Observe that $\B_{k,mat}$ and $\GG$ are $nr \times n$ and $nr \times nr$ matrices respectively while $\dd$ is an $nr \times 1$ vector. Moreover, for any $\w \in \mathbb{R}^{nr\times 1}$ it is easy to see that $\w'\B_{k,mat} = \b_k'\W'$ where $\W$ is matrix version of $\w$ with $\W \in \mathbb{R}^{n\times r}$.

\een
Recall from the algorithm that
\[\Uhat^{t+1}_{vec} = \mbox{argmin}_{\Uhat_{vec}} \sum_k \| \Chat_k \y_k - \A_k'\B_{k,mat}'\Uhat_{vec} \|^2. \]
This is an LS problem, it can be solved in closed form to give
\[
\Uhat^{t+1}_{vec} = \GG^{-1} \dd
\]
which is $nr$-length vector. We get the matrix $\Uhat^{t+1}$ by reshaping this vector into an $n \times r$ matrix.

To simplify the above expression, first recall that $\y_k = |\A_k{}' \xstar_k|$ and $\C_k$ is the diagonal matrix containing the signs of $(\A_k{}'\xstar_k)_i$: its $i,i$-th entry is the sign of $(\a_\ik{}' \xstar_k)$. Thus,
\[
\y_k = \C_k \A_k{}' \xstar_k.
\]
Since $\xstar_k = \Ustar \bm\Sigma^* \b_k^* $, we can rewrite $\dd$ as
\[
\dd = \sum_k   \B_{k,mat} \A_k \Chat_k {\C}_k \A_k{}' \U^* \bm\Sigma^* \b_k^* .
\]
Before proceeding further, we define a few more quantities.

\ben
\item Define
\[
\SSS  =   \sum_k \B_{k,mat} \Bstar_{k,mat}{}'.
\]	

\item Define the ``expanded'' singular value matrix which is of size $nr \times nr$,
\[
\bm\Sigma_{big}^* := \text{diag}(\sigma_1^*\I_n, \cdots, \sigma_r^*\I_n)
\]
where $\sigma_i^*$ are the singular values of $\Xstar$.

\item In order to separate the contribution of phase error from the rest, split $\dd$ as $\dd = \dd^{(1)} + \dd^{(2)}$ where
\begin{align*}
	& \dd^{(1)}
	= \sum_k   \B_{k,mat}  \A_k{}' \U^* \bm\Sigma^* \b_k^*   \nonumber \\
	& \dd^{(2)} = \sum_k  \B_{k,mat} (\Chat_k \C_k - \I) \A_k{}' \U^* \bm\Sigma^* \b_k^*, \nonumber   
\end{align*}
Thus,
\[
\Uhat^{t+1}_{vec} = \GG^{-1} \dd = \GG^{-1} (\dd^{(1)} + \dd^{(2)}).
\]

\item  Define the $nr$-length vector $\F_{vec}$ as follows
\begin{align*}
	& \F_{vec}
	= \GG^{-1} (\GG \SSS \bm\Sigma_{big}^* \U^*_{vec}  - \dd^{(1)} )  - \GG^{-1}\dd^{(2)}.
\end{align*}
and let  $\F \in \mathbb{R}^{n\times r}$ be the reshaped matrix formed from $\F_{vec}$. 
\end{enumerate}

We will now show that
\begin{equation}
\Uhat^{t+1} = \U^* \bm\Sigma^* \B^* \B' - \F.
\label{Uhat_express}
\end{equation}
This will be useful because when we try to bound $\SE(\Ustar, \U^{t+1})$, the first term will disappear. 
To do this, we add and subtract the vector $\GG \SSS \bm\Sigma^*_{big} \Ustar_{vec}$ from $\dd^{(1)}$. This gives
%
\begin{align*}
& \Uhat_{vec}^{t+1}= \SSS \bm\Sigma_{big}^* \U^*_{vec} - \F_{vec}
\end{align*}
Next we explain why the $n \times r$ reshaped matrix version of the vector $\SSS \bm\Sigma_{big}^* \U^*_{vec}$ equals $\Ustar \bm\Sigma^* \Bstar \B'$.
We have
\begin{align*}
\SSS \bm\Sigma_{big}^* \U^*_{vec} &= \sum_k \B_{k,mat} \B_{k,mat}^*{}' \bm\Sigma_{big}^* \U^*_{vec} \\
&= \sum_k \B_{k,mat} \Ustar \bm\Sigma^* \b_k^*\\
&= \sum_k \begin{bmatrix}
	& b_k (1)\Ustar \bm\Sigma^* \b_k^*\\
	& b_k (2)\Ustar \bm\Sigma^* \b_k^*\\
	&\vdots\\
	& b_{k} (r)\Ustar \bm\Sigma^* \b_k^*
\end{bmatrix}
\end{align*}
Matrix version of the above vector has $p$-th column as $\Ustar \bm\Sigma^* \left(\sum_k \b_k^*b_k(p)\right) $. This implies that matrix version of this vector is $\Ustar \bm\Sigma^* \Bstar \B'$.
%
Moreover, it is easy to see that $\B_{k,mat}' \SSS \bm\Sigma_{big}^* \U^*_{vec} = \Ustar \bm{\Sigma}^*\Bstar \B' \b_k$.


In the rest of this section, we use \eqref{Uhat_express} to obtain the desired bound on $\SE(\U^*, \U^{t+1})$. Recall that $\Uhat_{t+1} \qreq \U^{t+1} \R_U$.
%
Thus, $\U^{t+1} = \Uhat^{t+1}  (\R_U)^{-1} = ( \U^* \bm\Sigma^* \B^* {\B }' - \F )  (\R_U)^{-1}$ and so 
\begin{align}\label{eq:se_1_old}
\SE(\U^*, \U^{t+1}) = & \| \U^{*'}_\perp \F (\R_U)^{-1} \| \le \|\F  (\R_U)^{-1} \|  \nonumber \\
\leq & \| \F   \|_F \  \| \R_U^{-1} \| = \frac{\| \F_{vec} \|}{\sigma_{\min}(\R_U)} 
\end{align}
Since $\sigma_{\min}(\R_U) = \sigma_{\min}(\U^{t+1})$, we have  $\sigma_{\min}(\R_U) = \sigma_{\min}(\U^* \bm\Sigma^* \B^* \B' - \F) \ge \sigma_{\min}(\U^* \bm\Sigma^* \B^* \B') - \|\F\| \ge
\sigma_{\min}( \Ustar \bm\Sigma^* \B^* \B') - \|\F_{vec}\|$. 
Thus,
\begin{align}\label{eq:se_1}
\SE(\U^*, \U^{t+1}) \le \frac{\| \F_{vec} \|}{  \sigma_{\min}(\Ustar \bm\Sigma^* \B^* \B') - \|\F_{vec}\|}
\end{align}


In the rest of this proof, we show that  $\|\F_{vec}\|$ is upper bounded by $\mathrm{MainTerm}$. We have
\begin{align}\label{eq:se_2}
\|\F_{vec}\| \leq & \|\GG^{-1}\| \left(\|\GG \SSS  \bm\Sigma_{big}^* \U^*_{vec} - \dd^{(1)}  \|+  \|\dd^{(2)}\|\right)
\end{align}
Consider the first term, $\GG^{-1}$. Since $\GG$ is a symmetric positive semidefinite matrix
\begin{align*}
\sigma_{\min}(\GG) &= \min_{\w \in \mathbb{R}^{nr\times 1}, \ \|\w\|=1} \w'\GG\w
\end{align*}
For all $\w \in \mathbb{\R}^{nr \times 1}, \|\w\|^2 = 1$, we can write $\w' \GG \w $
\begin{align*}
&\sum_{k} \w' \B_{k,mat}\A_k\A_k' \B_{k,mat}' \w = \sum_{k} \b_k'\W'\A_k\A_k'\W\b_k \\
&= \sum_{ik} | \a_{ik}' \W \b_k |^2 = \mathrm{Term3}(\W)
\end{align*}
where $\W \in \mathbb{\R}^{n \times r}$ is the matrix version of $\w$ and $\w = \W_{vec}$.
Recall that $\S_W = \{\W \in \mathbb{R}^{n \times r}, \ \|\W\|_F=1 \} = \{\w \in \mathbb{R}^{nr\times 1}, \ \|\w\|=1 \}$.
Thus,
\begin{align}\label{eq:se_3}
\|\GG^{-1}\| = \frac{1}{\sigma_{\min}(\GG)} = \frac{1}{\min_{\W \in \S_{\W}} |\mathrm{Term3}(\W)|}
\end{align}	
Now consider the first term inside the parenthesis. 
Using the variational definition, 
\begin{align*}
&	\| \GG \SSS \bm\Sigma_{big}^* \U_{vec}^* - \dd^{(1)} \|  =  \\
& \max_{\w \in \mathbb{R}^{nr \times 1}, \|\w\|=1}  | \w' (\GG \SSS \bm\Sigma_{big}^* \U_{vec}^*- \dd^{(1)}) |.
\end{align*}
It follows from definitions that
\begin{align*}
& \w' \GG \SSS \bm\Sigma_{big}^* \U_{vec}^* =  \\
& \sum_{k}\left(\w'\B_{k,mat}\right)\A_k \A_k' \left(\B_{k,mat}'	\SSS \bm\Sigma_{big}^* \U_{vec}^*\right) =  \\
& \sum_{k} \b_k'\W'\A_k \A_k' \Ustar \bm{\Sigma}^*\Bstar \B' \b_k .
\end{align*}
Similarly
\begin{align*}
& \w' \dd^{(1)} = \sum_{k} \b_k {}' \W{}' \A_k \A_k{}' \U^* \bm\Sigma^* \b_k^*,
\end{align*}
and thus
\begin{align}\label{eq:se_4}
\| \GG \SSS \bm\Sigma_{big}^* \U_{vec}^* - \dd^{(1)} \| = \max_{\W \in \S_{\W}} |\mathrm{Term1}(\W)|.	
\end{align}
For the final term, $\|\dd^{(2)} \|$, by variational definition we have,
\begin{align*}
\|\dd^{(2)} \| &=  \max_{\w\in \mathbb{R}^{nr\times 1} : \|\w\|=1}  \w' \dd^{(2)}.  
\end{align*}
From definitions we know that
\begin{align*}
\w{}' \dd^{(2)} &=\sum_k   \w'\B_{k,mat} \A_k(\Chat_k \C_k - \I) \A_k{}' \U^* \bm\Sigma^* \b_k^*\\
&= \sum_k \b_k'\W'\A_k (\Chat_k \C_k - \I) \A_k{}' \U^* \bm\Sigma^* \b_k^*\\ &=\sum_{ik}(\hat\cb_{ik} \cb_{ik} - 1) \left(\a_{ik}{}' \W \b_{k}  \right)     \left( \a_{ik}{}' \x_{k}^*\right) = \mathrm{Term2}(\W),
\end{align*}
and thus
\begin{align}\label{eq:se_5}
\|\dd^{(2)} \| = \max_{\W \in \S_{\W}} |\mathrm{Term2}(\W)|.
\end{align}
Combining \eqref{eq:se_2} - \eqref{eq:se_5},
\begin{align*}
&	\|\F_{vec}\| \\
& \leq \frac{\max_{\W \in \S_{\W}} |\mathrm{Term1}(\W)| + \max_{\W \in \S_{\W}} |\mathrm{Term2}(\W)|}{\min_{\W \in \S_{\W}} |\mathrm{Term3}(\W)|} 
\end{align*}
Combining the above bound with \eqref{eq:se_1} proves the lemma.

\section{Proof Sketch of Corollary \ref{thm:pst}} \label{proof_pst}
Suppose first that the subspace change times $k_j$ were known. By our assumption, $k_{j+1} - k_j > T \alpha$.
Then the proof is almost exactly the same as that for the static case. 
The only difference is that, in the current case, every $\alpha$ time instants, we are using measurements corresponding to a new set of $\alpha$ signals (columns of $\Xstar_\full$) but we use the estimate of the subspace obtained from the measurements for the previous $\alpha$ time instants.
As long as the subspace has not changed between the two intervals,  Claims \ref{lemm:bounding_U} and \ref{lem:descent} apply without change. Combining them, we can again conclude that $\SE(\U_{\sub,(j)}^0, \Ustar_{\sub,(j)}) \le \deltinit$ at $k= k_j+\alpha$, and that the bound decreases $0.7$ times after each $\alpha$-length epoch so that $\SE(\U_{\sub,(j)}^T, \Ustar_{\sub,(j)}) \le \epsilon$  at $k=k_j + \alpha T$. By our assumption, $k_{j+1} > k_j + \alpha T$ so this happens before the next change.

The proof in the unknown $k_j$ case follows if we can show that, whp, $k_j \le \hat{k}_j \le \hat{k}_j + 2\alpha$. This can be done using Lemma \ref{lem:changedet} given below along with the following argument borrowed from \cite{rrpcp_icml,rrpcp_isit}. Consider the $\alpha$-length interval in which $k_j$ lies. Assume that, before this interval, we have an $\epsilon$-accurate estimate of the previous subspace.
In this interval, the first some data vectors satisfy $\xstar_k = \Ustar_{\sub,(j-1)} \td_k$, while the rest satisfy $\xstar_k = \Ustar_{\sub,(j)}\td_k$. By our assumption, this interval lies in the ``detect phase''. We cannot guarantee whether the change will get detected in this interval, but it may. However, in the interval after this interval, all data vectors satisfy $\xstar_k = \Ustar_{\sub,(j)} \td_k$. In this interval, Lemma \ref{lem:changedet} given below can be used to show that the change gets detected whp. Thus, either the change is detected in the first interval itself (the one that contains $k_j$), or it is not. If it is not, then, by Lemma \ref{lem:changedet}, whp, it {\em will} get detected in the second interval (in which all signals are generated from the $j$-th subspace).
Thus, $\hat{k}_j \le k_j + 2\alpha$. See Appendix A of \cite{rrpcp_icml} for a precise proof of this idea.  The key point to note here is that we are never updating the subspace in the interval that contains $k_j$ and hence we do not have to prove a new descent lemma that deals with the interval in which the subspace changes. 


{\em We will replace $\alpha$ by $q$ in the following lemma and its proof, in order to able to use bounds from earlier proofs. Thus in this lemma, we are considering a $q$-frame epoch.}
\begin{lemma}\label{lem:changedet}
	Consider the $(n-r) \times (n-r)$ matrix
	\begin{align*}
		& \Y_{U,det}  :=  \U_{\sub,(j-1),\perp}{}' \Y_U (\J_q)  \U_{\sub,(j-1),\perp},
	\end{align*}
	Assume that $\|\Y_U - \E[\Y_-] \| \leq \frac{\deltinit \sigmin^2}{q}$.  This is true by \eqref{yu_bnd}.
	Assume that $\SE(\U_{\sub,(j-1)},\Ustar_{\sub,(j-1)}) \le \epsilon$.
	Then,
	\begin{enumerate} 
		\item	 If $\J_q \subseteq [k_j, k_{j+1})$ (change has occurred), then
		\begin{align*}
			& \lambda_{\max}\left( \Y_{U,det} \right) - \lambda_{\min}\left( \Y_{U,det} \right)\\
			& \geq \frac{\sigmin^2}{q}	\left(1.5 \SE^2\left( \Ustar_{\sub,(j)} , \U_{\sub,(j-1)} \right) - 2 \deltinit \right) \\
			& \geq \frac{\sigmin^2}{q}	(1.5 (\SE( \Ustar_{\sub,(j)}, \Ustar_{\sub,(j-1)} ) - 2 \epsilon)^2 - 2 \deltinit )
		\end{align*}
		\item If $\J_q \subseteq [\hat{k}_{j-1} + T q, k_{j})$ (change has not occurred), then
		\begin{align*}
			&  \lambda_{\max}\left( \Y_{U,det} \right) - \lambda_{\min}\left( \Y_{U,det} \right)\\
			&\leq\frac{1}{q} \sigmax^2 \SE^2\left( \Ustar_{\sub,(j-1)} , \U_{\sub,(j-1)} \right) + \frac{2\deltinit \sigmin^2}{q}  \\
			&\leq
			\frac{\sigmin^2}{q}  ( \kappa^2 \epsilon^2) + 2 \deltinit )
		\end{align*}
	\end{enumerate}		
\end{lemma}

\begin{proof}
	This proof uses the following fact:
	For basis matrices $\bP_1, \bP_2, \bP_3$ of the same size,
	$\SE\left( \bP_1, \bP_2\right)- 2 \SE\left( \bP_2, \bP_3\right) \leq \SE\left( \bP_1, \bP_3\right) \leq \SE\left( \bP_1, \bP_2\right)+\SE\left( \bP_2, \bP_3\right)$.

	Define the $(n-r) \times (n-r)$ matrix
	\begin{align*}
		& \tSigma = \U_{\sub,(j-1),\perp}{}' \E[\Y_-] \U_{\sub,(j-1),\perp}{}.
	\end{align*}
	\textbf{Proof of item 1}
	\begin{align*}
		& \lambda_{\max}\left( \Y_{U,det} \right) \geq  \lambda_{\max}\left( \tSigma\right) - \|\Y_{U,det} - \tSigma\|\\
		& \|\Y_{U,det} - \tSigma\| \\ 
		&\leq \|\Y_U - \E[\Y_-] \| \leq \frac{\deltinit \sigmin^2}{q}.
	\end{align*}
	Also we have
	\begin{align*}
		& \lambda_{\min}\left( \Y_{U,det} \right) \leq \lambda_{\min}\left( \tSigma\right) + \|\Y_{U,det} - \tSigma\|
	\end{align*}
	Thus using the facts from Sec. \ref{facts} and $\min_k \beta_{1,k}^- \ge 1.5$ (proved while proving Claim \ref{lemm:bounding_U} for initializing $\Ustar$),
	\begin{align*}
		& \lambda_{\max}\left( \Y_{U,det} \right) - \lambda_{\min}\left( \Y_{U,det} \right) \\
		& \geq \lambda_{\max}\left( \frac{1}{q} \U_{\sub,(j-1),\perp}{}' \U_{\sub,(j)}^* \  (\sum_k \beta_{1,k}^{-} \tb_k \tb_k{}') \times \right. \\
		& \left.	 \U_{\sub,(j)}^*{}' \U_{\sub,(j-1),\perp} \right) - \frac{2\deltinit \sigmin^2}{q} \\
		& \geq \frac{(\sigmin)^2}{q} \left(1.5 \SE\left( \U_{\sub,(j)}^* , \U_{\sub,(j-1)} \right)^2 - 2\deltinit \right).
	\end{align*}
	%
	\textbf{Proof of item 2}
	\begin{align*}
		\lambda_{\max}\left( \Y_{U,det} \right) - \lambda_{\min}\left( \Y_{U,det} \right)
		\leq & \lambda_{\max}\left( \tSigma\right) - \lambda_{\min}\left( \tSigma\right)\\
		&+2\|\Y_{U,det} -\tSigma\|.
	\end{align*}
	It is easy to see that
	\begin{align*}
		&  \lambda_{\max}\left( \tSigma\right)  -  \lambda_{\min}\left( \tSigma\right) \\
		& \leq  \frac{ \max_k \beta_{1,k}^{-}}{q} \sigmax^2 \SE^2\left( \Ustar_{\sub,(j-1)} , \U_{\sub,(j-1)} \right) \\
		& \le  \frac{ \sigmax^2}{q} \SE^2\left( \Ustar_{\sub,(j-1)}, \U_{\sub,(j-1)} \right)
	\end{align*}
\end{proof}




\begin{thebibliography}{10}
	
	\bibitem{lrpr_icml}
	S.~Nayer, P.~Narayanamurthy, and N.~Vaswani,
	\newblock ``Phaseless pca: Low-rank matrix recovery from column-wise phaseless
	measurements,''
	\newblock in {\em Intnl. Conf. Machine Learning (ICML)}, 2019.
	
	\bibitem{fineup}
	James~R Fienup,
	\newblock ``Phase retrieval algorithms: a comparison,''
	\newblock {\em Applied optics}, vol. 21, no. 15, pp. 2758--2769, 1982.
	
	\bibitem{ger_saxton}
	R.~W. Gerchberg and W.~O. Saxton,
	\newblock ``A practical algorithm for the determination of phase from image and
	diffraction plane pictures,''
	\newblock {\em Optik}, 1972.
	
	\bibitem{candes2013phase}
	E.~J. Candes, Y.~C. Eldar, T.~Strohmer, and V.~Voroninski,
	\newblock ``Phase retrieval via matrix completion,''
	\newblock {\em SIAM J. Imaging Sci.}, vol. 6, no. 1, pp. 199--225, 2013.
	
	\bibitem{candes_phaselift}
	E.~J. Candes, T.~Strohmer, and V.~Voroninski,
	\newblock ``Phaselift: Exact and stable signal recovery from magnitude
	measurements via convex programming,''
	\newblock {\em Comm. Pure Appl. Math.}, 2013.
	
	\bibitem{pr_altmin}
	P.~Netrapalli, P.~Jain, and S.~Sanghavi,
	\newblock ``Phase retrieval using alternating minimization,''
	\newblock in {\em NIPS}, 2013, pp. 2796--2804.
	
	\bibitem{altmin_irene_w}
	Ir{\`e}ne Waldspurger,
	\newblock ``Phase retrieval with random gaussian sensing vectors by alternating
	projections,''
	\newblock {\em IEEE Transactions on Information Theory}, vol. 64, no. 5, pp.
	3301--3312, 2018.
	
	\bibitem{wf}
	E.~J. Candes, X.~Li, and M.~Soltanolkotabi,
	\newblock ``Phase retrieval via wirtinger flow: Theory and algorithms,''
	\newblock {\em IEEE Trans. Info. Th.}, vol. 61, no. 4, pp. 1985--2007, 2015.
	
	\bibitem{twf}
	Y.~Chen and E.~Candes,
	\newblock ``Solving random quadratic systems of equations is nearly as easy as
	solving linear systems,''
	\newblock in {\em NIPS}, 2015, pp. 739--747.
	
	\bibitem{rwf}
	H.~Zhang, Y.~Zhou, Y.~Liang, and Y.~Chi,
	\newblock ``Reshaped wirtinger flow and incremental algorithm for solving
	quadratic system of equations,''
	\newblock in {\em NIPS}, 2016.
	
	\bibitem{taf}
	G.~Wang, G.~B. Giannakis, and Y.~C. Eldar,
	\newblock ``Solving systems of random quadratic equations via truncated
	amplitude flow,''
	\newblock {\em arXiv preprint arXiv:1605.08285}, 2016.
	
	\bibitem{pr_mc_reuse_meas}
	Cong Ma, Kaizheng Wang, Yuejie Chi, and Yuxin Chen,
	\newblock ``Implicit regularization in nonconvex statistical estimation:
	Gradient descent converges linearly for phase retrieval, matrix completion
	and blind deconvolution,''
	\newblock in {\em Intnl. Conf. Machine Learning (ICML)}, 2018.
	
	\bibitem{pr_random_init}
	Yuxin Chen, Yuejie Chi, Jianqing Fan, and Cong Ma,
	\newblock ``Gradient descent with random initialization: Fast global
	convergence for nonconvex phase retrieval,''
	\newblock {\em arXiv preprint arXiv:1803.07726}, 2018.
	
	\bibitem{voroninski13}
	Xiaodong Li and Vladislav Voroninski,
	\newblock ``Sparse signal recovery from quadratic measurements via convex
	programming,''
	\newblock {\em SIAM Journal on Mathematical Analysis}, vol. 45, no. 5, pp.
	3019--3033, 2013.
	
	\bibitem{jaganathan2013sparse2}
	Kishore Jaganathan, Samet Oymak, and Babak Hassibi,
	\newblock ``Sparse phase retrieval: Uniqueness guarantees and recovery
	algorithms,''
	\newblock {\em arXiv preprint arXiv:1311.2745}, 2013.
	
	\bibitem{sparta}
	G.~Wang, L.~Zhang, G.~B. Giannakis, M.~Akcakaya, and J.~Chen.,
	\newblock ``Sparse phase retrieval via truncated amplitude flow,''
	\newblock {\em arXiv preprint arXiv:1611.07641}, 2016.
	
	\bibitem{cai}
	T.T. Cai, X.~Li, and Z.~Ma,
	\newblock ``Optimal rates of convergence for noisy sparse phase retrieval via
	thresholded wirtinger flow,''
	\newblock {\em The Annals of Statistics}, vol. 44, no. 5, pp. 2221--2251, 2016.
	
	\bibitem{fastphase}
	Gauri Jagatap and Chinmay Hegde,
	\newblock ``Sample-efficient algorithms for recovering structured signals from
	magnitude-only measurements,''
	\newblock {\em IEEE Transactions on Information Theory}, 2019.
	
	\bibitem{st_imaging}
	Zhi-Pei Liang,
	\newblock ``Spatiotemporal imagingwith partially separable functions,''
	\newblock in {\em 2007 4th IEEE International Symposium on Biomedical Imaging:
		From Nano to Macro}. IEEE, 2007, pp. 988--991.
	
	\bibitem{lrpr_tsp}
	N.~Vaswani, S.~Nayer, and Y.~C. Eldar,
	\newblock ``Low rank phase retrieval,''
	\newblock {\em IEEE Trans. Sig. Proc.}, August 2017.
	
	\bibitem{TCIgauri}
	G.~Jagatap, Z.~Chen, S.~Nayer, C.~Hegde, and N.~Vaswani,
	\newblock ``Sample efficient fourier ptychography for structured data,''
	\newblock {\em IEEE Trans Comput Imaging}, 2019.
	
	\bibitem{wainwright_linear_columnwise}
	Sahand Negahban, Martin~J Wainwright, et~al.,
	\newblock ``Estimation of (near) low-rank matrices with noise and
	high-dimensional scaling,''
	\newblock {\em The Annals of Statistics}, vol. 39, no. 2, pp. 1069--1097, 2011.
	
	\bibitem{cov_sketch}
	Yuxin Chen, Yuejie Chi, and Andrea~J Goldsmith,
	\newblock ``Exact and stable covariance estimation from quadratic sampling via
	convex programming,''
	\newblock {\em IEEE Transactions on Information Theory}, vol. 61, no. 7, pp.
	4034--4059, 2015.
	
	\bibitem{matcomp_candes}
	E.~J. Candes and B.~Recht,
	\newblock ``Exact matrix completion via convex optimization,''
	\newblock {\em Found. of Comput. Math}, , no. 9, pp. 717--772, 2008.
	
	\bibitem{lowrank_altmin}
	P.~Netrapalli, P.~Jain, and S.~Sanghavi,
	\newblock ``Low-rank matrix completion using alternating minimization,''
	\newblock in {\em STOC}, 2013.
	
	\bibitem{rmc_gd}
	Y.~Cherapanamjeri, K.~Gupta, and P.~Jain,
	\newblock ``Nearly-optimal robust matrix completion,''
	\newblock {\em ICML}, 2016.
	
	\bibitem{nonconvex_review}
	Yuejie Chi, Yue~M Lu, and Yuxin Chen,
	\newblock ``Nonconvex optimization meets low-rank matrix factorization: An
	overview,''
	\newblock {\em IEEE Transactions on Signal Processing}, vol. 67, no. 20, pp.
	5239--5269, 2019.
	
	\bibitem{Bresler_matrix}
	K.~Lee and Y.~Bresler,
	\newblock ``Admira: Atomic decomposition for minimum rank approximation,''
	\newblock {\em IEEE Transactions on Information Theory}, vol. 56, no. 9,
	September 2010.
	
	\bibitem{versh_book}
	Roman Vershynin,
	\newblock {\em High-dimensional probability: An introduction with applications
		in data science}, vol.~47,
	\newblock Cambridge University Press, 2018.
	
	\bibitem{dyn_mri1}
	Sajan~Goud Lingala, Yue Hu, Edward DiBella, and Mathews Jacob,
	\newblock ``Accelerated dynamic mri exploiting sparsity and low-rank structure:
	kt slr,''
	\newblock {\em IEEE transactions on medical imaging}, vol. 30, no. 5, pp.
	1042--1054, 2011.
	
	\bibitem{dyn_mri2}
	Jiawen Yao, Zheng Xu, Xiaolei Huang, and Junzhou Huang,
	\newblock ``An efficient algorithm for dynamic mri using low-rank and total
	variation regularizations,''
	\newblock {\em Medical image analysis}, vol. 44, pp. 14--27, 2018.
	
	\bibitem{butala}
	M.D. Butala, R.A. Frazin, Y.~Chen, and F.~Kamalabadi,
	\newblock ``A monte carlo technique for large-scale dynamic tomography,''
	\newblock in {\em IEEE Intl. Conf. Acoustics, Speech, Sig. Proc. (ICASSP)},
	2007.
	
	\bibitem{holloway}
	J.~Holloway, M.~S. Asif, M.~K. Sharma, N.~Matsuda, R.~Horstmeyer, O.~Cossairt,
	and A.~Veeraraghavan,
	\newblock ``Toward long-distance subdiffraction imaging using coherent camera
	arrays,''
	\newblock {\em IEEE Trans Comput Imaging}, vol. 2, no. 3, pp. 251--265, 2016.
	
	\bibitem{icip20}
	Z.~Chen, S.~Nayer, and N.~Vaswani,
	\newblock ``Fast and sample-efficient low rank fourier ptychography,''
	\newblock in {\em IEEE Intl. Conf. Image Proc. (ICIP)}, submitted, 2020.
	
	\bibitem{hughes_icip_2012}
	Hanchao Qi and Shannon~M Hughes,
	\newblock ``Invariance of principal components under low-dimensional random
	projection of the data,''
	\newblock in {\em 2012 19th IEEE International Conference on Image Processing}.
	IEEE, 2012, pp. 937--940.
	
	\bibitem{hughes_icml_2014}
	Farhad~Pourkamali Anaraki and Shannon Hughes,
	\newblock ``Memory and computation efficient pca via very sparse random
	projections,''
	\newblock in {\em International Conference on Machine Learning}, 2014, pp.
	1341--1349.
	
	\bibitem{aarti_singh_subs_learn}
	Akshay Krishnamurthy, Martin Azizyan, and Aarti Singh,
	\newblock ``Subspace learning from extremely compressed measurements,''
	\newblock in {\em Asilomar Conference}, 2014.
	
	\bibitem{sketch_1}
	David~P Woodruff et~al.,
	\newblock ``Sketching as a tool for numerical linear algebra,''
	\newblock {\em Foundations and Trends{\textregistered} in Theoretical Computer
		Science}, vol. 10, no. 1--2, pp. 1--157, 2014.
	
	\bibitem{sketch_2}
	Anna~C Gilbert, Jae~Young Park, and Michael~B Wakin,
	\newblock ``Sketched svd: Recovering spectral features from compressive
	measurements,''
	\newblock {\em arXiv preprint arXiv:1211.0361}, 2012.
	
	\bibitem{aarti_singh_cov_est}
	Martin Azizyan, Akshay Krishnamurthy, and Aarti Singh,
	\newblock ``Extreme compressive sampling for covariance estimation,''
	\newblock {\em IEEE Trans. Info. Th.}, 2018.
	
	\bibitem{local_conv_pr}
	Sujay Sanghavi, Rachel Ward, and Chris~D White,
	\newblock ``The local convexity of solving systems of quadratic equations,''
	\newblock {\em Results in Mathematics}, vol. 71, no. 3-4, pp. 569--608, 2017.
	
	\bibitem{first_lrms_convex}
	Benjamin Recht, Maryam Fazel, and Pablo~A Parrilo,
	\newblock ``Guaranteed minimum-rank solutions of linear matrix equations via
	nuclear norm minimization,''
	\newblock {\em SIAM review}, vol. 52, no. 3, pp. 471--501, 2010.
	
	\bibitem{tanner}
	Jared Tanner and Ke~Wei,
	\newblock ``Normalized iterative hard thresholding for matrix completion,''
	\newblock {\em SIAM Journal on Scientific Computing}, vol. 35, no. 5, pp.
	S104--S125, 2013.
	
	\bibitem{svp}
	Prateek Jain, Raghu Meka, and Inderjit~S Dhillon,
	\newblock ``Guaranteed rank minimization via singular value projection,''
	\newblock in {\em Advances in Neural Information Processing Systems}, 2010, pp.
	937--945.
	
	\bibitem{procrustes}
	S.~Tu, R.~Boczar, M.~Soltanolkotabi, and B.~Recht,
	\newblock ``Low-rank solutions of linear matrix equations via procrustes
	flow,''
	\newblock {\em arXiv preprint arXiv:1507.03566}, 2015.
	
	\bibitem{zheng_lafferty}
	Q.~Zheng and J.~Lafferty,
	\newblock ``A convergent gradient descent algorithm for rank minimization and
	semidefinite programming from random linear measurements,''
	\newblock in {\em NIPS}, 2015.
	
	\bibitem{optspace}
	R.H. Keshavan, A.~Montanari, and S.~Oh,
	\newblock ``Matrix completion from a few entries,''
	\newblock {\em IEEE Trans. Info. Th.}, vol. 56, no. 6, pp. 2980--2998, 2010.
	
	\bibitem{mc_luo}
	R.~Sun and Z-Q. Luo,
	\newblock ``Guaranteed matrix completion via non-convex factorization,''
	\newblock {\em IEEE Trans. Info. Th.}, vol. 62, no. 11, pp. 6535--6579, 2016.
	
	\bibitem{jaganathan2012recovery}
	K.~Jaganathan, S.~Oymak, and B.~Hassibi,
	\newblock ``Recovery of sparse 1-d signals from the magnitudes of their fourier
	transform,''
	\newblock in {\em IEEE Intl. Symp. Info. Th. (ISIT)}. IEEE, 2012, pp.
	1473--1477.
	
	\bibitem{coherence_est}
	Mehryar Mohri and Ameet Talwalkar,
	\newblock ``Can matrix coherence be efficiently and accurately estimated?,''
	\newblock in {\em Proceedings of the Fourteenth International Conference on
		Artificial Intelligence and Statistics}, 2011, pp. 534--542.
	
	\bibitem{robpca_nonconvex}
	P.~Netrapalli, U~N Niranjan, S.~Sanghavi, A.~Anandkumar, and P.~Jain,
	\newblock ``Non-convex robust pca,''
	\newblock in {\em NIPS}, 2014.
	
	\bibitem{candes2009tight}
	E~Candes and Y~Plan,
	\newblock ``Tight oracle bounds for low-rank matrix recovery from a minimal
	number of random measurements. to appear,''
	\newblock {\em IEEE Trans. Info. Theo}, 2009.
	
	\bibitem{davis_kahan}
	C.~Davis and W.~M. Kahan,
	\newblock ``The rotation of eigenvectors by a perturbation. iii,''
	\newblock {\em SIAM J. Numer. Anal.}, vol. 7, pp. 1--46, Mar. 1970.
	
	\bibitem{vershynin}
	R.~Vershynin,
	\newblock ``Introduction to the non-asymptotic analysis of random matrices,''
	\newblock {\em Compressed sensing}, pp. 210--268, 2012.
	
	\bibitem{rrpcp_icml}
	P.~Narayanamurthy and N.~Vaswani,
	\newblock ``Nearly optimal robust subspace tracking,''
	\newblock in {\em Intnl. Conf. Machine Learning (ICML)}, 2018, pp. 3701--3709.
	
	\bibitem{rrpcp_dynrpca}
	P.~Narayanamurthy and N.~Vaswani,
	\newblock ``Provable dynamic robust pca or robust subspace tracking,''
	\newblock {\em IEEE Transactions on Information Theory}, vol. 65, no. 3, pp.
	1547--1577, 2019.
	
	\bibitem{streamingpca_miss}
	I.~Mitliagkas, C.~Caramanis, and P.~Jain,
	\newblock ``Streaming pca with many missing entries,''
	\newblock {\em Preprint}, 2014.
	
	\bibitem{lrpr_globalsip}
	S.~Nayer and N.~Vaswani,
	\newblock ``Phaseless subspace tracking,''
	\newblock {\em IEEE conference GlobalSIP (arXiv preprint arXiv:1809.04176)},
	2018.
	
	\bibitem{lrpr_icassp19}
	S.~Nayer and N.~Vaswani,
	\newblock ``Phast: Model-free phaseless subspace tracking,''
	\newblock {\em IEEE Intl. Conf. Acoustics, Speech, Sig. Proc. (ICASSP)}, 2019.
	
	\bibitem{stab_jinchun_jp}
	J.~Zhan and N.~Vaswani,
	\newblock ``Time invariant error bounds for modified-{CS} based sparse signal
	sequence recovery,''
	\newblock {\em IEEE Trans. Info. Th.}, vol. 61, no. 3, pp. 1389--1409, 2015.
	
	\bibitem{sslearn_jmlr}
	A.~Gonen, D.~Rosenbaum, Y.~C. Eldar, and S.~Shalev-Shwartz,
	\newblock ``Subspace learning with partial information,''
	\newblock {\em The Journal of Machine Learning Research}, vol. 17, no. 1, pp.
	1821--1841, 2016.
	
	\bibitem{grouse}
	L.~Balzano, B.~Recht, and R.~Nowak,
	\newblock ``{Online Identification and Tracking of Subspaces from Highly
		Incomplete Information},''
	\newblock in {\em Allerton Conf. Comm., Control, Comput.}, 2010.
	
	\bibitem{petrels}
	Y.~Chi, Y.~C. Eldar, and R.~Calderbank,
	\newblock ``Petrels: Parallel subspace estimation and tracking by recursive
	least squares from partial observations,''
	\newblock {\em IEEE Trans. Sig. Proc.}, December 2013.
	
	\bibitem{local_conv_grouse}
	L.~Balzano and S.~Wright,
	\newblock ``Local convergence of an algorithm for subspace identification from
	partial data,''
	\newblock {\em Found. Comput. Math.}, vol. 15, no. 5, 2015.
	
	\bibitem{chi_review}
	L.~Balzano, Y.~Chi, and Y.~M. Lu,
	\newblock ``Streaming pca and subspace tracking: The missing data case,''
	\newblock {\em Proceedings of IEEE}, 2018.
	
	\bibitem{rrpcp_tsp19}
	P.~Narayanamurthy, V.~Daneshpajooh, and N.~Vaswani,
	\newblock ``Provable subspace tracking from missing data and matrix
	completion,''
	\newblock {\em IEEE Trans. Sig. Proc.}, pp. 4245--4260, 2019.
	
	\bibitem{rpca_gd}
	X.~Yi, D.~Park, Y.~Chen, and C.~Caramanis,
	\newblock ``Fast algorithms for robust pca via gradient descent,''
	\newblock in {\em NIPS}, 2016.
	
	\bibitem{lee2017near}
	Kiryung Lee, Yihong Wu, and Yoram Bresler,
	\newblock ``Near-optimal compressed sensing of a class of sparse low-rank
	matrices via sparse power factorization,''
	\newblock {\em IEEE Transactions on Information Theory}, vol. 64, no. 3, pp.
	1666--1698, 2017.
	
	\bibitem{ermolova2004simplified}
	Natalia Ermolova and Sven-Gustav Haggman,
	\newblock ``Simplified bounds for the complementary error function; application
	to the performance evaluation of signal-processing systems,''
	\newblock in {\em Signal Processing Conference, 2004 12th European}. IEEE,
	2004, pp. 1087--1090.
	
	\bibitem{rrpcp_isit}
	C.~Qiu and N.~Vaswani,
	\newblock ``Support predicted modified-cs for recursive robust principal
	components' pursuit,''
	\newblock in {\em IEEE Intl. Symp. Info. Th. (ISIT)}, 2011.
	
\end{thebibliography}
\section*{Author Biographies}

\begin{IEEEbiography}[{\includegraphics[width=1in,height=1.25in,clip,keepaspectratio]{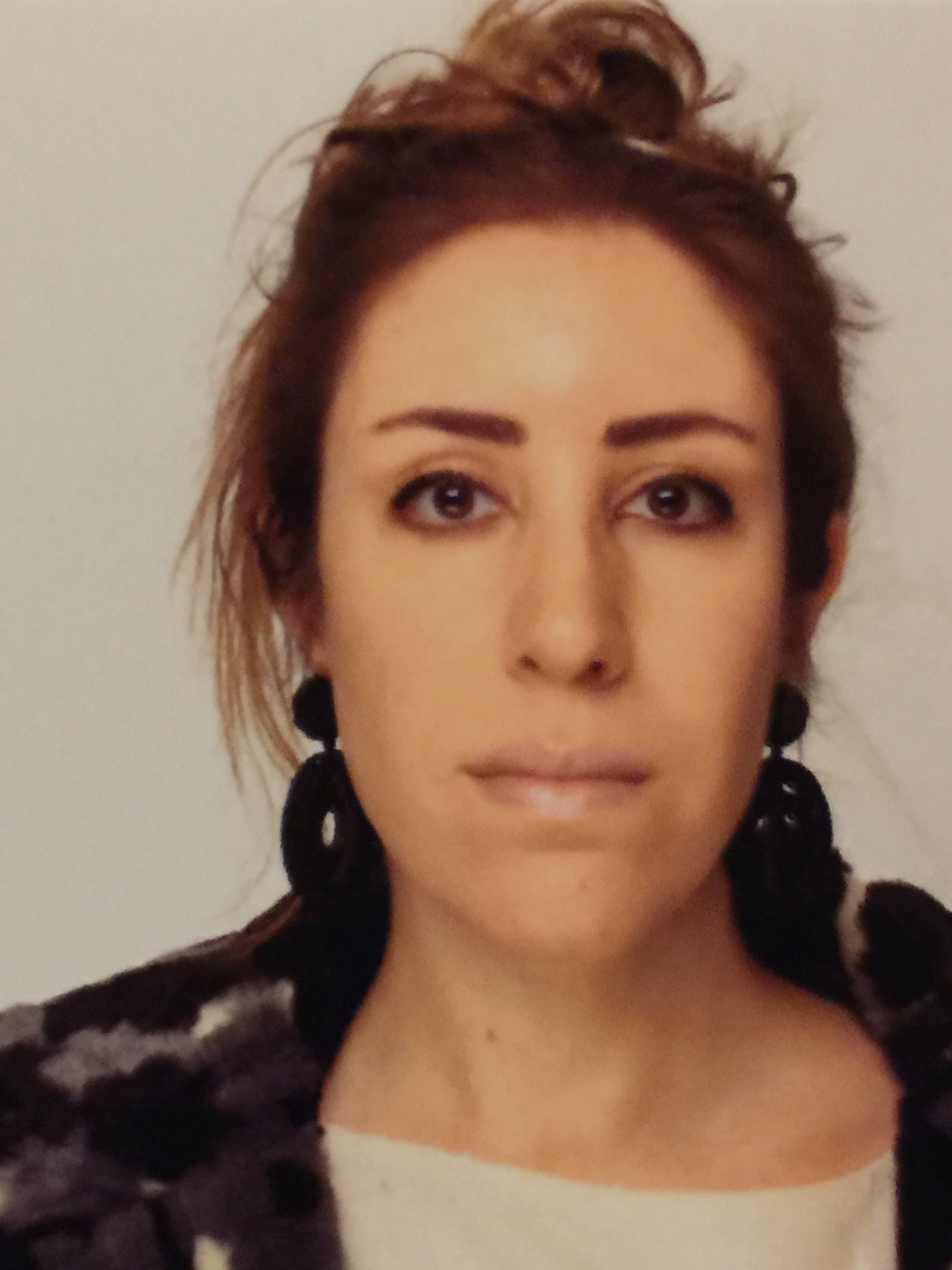}}]
{Seyedehsara Nayer}(Email: sarana@iastate.edu)
	is currently a PhD student in the Department of Electrical and Computer Engineering at Iowa State University. Previously she was working at Signal Processing research laboratory, Sharif University of Technology.
	Her research interests are around various aspects of information science and focuses on  Computer Vision, Signal Processing, and Statistical Machine Learning,.
\end{IEEEbiography}

\begin{IEEEbiography}[{\includegraphics[width=1in,height=1.25in,clip,keepaspectratio]{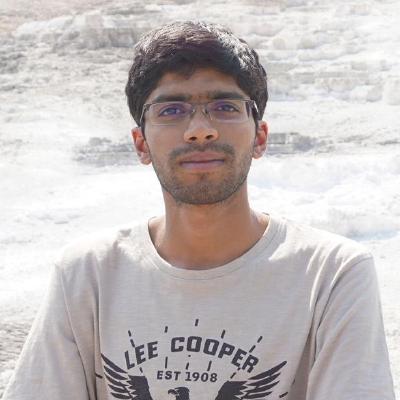}}]
{Praneeth Narayanamurthy} (Email: pkurpadn@iastate.edu) (S' 18) is a Ph.D. student in the Department of Electrical and Computer Engineering at Iowa State University (Ames, IA). He previously obtained his B.Tech degree in Electrical and Electronics Engineering from National Institute of Technology Karnataka (Surathkal, India) in 2014. His research interests include the algorithmic and theoretical aspects of High-Dimensional Statistical Signal Processing, and Machine Learning.
\end{IEEEbiography}

\begin{IEEEbiography}[{\includegraphics[width=1in,height=1.25in,clip,keepaspectratio]{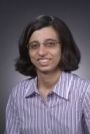}}]
{Namrata Vaswani}(Email: namrata@iastate.edu)
	received a B.Tech from the Indian Institute of Technology (IIT), Delhi, in 1999 and a Ph.D. from UMD in 2004, both in electrical engineering. Since Fall 2005, she has been with the Iowa State University where she is currently the Anderlik Professor of Electrical and Computer Engineering. Her research interests lie in a data science, with a particular focus on Statistical Machine Learning, Signal Processing, and Computer Vision. She has served two terms as an Associate Editor for the IEEE Transactions on Signal Processing; as a lead guest-editor for a Proceedings of the IEEE Special Issue (Rethinking PCA for modern datasets); and is currently serving as an Area Editor for the  IEEE Signal Processing Magazine.
	
	Vaswani is a recipient of the Iowa State Early Career Engineering Faculty Research Award (2014), the Iowa State University Mid-Career Achievement in Research Award (2019), University of Maryland's ECE Distinguished Alumni Award (2019), and the 2014 IEEE Signal Processing Society Best Paper Award. This award recognized the contributions of her 2010 IEEE Transactions on Signal Processing paper co-authored with her student Wei Lu on {\em Modified-CS: Modifying compressive sensing for problems with partially known support}. She is a Fellow of the IEEE (class of 2019).
\end{IEEEbiography}

\end{document}